\documentclass[11pt]{article}
\usepackage{xr}
\usepackage[margin=1in]{geometry}
\usepackage[small]{titlesec}
\usepackage{natbib,amsmath,amssymb,amsthm,graphicx,setspace,paralist,booktabs,rotating,subcaption,float,color}
\usepackage{times} 
\usepackage{multirow}
\usepackage[dvipsnames,table,dvipsnames*, svgnames*, hyperref]{xcolor}
\usepackage[colorlinks=true,linkcolor=blue,citecolor=blue,urlcolor=blue,breaklinks]{hyperref}
\usepackage{xr}
\usepackage{algpseudocode,float,algorithm}
      
\usepackage{amssymb,bm,bbm} 	
    
\usepackage{tikz}
\usetikzlibrary{chains,fit,arrows,shapes,snakes}
\usetikzlibrary{shapes.misc}
\usetikzlibrary{shapes.multipart}
\usetikzlibrary{positioning}
\usetikzlibrary{decorations.pathreplacing,angles,quotes}
\usepackage{smartdiagram}
\usepackage{metalogo} 
     
\AtBeginDocument{
\abovedisplayskip=5pt plus 2pt minus 2pt
\belowdisplayskip=\abovedisplayskip
\abovedisplayshortskip=2pt plus 2pt minus 2pt
\belowdisplayshortskip=\belowdisplayskip}
\titlespacing*{\section}{0pt}{*2}{*1}
\titlespacing*{\subsection}{0pt}{*2}{*1} 
\bibpunct{(}{)}{;}{a}{}{,}
\newtheorem{thm}{Theorem}[section]
\newtheorem{cor}[thm]{Corollary}
\newtheorem{lem}[thm]{Lemma}
\newtheorem{assu}[thm]{Assumption}
\newtheorem{prop}[thm]{Proposition}
\newtheorem{defn}[thm]{Definition}

\theoremstyle{remark}
\newtheorem{rem}[thm]{Remark}
\newtheorem{algo}[thm]{Algorithm}

\newcommand{\beq}{\begin{equation}}
\newcommand{\eeq}{\end{equation}}
\newcommand{\beqs}{\begin{equation*}}
\newcommand{\eeqs}{\end{equation*}}
\newcommand{\beas}{\begin{align*}}
\newcommand{\eeas}{\end{align*}}
\newcommand{\bea}{\begin{align}}
\newcommand{\eea}{\end{align}}
\newcommand{\bei}{\begin{itemize}}
	\newcommand{\eei}{\end{itemize}}
\newcommand{\ben}{\begin{enumerate}}
	\newcommand{\een}{\end{enumerate}}
\newcommand{\bet}{\begin{theorem}}
	\newcommand{\eet}{\end{theorem}}
\newcommand{\bel}{\begin{lemma}}
	\newcommand{\eel}{\end{lemma}}
\newcommand{\bep}{\begin{proposition}}
	\newcommand{\eep}{\end{proposition}}
\newcommand{\bed}{\begin{definition}}
	\newcommand{\eed}{\end{definition}}
\newcommand{\bec}{\begin{corollary}}
	\newcommand{\eec}{\end{corollary}}
\newcommand{\bex}{\begin{example}}
	\newcommand{\eex}{\end{example}}

\newcommand{\bu}{\bold{u}}

\newcommand{\bw}{\bold{w}}
\newcommand{\bv}{\bold{v}}

\newcommand{\SMP}{\mathsf{MP}}

\newcommand{\bU}{\bold{U}}

\newcommand{\bW}{\bold{W}}

\newcommand{\bN}{\bold{N}}
\newcommand{\bB}{\bold{B}}
\newcommand{\bA}{\bold{A}}
\newcommand{\bF}{\bold{F}}
\newcommand{\bR}{\bold{R}}
\newcommand{\bV}{\bold{V}}

\newcommand{\bK}{\bold{K}}

\newcommand{\bO}{\bold{O}}

\newcommand{\bLam}{\bold{\Lambda}}
\newcommand{\bTheta}{\bold{\Theta}}

\newcommand{\bSig}{\bold{\Sigma}}

\newcommand{\bzeta}{\boldsymbol{\zeta}}

\newcommand{\bbeta}{\boldsymbol{\beta}}
\newcommand{\bmu}{\boldsymbol{\mu}}

\newcommand{\bxi}{\bm{\xi}}
\newcommand{\bx}{\bm{x}}
\newcommand{\by}{\bm{y}}

\newcommand{\R}{\mathbb{R}}
\newcommand{\E}{\mathbb{E}}

\newcommand{\vertiii}[1]{{\left\vert\kern-0.25ex\left\vert\kern-0.25ex\left\vert #1 
		\right\vert\kern-0.25ex\right\vert\kern-0.25ex\right\vert}}

\newcommand{\xb}{\mathbf{x}}
\newcommand{\yb}{\mathbf{y}}
\newcommand{\vb}{\mathbf{v}}
\newcommand{\ub}{\mathbf{u}}
\newcommand{\zb}{\mathbf{z}}
\newcommand{\OO}{\mathrm{O}}
\newcommand{\oo}{\mathrm{o}}
\newcommand{\dd}{\mathrm{d}}
\newcommand{\ri}{\mathrm{i}}

\newcommand{\sfP}{\mathsf{P}}

\newcommand{\sfr}{\mathsf{r}}

\allowdisplaybreaks[1]

\newcommand{\ignore}[1]{}
\begin{document}

\begin{titlepage}
\setstretch{1.14}
\title{Kernel Spectral Joint Embeddings for High-Dimensional Noisy Datasets using Duo-Landmark Integral Operators}
	\author{Xiucai Ding\thanks{Xiucai Ding is Associate Professor at Department of Statistics, University of California, Davis, CA 95616 (E-mail: \emph{xcading@ucdavis.edu}). He is currently partially supported by NSF DMS-2515104.} and  Rong Ma \thanks{Rong Ma is Assistant Professor at Department of Biostatistics, T.H. Chan School of Public Health, Harvard University, Department of Data Science, Dana-Farber Cancer Institute, and the Broad Institute of MIT and Harvard, Boston, MA 02115 (E-mail: \emph{rongma@hsph.harvard.edu}). Rong gratefully acknowledges Boris Landa and Yuval Kluger for constructive discussions on methodology and numerical simulations during the development of the methods for a related problem in \cite{landa2024entropic}.}}
	\date{}
	\maketitle
	\thispagestyle{empty}

\begin{abstract}
Integrative analysis of multiple heterogeneous datasets has arised in many research fields. Existing approaches oftentimes suffer from limited power in capturing nonlinear structures, insufficient account of noisiness and effects of high-dimensionality, lack of adaptivity to signals and sample sizes imbalance, and their results are sometimes difficult to interpret. To address these limitations, we propose a kernel spectral method that achieves joint embeddings of two independently observed high-dimensional noisy datasets. The proposed method automatically captures and leverages  {shared} low-dimensional structures across datasets to enhance embedding quality. The obtained low-dimensional embeddings can be utilized for downstream tasks such as simultaneous clustering, data visualization, and denoising. The proposed method is justified by rigorous theoretical analysis, which guarantees its consistency in capturing the signal structures, and provides a geometric interpretation of the embeddings. Under a joint manifolds model framework, we establish the convergence of the embeddings to the eigenfunctions of some natural integral operators. These operators, referred to as duo-landmark integral operators, are defined by the convolutional kernel maps of some reproducing kernel Hilbert spaces (RKHSs). These RKHSs capture the underlying,  {shared} low-dimensional nonlinear signal structures between the two datasets. Our numerical experiments and analyses of two pairs of single-cell omics datasets demonstrate the empirical advantages of the proposed method over existing methods in both embeddings and several downstream tasks.

\end{abstract}

\end{titlepage}

\section{Introduction}

\subsection{Background and motivation}\label{sec_backgroundandmotivation}

The rapid advancement in technology and computing power has substantially improved the capacity of data collection, storage, processing and management. With the increasing availability of large, complex, and heterogeneous datasets, in  many fields such as molecular biology \citep{joyce2006model,li2023high}, precision medicine \citep{martinez2022data,ma2022optimal,einav2023using}, business intelligence \citep{dayal2009data}, and econometrics \citep{hunermund2023causal}, the interests and needs to combine and jointly analyze multiple datasets arises naturally, where the hope is to leverage the (possibly partially) shared information across datasets to obtain better characterizations of the underlying signal structures. 
As a prominent example, in single-cell omics research, integrating diverse datasets generated from different studies, samples, tissues, or across different sequencing technologies, time points, experimental conditions, etc., have become a standard practice \citep{stuart2019integrative,argelaguet2021computational}. Since many biological processes such as  gene regulations and cellular signaling are likely shared across different biological samples, tissues, or organs, combining multiple datasets has been found helpful to better reveal the biological signals underlying individual datasets \citep{xiong2022online,luecken2022benchmarking,ma2024principled,sun2024tissue}; see Figure \ref{fig_intro} for an illustration. 


\begin{figure}[ht]
	\centering
	\includegraphics[width=14cm]{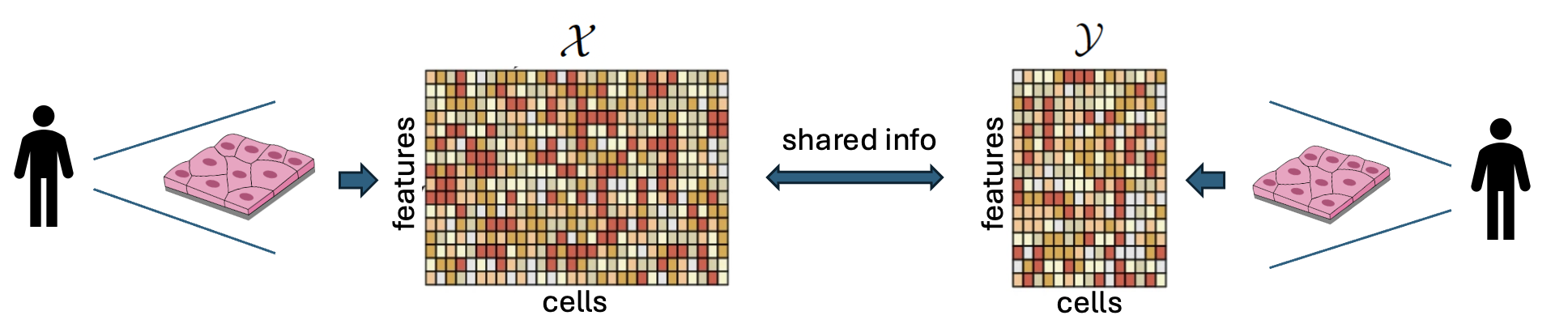}
	\caption{Illustration of multiple independently observed datasets with potentially shared information. Each dataset contains the same features but possibly different number of samples. }\label{fig_intro}
\end{figure}


Motivated by these applications, this paper focuses on the joint embedding of two independently observed high-dimensional noisy datasets $\mathcal{X}=\{\xb_i\}_{i=1}^{{n_1}} \subset \mathbb{R}^p$ and $\mathcal{Y}=\{\yb_j\}_{j=1}^{n_2} \subset \mathbb{R}^p,$ with possibly shared signal structures, as depicted in Figure \ref{fig_intro}. Our goal is to effectively learn the valuable signal structures from both datasets, which typically exhibit highly nonlinear patterns in biomedical research \citep{moon2019visualizing, sapoval2022current,ma2023spectral,sun2023dynamic}.

Before proceeding further, we must first underscore the distinction between our concerned problem and a closely related yet markedly different problem which is also frequently encountered in modern biomedical research: the necessity of integrating datasets pertaining to identical samples, while each dataset comprising varying ``views" or types of measurements \citep{rappoport2018multi,li2020inferring,kline2022multimodal}. For example, in single-cell multiomics research \citep{argelaguet2020mofa+,baysoy2023technological}, measurements on  genomic variation and gene expression can be obtained simultaneously for the same set of cells, leading to two data matrices concerning distinct types (and numbers) of features \citep{yu2023scone}. The goal of multiomic analysis is to combine two or more views or modalities (e.g., epigenomic and transcriptomic information) of the same cells for more comprehensive profiling of their cell identities and associated molecular processes. The underlying analytical task is often known in literature as multi-view learning \citep{zhao2017multi} or sensor fusion \citep{gustafsson2010statistical}. Lots of efforts are made to find effective algorithms under various assumptions, to list but a few,  canonical correlation analysis (CCA) \citep{da6385d2-9c65-3860-bbcd-b821fdff69ff,hardoon2004canonical}, nonparametric canonical correlation analysis (NCCA) \citep{michaeli2016nonparametric}, kernel canonical correlation analysis (KCCA) \citep{lai2000kernel}, deep canonical correlation analysis (DCCA) \citep{andrew2013deep}, alternating diffusion \citep{talmon2019latent}, and time coupled diffusion maps \citep{marshall2018time}, etc. For the sensor fusion problem, the two datasets, denoted as $\mathcal{X}'=\{\xb'_i\}_{i=1}^{n} \subset \mathbb{R}^{p_1}$ and $\mathcal{Y}'=\{\yb'_j\}_{j=1}^{n} \subset \mathbb{R}^{p_2}$, typically display dependence and possess an equal number of samples, $n$, yet they may differ in their feature dimensions, represented by $p_1$ and $p_2$, respectively. 
On the contrary, our work addresses a different problem where the two datasets $\mathcal{X}$ and $\mathcal{Y}$ are independent and share the same set of $p$ features, although they may possess varying sample sizes, $n_1$ and $n_2$.

Despite the importance and significance of the joint embedding problem in modern biomedical research, to our knowledge, it has not yet been formally or systematically addressed in the statistical literature. Several  computational approaches have been developed specifically for single-cell applications \citep{haghverdi2018batch,stuart2019comprehensive, cao2022manifold,luecken2022benchmarking,kang2022roadmap}.   
In spite of their initial success, there are several key limitations about the existing methods. First, all of these methods are developed heuristically, and the absence of a theoretical foundation significantly constrains the users' interpretation of the results and the potential to address new tasks \citep{mazumder2024dataperf,omiye2024large}. Second, existing methods are based on conceptions built upon noiseless observations in the low-dimensional setting so that their  applicability and validity with respect to high-dimensional and noisy datasets is in general left unguaranteed  \citep{chari2023specious,liu2025assessing}. Finally, biomedical datasets often exhibit high heterogeneity, characterized by varying signal-to-noise levels and sample sizes across individual datasets; however, existing methods  do not provide automatic adaptation and often fail to effectively handle  such information imbalance \citep{maan2024characterizing, maan2024characterizing1}.

To address these limitations, in this paper, we formulate this problem rigorously under a joint manifolds model framework (Assumption \ref{assum_signal} and Figure \ref{fig_manifoldmodelillustration}) and  propose a novel kernel spectral method that achieves joint embeddings of two independently observed high-dimensional and noisy datasets. Building upon the mathematical insights surrounding the newly introduced duo-landmark integral operators (Definition \ref{defn_landmarkintegral}), we propose an asymmetric cross-data kernel matrix along with a data-adaptive bandwidth selection procedure (Algorithm \ref{al0}) to capture the joint information contained in both datasets, and develop spectral joint embeddings of both datasets based on the singular values and singular vectors of the kernel matrix. The proposed method automatically captures and leverages possibly shared low-dimensional structures across datasets to enhance embedding quality.  The obtained joint embeddings can be useful for many downstream tasks, such as simultaneous clustering (Section \ref{sec_biclustering}) and learning low-dimensional nonlinear structures (Section \ref{sec_ml}) on both datasets. The proposed method is justified by rigorous theoretical analysis, which guarantees consistency of the low-dimensional joint embeddings concerning the underlying signal structures, and provides explicit interpretations about its geometric meaning. When applied to  real-world biological datasets (Section \ref{sec_realdataanalysis}) for joint analysis of noisy  high-dimensional single-cell omics datasets, the proposed method demonstrates superior performance in capturing the underlying biological signals as compared with alternative methods.

\subsection{Our algorithm, main results and contributions}

Our proposed method is summarized in Algorithm \ref{al0}. Influenced by \cite{10.1093/imaiai/iaac013}, the algorithm begins by {assessing the information sharing or ``alignability" between the two datasets of interest, avoiding unwanted distortions and misinterpretations that can arise from integrating unrelated data. The core of the algorithm involves linking data points from the point clouds $\mathcal{X}=\{\xb_i\}_{1 \leq i \leq n_1}$, $\mathcal{Y}=\{\yb_j\}_{1 \leq j \leq n_2}\subset \mathbb{R}^p$ via an asymmetric rectangular kernel affinity matrix, as described  in (\ref{eq_kernelmatrix}). 
Importantly, we exclude ``self-connections" within the individual dataset, $\mathcal{X}$ and $\mathcal{Y}$, and instead focus exclusively on establishing connections between the two datasets. This design choice is motivated by both theoretical and practical considerations, as detailed below.}


First, as mentioned in Section \ref{sec_backgroundandmotivation}, our primary objective in practical applications is to effectively utilize both datasets and accomplish various downstream tasks simultaneously or separately. In this context, $\mathcal{X}$ and $\mathcal{Y}$ might only share partial information, limiting the relevance of self-connections within one dataset for understanding the structure of the other. This rationale also clarifies our decision not to amalgamate the two datasets and create a symmetric affinity matrix treating them as one entity, as the signal structures may not be entirely shared across both datasets. 
This is further confirmed by our simulations (Section \ref{sec_simulation}) and analyses of real single-cell omics datasets (Section \ref{sec_realdataanalysis}), which demonstrate that our approach significantly outperforms the methods that merely concatenate the two datasets and involve self-connections.

Second, from a theoretical standpoint, our approach draws significant inspiration from recent advancements in manifold learning that enhance the scalability of various algorithms through the utilization of landmark datasets  \citep{shen2022robust,10.1093/imaiai/iaac013, wu2024design, 2024arXiv240419649Y}. However, there exist significant differences between our aims and setups and theirs. For clarity, we focus on explaining the differences from \cite{10.1093/imaiai/iaac013}. In their configuration, only one set of cloud points is observed (denoted as $\mathcal{X}$), with a single landmark dataset (denoted as $\mathcal{Y}$), essentially resampled from $\mathcal{X}$. Consequently, the two point clouds are dependent. The remarkable contribution of \cite{10.1093/imaiai/iaac013} lies in providing guidance on selecting resampling schemes and resample sizes based on an analysis of a single differential operator,  the Laplace operator.  In contrast, as noted in Section \ref{sec_backgroundandmotivation},  in our applications independent datasets $\mathcal{X}$ and $\mathcal{Y}$ are observed. In this context, $\mathcal{Y}$ can be considered the landmark dataset for $\mathcal{X}$, and conversely, $\mathcal{X}$ can also be viewed as the landmark dataset for $\mathcal{Y}$. This enables mutual learning between the two datasets through a pair of operators, namely, the newly introduced duo-landmark integral operators outlined in Definition \ref{defn_landmarkintegral}. In both our theoretical analysis (Section \ref{sec_theoreticalanalysis}) and numerical evaluations (Sections \ref{sec_simulation} and \ref{sec_realdataanalysis}), the duo-landmark integral operators demonstrate the capacity to extract crucial structures from the underlying datasets with assistance from each other, despite only sharing partial  information.

Another important component in Algorithm \ref{al0} is selecting an appropriate bandwidth parameter. Generalizing the approach from \cite{DW2}, our proposed method utilizes pairwise distances between datasets $\mathcal X$ and $\mathcal Y$, as described in (\ref{eq_bandwidthselection}). It is adaptive and does not depend on prior knowledge of the underlying  structures. Moreover, it theoretically ensures that our proposed joint embeddings (Step 3 of Algorithm \ref{al0}) can be geometrically interpreted using the eigenfunctions of the duo-landmark integral operators. {Importantly, unlike these existing works which deal with nonlinear embedding of a single dataset, the current work focuses on jointly embedding two related datasets, which has distinct objectives and applications, and involves very different ideas behind our proposed methods (see discussions after Definition \ref{defn_landmarkintegral}).}

In a manifold learning framework, we theoretically demonstrate that our proposed algorithm can produce embeddings capable of capturing the intrinsic structures of the underlying nonlinear manifolds, leveraging the eigenfunctions of the duo-integral operators linked with these manifolds.  
In our setup, shaped by the demands of our applications wherein datasets $\mathcal{X}$ and $\mathcal{Y}$ share only partial information, we propose a new \emph{joint manifolds model} (Assumption \ref{assum_signal} and Figure \ref{fig_manifoldmodelillustration}) for the underlying clean signals. The joint manifolds exhibit partially shared nonlinear structures, which can also be identical, thus encapsulating the common manifold model introduced in \cite{ding2021kernel,talmon2019latent} as a notable special case within our framework. For practical relevance, we consider the commonly used 
signal-plus-noise model for the observations where for $1 \leq i \leq n_1$ and $1 \leq j \leq n_2$ 
\begin{equation} \label{model}
	\xb_i=\xb_i^0+\bm{\xi}_i \in \mathbb{R}^p, \qquad \yb_j=\yb_j^0+\bm{\zeta}_j \in \mathbb{R}^p,
\end{equation} 
where $\{\xb_i^0\}$ and $\{\yb_j^0\}$ are the clean signals in $\mathbb{R}^p$ sampled from the joint manifolds model and $\{\bm{\xi}_i\}$ and $\{\bm{\zeta}_j\}$ are the noise terms. 

We emphasize that our approach allows the dimension $p$ to diverge alongside $n_1$ and $n_2$, addressing the issue of high dimensionality while accommodating significant differences between $n_1$ and $n_2$. More specifically, our methods are tailored to high-dimensional inputs, effectively mitigating potential imbalanced sample sizes. In addition, we allow the signal strengths and signal-to-noise ratios (SNRs) to vary across different datasets, addressing the issue of information imbalance. 

Under some mild conditions, we prove the convergence and robustness of our proposed algorithms. The theoretical novelty and significance can be summarized as follows.
\begin{itemize}
	\item 
	Under the newly introduced joint manifolds model framework, we present a set of innovative integral operators known as duo-landmark integral operators (Definition \ref{defn_landmarkintegral}). These operators {are built on} the nonlinear manifold models {to capture some common geometric structures} and are constructed upon convolution kernels alternating between the joint manifolds. This design facilitates the integration of information from datasets sampled from the two manifolds, enabling mutual learning between them. 
	\item Under some mild conditions, {when the two datasets have common structures and are alignable},  we establish the spectral convergence of our algorithms towards the duo-landmark operators. In other words, when the datasets remain pristine (i.e., $\xb_i=\xb_i^0$ and $\yb_j=\yb_j^0$ in (\ref{model})), we show that the outputs of our algorithms converge to the eigenvalues and eigenfunctions of the duo-landmark operators, accompanied by detailed convergence rate analyses.
	\item {We  analyze the impact of the high-dimensional noise on the proposed method and its  phase transition.}
We prove the robustness of our algorithm under the high-dimensional noise {when the signals dominate the noise}. We show that, under mild conditions, our proposed algorithm is robust against the noise and  still converges to the duo-landmark integral operators. {Moreover, we show that when the noise is stronger, although the embeddings may lose interpretability, they still exhibit a deterministic limiting behavior. This phenomenon can be understood through random matrix theory, specifically via the global spectral distribution, which follows the free multiplicative convolution of two Marchenko–Pastur laws.}
\end{itemize}

Before concluding this section, it is worth noting that our algorithm can effectively handle a range of downstream tasks when the joint manifolds possess specific structures. For instance, in cases where both manifolds exhibit cluster structures, the joint embeddings will leverage the possibly shared cluster patterns, so that employing some clustering algorithms (e.g., k-means) simultaneously to the joint embeddings will lead to improved clustering of the datasets (Section \ref{sec_biclustering}). As another example, when the two datasets contain some partially shared nonlinear smooth manifold structures, but one has a higher SNR while the other has a lower SNR, the proposed joint embeddings  would facilitate enhancing the embedding quality of the low-SNR dataset for better learning its underlying nonlinear manifolds (Section \ref{sec_ml}). In the applications to integrative single-cell omics analysis, we observe that compared with alternative methods, our algorithm achieves better performance in identifying the distinct cell types across pairs of datasets of gene expression or chromatin accessibility measures, generated from different experimental conditions or studies.

This work was developed concurrently with \cite{landa2024entropic}. At a high level, both papers study joint embeddings of two datasets using singular vectors/eigenfunctions from an asymmetric cross-dataset affinity matrix/operator.  Several elements of the present work were indeed influenced by the ideas there. For example, the normalization of the asymmetric matrix  and aspects of the numerical simulation design. While the two works address related problems,  \cite{landa2024entropic} develop an optimal-transport-based approach grounded in a diffusion-type framework that accommodates unshared dataset-specific deformations and corruptions (e.g., shifts/dilations and nuisance structures).

The rest of the paper is organized as follows. In Section \ref{sec_algo}, we introduce our proposed algorithm. In Section \ref{sec_theoreticalanalysis}, we  introduce the duo-landmark integral operators and prove the spectral convergence and robustness of our proposed methods. In Section \ref{sec_simulation}, we use extensive numerical simulations to illustrate the usefulness of our method and in Section \ref{sec_realdataanalysis}, we apply our algorithm to the integrative single-sell omics analysis. An online supplement is provided to present additional numerical results, provide detailed proofs of the main results and related technical lemmas, and include further discussions and remarks.


We finish this section by introducing some notation. To streamline our statements, we use the notion of \emph{stochastic domination}, which  is commonly adopted in random matrix theory \citep{erdos2017dynamical} to syntactically simplify precise statements of the form ``$\mathsf{X}^{(n)}$ is bounded with high probability by $\mathsf{Y}^{(n)}$ up to small powers of $n$."  More notation is given in the supplement.

\begin{defn} [Stochastic domination]\label{defn_stochasdomi} Let $	\mathsf{X}=\big\{\mathsf{X}^{(n)}(u):  n \in \mathbb{N}, \ u \in \mathsf{U}^{(n)}\big\}$ and $\mathsf{Y}=\big\{\mathsf{Y}^{(n)}(u):  n \in \mathbb{N}, \ u \in \mathsf{U}^{(n)}\big\}$
	be two families of nonnegative random variables, where $\mathsf{U}^{(n)}$ is a possibly $n$-dependent parameter set. We say  $\mathsf{X}$ is {\em stochastically dominated} by $\mathsf{Y}$, or $\mathsf{X}\prec \mathsf{Y}$, uniformly in the parameter $u$, if for all small $\upsilon>0$ and large $ D>0$, there exists $n_0(\upsilon, D)\in \mathbb{N}$ so that $
		\sup_{u \in \mathsf{U}^{(n)}} \mathbb{P} \big( \mathsf{X}^{(n)}(u)>n^{\upsilon}\mathsf{Y}^{(n)}(u) \big) \leq n^{- D},$
	for all $n \geq  n_0(\upsilon, D)$. 
\end{defn}

\section{Proposed algorithm}\label{sec_algo}

In this section, we introduce  our new algorithm for the joint embedding of two datasets with the same features, summarized below in Algorithm \ref{al0}. 

\begin{algorithm}[h!]
	\caption{Kernel spectral joint embeddings using duo-landmark integral operator} \label{al0}
	\begin{algorithmic}
		\State {\bf Input:} Observed (centered) samples $\{\xb_j\}_{j \in [n_1]}$, $\{\yb_i\}_{i\in[n_2]}$, {percentile} $\omega\in(0,1)$, { nearest neighbor $k$}, and {singular vector index sets} $\Gamma \subseteq\{1,2,...,\min\{n_1,n_2\}\}$. {
		\State 1. {\bf Alignability screening:} 
		\State \hspace{4mm} (i) define the full kernel matrix $\bF := (F(i,j))\in\R^{(n_1+n_2)\times(n_1+n_2)}$ based on the merged dataset $\{\zb_1,...,\zb_{n_1+n_2}\}\equiv\{\xb_1,...,\xb_{n_1},\yb_1,...,\yb_{n_2}\}$ by letting $
		F(i,j)=\exp\big( -\frac{\|\zb_i-\zb_j\|_2^2}{q_n} \big),$
		where $q_n$ is the the smallest value satisfying the inequality $\nu^{(z)}_n(q_n)\ge\omega$ where
		$\nu^{(z)}_n(t)=\frac{2}{(n_1+n_2)(n_1+n_2-1)}\sum_{\substack{1\le i\ne  j\le n_1+n_2}}1_{\{\|\zb_i-\zb_j\|_2^2\le t\}}$.
		\State \hspace{4mm}  (ii) let $\mathbf{L}_\Gamma \in \mathbb{R}^{(n_1+n_2) \times |\Gamma|}$ contain the eigenvectors of $\mathbf{F}$ indexed by $\Gamma$, and define the embeddings $\widehat{\mathcal{Z}} = \{\widehat{\mathbf{z}}_i\}_{1 \le i \le n_1+n_2}$ as the rows of $\mathbf{L}_\Gamma$; define the label set $\{l(\widehat{\mathbf{z}}_i)\}$ by 
		\[
		l(\widehat{\mathbf{z}}_i) = 
		\begin{cases}
			0, & \text{if } \mathbf{z}_i \in \{\mathbf{x}_i\}, \\
			1, & \text{if } \mathbf{z}_i \in \{\mathbf{y}_i\}.
		\end{cases}
		\]
		\State \hspace{4mm}  (iii) for each $\widehat\zb_i\in \widehat{\mathcal Z}$, we compute its $k$-nearest neighborhood purity: $p_i = \frac{\mathcal{N}_i}{k}$,
		where $\mathcal{N}_i$ is the number of nearest neighbors of $\widehat{\mathbf{z}}_i$ in $\widehat{\mathcal{Z}}$ sharing the same label $l_i$.
		\State \hspace{4mm}  (iv) if $\operatorname{median}(\{p_i\}) = 1$, stop the algorithm; otherwise, proceed to the next step. }
		\State 2. {\bf Duo-landmark kernel matrix construction:} define the kernel matrix $\bK:=(K(i,j)) \in \mathbb{R}^{n_1 \times n_2}$ by letting
		\begin{equation}\label{eq_kernelmatrix}
			K(i,j)=\exp\bigg( -\frac{\|\xb_i-\yb_j\|_2^2}{h_n} \bigg),
		\end{equation} 
		where $h_n$ is the the smallest value satisfying the inequality 
			\begin{equation}\label{eq_bandwidthselection}
			\nu_n(h_n)\ge\omega. 
		\end{equation}
		where
		$\nu^{(xy)}_n(t)=\frac{1}{n_1n_2}\sum_{\substack{1\le i\le n_1\\
				1\le j\le n_2}}1_{\{d_{ij}\le t\}},$ and $d_{ij}=\|\xb_i-\yb_j\|_2^2$  for all $1\le i\le n_1$ and $1 \leq j \leq n_2$.
		\State 3. {\bf Obtain duo-landmark joint embeddings:} 
		\State \hspace{4mm} (i) obtain the singular value decomposition (SVD) of the scaled kernel matrix $(n_1 n_2)^{-1}\bK$ as {
		\begin{equation} \label{eigen_decomp}
			\frac{1}{\sqrt{n_1 n_2}}\bK=\sum_{i=1}^{\min\{n_1, n_2\}} s_i \ub_i \vb_i^\top,
		\end{equation} }
		where $s_1\ge s_2\ge...\ge s_{\min\{n_1, n_2\}}$ are the singular values of $(n_1 n_2)^{-1/2}\bK$, and $\{\ub_i\}$ and $\{\vb_i\}$ are the corresponding left and right singular vectors.
		\State \hspace{4mm} (ii) let $\bU_{\Gamma_1} \in \mathbb{R}^{n_1 \times |\Gamma_1|}$  contain the left singular vectors indexed by the elements in $\Gamma_1$, $\bV_{\Gamma_2} \in \mathbb{R}^{n_2 \times |\Gamma_2|}$ contain the right singular vectors indexed by the elements in $\Gamma_2.$ Moreover, {for $\ell=1,2$, let $\bLam_{\Gamma_{\ell}}$ be a $|\Gamma_{\ell}| \times | \Gamma_{\ell}|$ diagonal matrix containing the singular values indexed by the elements in $\Gamma_{\ell}.$ }
		\State {\bf Output:} { kernel embeddings $\sqrt{n_1}\bU_{\Gamma_1} \bLam_{\Gamma_1}$ and $ \sqrt{n_2}\bV_{\Gamma_2}\bLam_{\Gamma_2}$}  for $\{\xb_i\}$ and $\{\yb_i\}$, respectively. 
	\end{algorithmic}
\end{algorithm}	
In Step 1 of the algorithm, we combine the two datasets $\{\xb_i\}$ and $\{\yb_j\}$ and obtain a joint kernel embedding of both datasets based on all the pairwise distances among the samples, using the method proposed in \cite{ding2022learning}. We then evaluate locally the degree of mixing, or ``overlap" between the low-dimensional embeddings of the two datasets. If there is a significant portion of data points indicating clear signal discrepancy between the two datasets, we will not proceed to integrate the two datasets. In other words, we only perform our proposed joint embedding algorithm when there is evidence of information sharing between the two dataests. {See Sections \ref{supp.sec.nu} and \ref{sec_tuningparametersection} of the supplement for more discussions regarding practical considerations.}

 {In Step 2 of the algorithm, we construct the kernel matrix $\bK=(K(i,j))\in\R^{n_1\times n_2}$, $K(i,j)=\exp(-{d_{ij}}/{h_n})$, using the  distances of the data points solely between the two datasets. Notably, the kernel matrix $\bK$  is of dimension $n_1 \times n_2$ and rectangular in general. For simplicity, here we focus on the Gaussian kernel function.  The results can be extended to other positive definite kernels and possibly multiple datasets; see Remarks \ref{rem_remarkkernel} and \ref{rem_morethanthree} of our supplement for more details. Meanwhile, a data-adaptive bandwidth parameter $h_n$ is chosen as the $\omega$-{percentile} of the empirical cumulative distribution function $\nu_n(t)$ of the pairwise between-dataset squared-distances $\{d_{ij}\}_{\substack{1\le i\le n_1,\\ 1 \leq j \leq n_2}}\equiv \{\|\xb_i-\yb_j\|_2^2\}_{\substack{1\le i\le n_1, \\1 \leq j \leq n_2}}$ of the observed datasets. This strategy is motivated by our theoretical analysis of the spectrum of kernel random matrices and its dependence on the associated bandwidth parameter; see Proposition \ref{lem_bandwidthconcentration}. It ensures the thus determined bandwidth $h_n$ adapts well to the unknown nonlinear structures and the SNRs of the datasets, so that the singular values and singular vectors of the associated kernel matrix capture the respective underlying low-dimensional structures via  the duo-landmark integral operators; see Section \ref{sec_theoreticalanalysis} for more details. {The percentile} $\omega$ is a tunable hyperparameter. {In Section \ref{sec_theoreticalanalysis} (and Section \ref{band.sec} of the supplement), we show in theory $\omega$ can be chosen as any constant between 0 and 1 to have the final embeddings achieve the same asymptotic behavior. In practice, as demonstrated by our numerical simulations, the performance remains robust across different choices of $\omega$. To improve automation of the method, we recommend using a resampling approach as in \cite{DW2}, described in Section \ref{sec_tuningparametersection} of the supplement, to determine {the percentile} $\omega$.}
In Step 3, the final embeddings are defined as the leading left and right singular vectors of the scaled kernel matrix, weighted by their respective singular values and square-root sample sizes. The {index sets} $\Gamma_1$ and $\Gamma_2$ of the embedding space are determined by the users, depending on the specific aims or downstream applications (see Section \ref{supp.sec.nu} of the supplement for additional discussions).}

We point out that in (\ref{eq_originalkernel}) and the constructions in Definition \ref{defn_clmd} below, for simplicity we considered the Gaussian kernel. However, we can easily generalize our results for a large class of kernel functions. For more details, we refer the readers to Remark \ref{rem_remarkkernel} of our supplement for more details. {Finally, we point out that the proposed algorithm—aimed at providing joint embeddings to extract {commonly shared structures}—is fundamentally different from the one in \cite{ding2022learning}. For more details, we refer the reader to Remark~\ref{rem_algo} in our supplement.}



%

\section{Theoretical properties and justifications}\label{sec_theoreticalanalysis}

We generalize the ideas in \cite{DW2,elkaroui2016} and assume that  $\{\xb_i\}$ and $\{\yb_j\}$ are sampled from some nonlinear manifolds model and corrupted by high-dimensional noise as in (\ref{model}). In what follows, we introduce the nonlinear manifolds model, {the common structures} and the duo-landmark integral operators {for capturing these common structures} in Section \ref{sec_manifoldmodelandlandmarkoperator}. In Section \ref{sec_spectralanalysis}, we establish the convergence results of $(n_1 n_2)^{-1} \bK^0$ to the duo-landmark integral operators, where $\bK^0$ is defined similarly as in (\ref{eq_kernelmatrix}) but using the clean signals, that is, 
\begin{equation}\label{eq_clearnsignalkernel}
	\bK^0=(K^0(i,j)),\qquad K^0(i,j)=\exp \left(-\|\xb_i^0-\yb_j^0\|_2^2/h^0_n \right),
\end{equation}
where $h_n^0$ is defined in the same way as in (\ref{eq_bandwidthselection}) using the clean signals.
In Section \ref{sec_robustness}, we {discuss how the high-dimensional noise impact our proposed algorithm}. 


\subsection{Joint manifolds model, {common structures} and duo-landmark integral operators}\label{sec_manifoldmodelandlandmarkoperator}

For the clean signals  $\{\xb_i^0\}$ and $\{\yb_j^0\}$, we suppose that they are sampled from some smooth manifolds model, which generalizes those considered  in    \cite{cheng2013local,talmon2019latent, Wu_Wu:2017}, and is formally summarized in Assumption \ref{assum_signal} and illustrated in Figure \ref{fig_manifoldmodelillustration}.

\begin{figure}[ht]
	\centering
	\includegraphics[height=5cm,width=8cm]{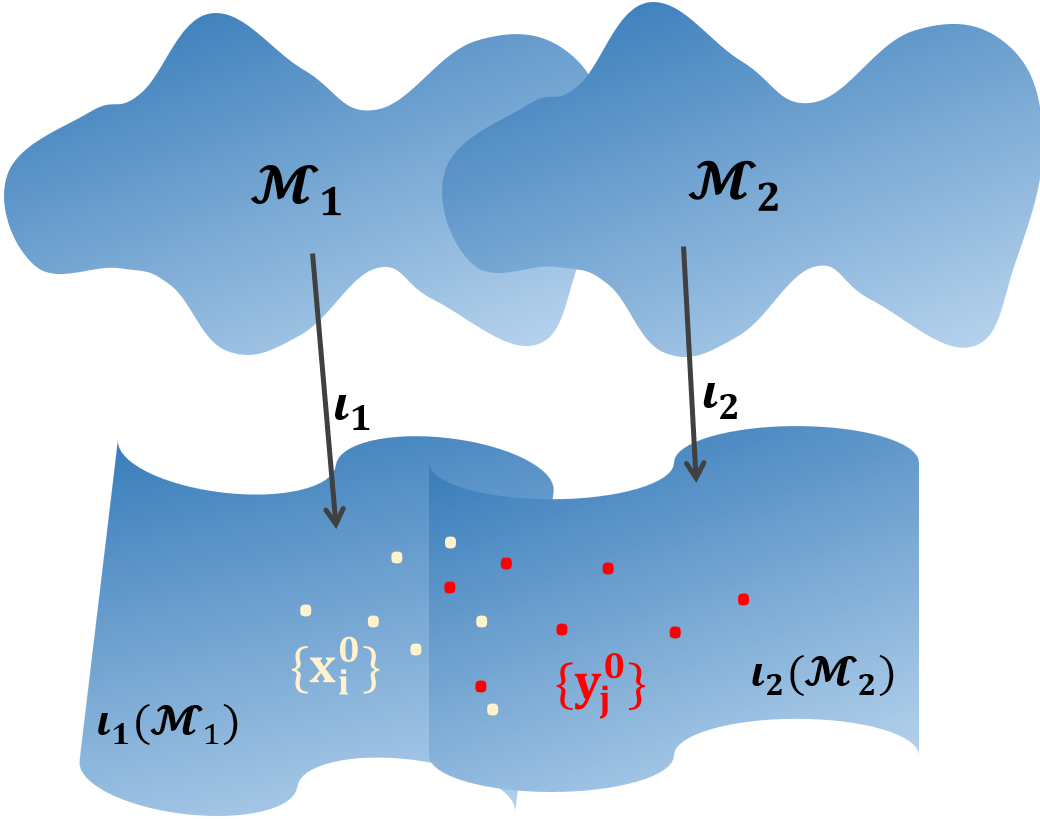}
	\caption{Illustration of the joint manifolds model. Here $\iota_1(\mathcal{M}_1)$ and $\iota_2(\mathcal{M}_2)$ contain partially overlapped or identical structures. This generalizes the common manifold model considered in \cite{ding2021kernel,talmon2019latent}.}\label{fig_manifoldmodelillustration}
\end{figure}
\vspace{-5pt}

\begin{assu}[Joint manifolds model] \label{assum_signal}
	For $\xb_{i}^0, 1 \leq i \leq n_1,$ we assume that they are centered and i.i.d. sampled from some sub-Gaussian random vector $X: \Omega_1 \rightarrow \mathbb{R}^p$ with respect to some probability space $(\Omega_1,\mathcal{F}_1, \mathbb{P}_1).$ Moreover, we assume that the range of $X$ is supported on an $m_1$-dimensional connected Riemannian manifold $\mathcal{M}_1$ isometrically embedded in $\mathbb{R}^p$ via $\iota_1: \mathcal{M}_1 \rightarrow \mathbb{R}^p.$ We suppose that $\operatorname{dim}(\iota_1(\mathcal{M}_1))=r_1 \leq p.$ Let $\widetilde{\mathcal{F}}_1$ be the Borel sigma algebra of $\iota_1(\mathcal{M}_1)\subseteq\R^p$ and denotes $\widetilde{\mathbb{P}}_1$ be the probability measure of $X$ defined on $\widetilde{\mathcal{F}}_1$ induced from $\mathbb{P}_1. $ We assume that $\widetilde{\mathbb{P}}_1$ is absolutely continuous with respect to the volume measure on $\iota_1(\mathcal{M}_1)$ whose density function is strictly positive.  In addition, we assume that similar conditions hold for $\yb_j, 1 \leq j \leq n_2,$ by replacing the quantities $(\Omega_1, \mathcal{F}_1, \mathbb{P}_1, X, \mathcal{M}_1, \iota_1, m_1, r_1)$ by some other combinations $(\Omega_2, \mathcal{F}_2, \mathbb{P}_2, Y, \mathcal{M}_2, \iota_2, m_2, r_2).$ Finally, we assume that $\{\xb_i^0\}$ and $\{\yb_j^0\}$ are independent. 
\end{assu}

{A few remarks are provided in order. First, since our Algorithm \ref{al0} concerns datasets after centralization, without loss of generality, we assume that all the samples are centered. Second, the connectedness of $\mathcal{M}_1$ and $\mathcal{M}_2$ guarantees that $\mathbb{P}_1$ and $\mathbb{P}_2$ correspond to some continuous random vectors, respectively. Third, we allow the dimensions $r_1$ and $r_2$ of  $\iota_1(\mathcal{M}_1)$ and $\iota_2(\mathcal{M}_2)$ to be generic numbers, which can diverge with $p$. As a result, our theoretical framework does not depend on the specific mappings $\iota_1$ and  $\iota_2$, and allows for more flexible nonlinear structures.
Finally, in real applications, oftentimes the object of interest is the embedded manifold  as the observations are sampled according to $X$ (or $Y$) which is supported on $\iota_1(\mathcal{M}_1)$ (or $\iota_2(\mathcal{M}_2)$). Consequently, we focus on understanding the shared geometric structures of $\iota_1(\mathcal{M}_1)$ and $\iota_2(\mathcal{M}_2)$. Nevertheless, our discussion is naturally related to $(\mathcal{M}_\ell,\iota_\ell), \ell=1,2;$ { see Section \ref{manifold.sec} of our supplement for more details. To formulate the common structures, we consider the following setup.
\begin{assu}[Common structure] \label{assum_commonstructureassumption} For the $p$-dimensional independent random vectors $X$ and $Y$ in Assumption \ref{assum_signal}, we assume that there exists some $d$-dimensional ($d \leq p$) subvectors,  $X_d \subset X$ and $Y_d \subset Y,$ and some differential map
$\vartheta: \mathbb{R}^d \rightarrow \mathbb{R}^d$ so that 
\begin{equation}\label{eq_orignalmap}
Y_d \simeq \vartheta(X_d),
\end{equation}
where $\simeq$ means equal in distribution. 
\end{assu}
The assumption above implies that $X$ and $Y$ are (partially) statistically related. As explained in Section \ref{manifold.sec} of our supplement, under the manifold model Assumption \ref{assum_commonstructureassumption}, they also share common geometric structures with respect to the underlying manifold models. These common structures can be learned via our newly introduced duo-landmark integral operators below.}}

{Before proceeding to introduce the duo-landmark integral operators, we introduce a model reduction scheme which largely generalizes the framework proposed in \cite{DW2}.} Let 
\begin{equation}\label{eq_covsetting}
	\operatorname{Cov}(X)=\bSig_1, \qquad  \operatorname{Cov}(Y)=\bSig_2. 
\end{equation}
Due to independence, we have that for $Z=X-Y,$ $\operatorname{Cov}(Z):= \bSig=\bSig_1+\bSig_2. $ Under Assumption \ref{assum_signal}, it follows that the rank $r=\operatorname{rank}(\bSig)$ satisfies $\max\{r_1, r_2\} \leq r \leq r_1+r_2$, and that  there exists a rotation matrix $\bR\in\R^{p\times p}$ such  that $\bR Z=(Z_1, Z_2,\cdots, Z_r, 0, \cdots, 0)^\top\in\R^p. $

Let the spectral decomposition of $\operatorname{Cov}(Z_1, \cdots, Z_r)$ be $\bSig_r=\bV \bTheta \bV^\top,$ where $\bTheta=\operatorname{diag}(\theta_1, \cdots, \theta_r)$ is a diagonal matrix of  the eigenvalues and $\bV$ contains the eigenvectors as its columns. For the orthogonal matrix $\bO\in\R^{p\times p}$ given by
$
\bO=\begin{bmatrix}
	\bV^\top & {\bf 0}\\
	{\bf 0} & {\bf I}_{p-r}
\end{bmatrix},
$
it holds that 
$
	\text{Cov}(\bO \bR Z) = \begin{bmatrix}
		\bTheta& {\bf 0}\\
		{\bf 0} & {\bf 0}
	\end{bmatrix}.
$
Based on the above discussion, since Euclidean distance is invariant to rotations, for (\ref{eq_clearnsignalkernel}), we have
\[
\exp\bigg(-\frac{\|\xb^0_i-\yb^0_j\|_2^2}{h^0_n} \bigg)=\exp\bigg(-\frac{\|\bO \bR\xb^0_i-\bO \bR\yb^0_j\|_2^2}{h^0_n} \bigg).
\]
Consequently, without loss of generality, we can directly use $\bO \bR\xb^0_i-\bO \bR\yb^0_j$ and focus on
\begin{equation}\label{eq_datareducedstructure}
	\zb^0_{ij}:=\bO \bR (\xb_i^0-\yb^0_j)=(z_{ij,1}, \cdots, z_{ij,r}, \mathbf{0}),
\end{equation}
with
\begin{equation}\label{eq_covstructure}
	\operatorname{Cov}(z_{ij,1}, \cdots, z_{ij,r})=\operatorname{diag}(\theta_1, \cdots, \theta_r). 
\end{equation}
Intuitively, the parameters $\{\theta_1,...,\theta_r\}$ are {jointly} determined by the probability measures $\widetilde{\mathbb{P}}_1$ and $\widetilde{\mathbb{P}}_2$ capturing the signal structures of the individual datasets, as well as the structural sharing between the two embedded manifolds $\iota_1(\mathcal{M}_1)$ and $\iota_2(\mathcal{M}_2)$.
Denote the sets 
\begin{equation}\label{eq_setstwo}
	\mathcal{S}_1=\left\{ \bO \bR \xb^0: \ \xb^0 \in \iota_1(\mathcal{M}_1)  \right\}, \qquad \mathcal{S}_2=\left\{ \bO \bR \yb^0: \ \yb^0 \in \iota_2(\mathcal{M}_2)  \right\}. 
\end{equation}
We define the probability measure on $\mathcal{S}_{1}$ for $\bO \bR X$ based on $\widetilde{\mathbb{P}}_1$ as $\widetilde{\mathsf{P}}_1,$ and define the measure on $\mathcal{S}_2$ for $\bO \bR Y$ based on  $\widetilde{\mathbb{P}}_2$ as $\widetilde{\mathsf{P}}_2.$ For random vectors $X'$ and $Y'$ respectively defined on $\mathcal{S}_{1}$ and $\mathcal{S}_{2}$ associated with $\widetilde{\mathsf{P}}_1$ and $\widetilde{\mathsf{P}}_2$, according to Assumption \ref{assum_commonstructureassumption}, it is easy to see that there exists some $d'$-dimensional ($d' \leq r$) subvectors $X_d'$ and $Y_d',$ and some differential  mapping $\vartheta': \mathbb{R}^{d'} \rightarrow \mathbb{R}^{d'}$ satisfying 
\begin{equation}\label{eq_reducedmapping}
Y_{d'} \simeq \vartheta'(X_{d'}). 
\end{equation}
Next we introduce the duo-landmark integral operators for the reduced model (\ref{eq_datareducedstructure}) on the sets (\ref{eq_setstwo}).
For notional simplicity, we set that for $\bm{x}, \bm{y} \in \mathbb{R}^p$  
\begin{equation}\label{eq_originalkernel}
	k(\bm{x}, \bm{y})=\exp\left(-\|\bm{x}-\bm{y} \|_2^2/h^0_n \right).
\end{equation}

{ Based on the kernel in (\ref{eq_originalkernel}), the joint manifold model in Assumption \ref{assum_signal} and the common structures in Assumption \ref{assum_commonstructureassumption}}, we now define a pair of new kernel functions as follows. { For $\ell\in\{1,2\}$, and for any $\bm{z} \in \mathcal{S}_{\ell},$ for notional simplicity, we write it as $\bm{z}=\begin{pmatrix} \bm{z}_{d'} \\ \bm{z}_{d'}^\perp\end{pmatrix},$ where $\bm{z}_{d'}$ corresponds to the $d'$-dimensional subvector, and  {assume the direct-sum decomposition $\mathcal{S}_{\ell}=\mathcal{S}_{\ell,d'} \oplus \mathcal{S}_{\ell,d'}^\perp$ hold, so   that $\bm{z}_{d'} \in \mathcal{S}_{\ell,d'}$ and $\bm{z}_{d'}^\perp \in \mathcal{S}_{\ell,d'}^\perp.$ Moreover, we denote $\widetilde{\mathsf{P}}_{\ell,d'}$ and $\widetilde{\mathsf{P}}^\perp_{\ell,d'}$ as the  independent marginal components of the measures $\widetilde{\mathsf{P}}_{\ell}$ in the subspaces $\mathcal{S}_{\ell,d'}$ and $\mathcal{S}_{\ell,d'}^\perp,$ respectively.} 



\begin{defn}[Convolutional landmark kernels]\label{defn_clmd}
	For $\bm{x}_1, \bm{x}_2 \in \mathcal{S}_1,$ we define the kernel $k_1(\cdot, \cdot): \mathcal{S}_1 \times \mathcal{S}_1 \rightarrow \mathbb{R}$ as
	\begin{equation}\label{eq_kernelone}
		k_1(\bm{x}_1, \bm{x}_2):= \int_{\mathcal{S}_{1,d'}} \int_{\mathcal{S}_{2,d'}^\perp} k\left(\bm{x}_1, \begin{pmatrix}
\vartheta'(\bm{z}_{d'}) \\ \bm{z}_{d'}^\perp		
\end{pmatrix} \right) k\left(\begin{pmatrix}
\vartheta'(\bm{z}_{d'}) \\ \bm{z}_{d'}^\perp		
\end{pmatrix} , \bm{x}_2\right) \widetilde{\mathsf{P}}^\perp_{2,d'}(\mathrm{d} \bm{z}_{d'}^\perp) \widetilde{\mathsf{P}}_{1,d'}(\mathrm{d} \bm{z}_{d'}) . 
	\end{equation}
	Similarly, for $\bm{y}_1, \bm{y}_2 \in \mathcal{S}_2,$ we define the kernel $k_2(\cdot,\cdot): \mathcal{S}_2 \times \mathcal{S}_2 \rightarrow \mathbb{R}$ as 
	\begin{equation}\label{eq_kerneltwo}
		k_2(\bm{y}_1, \bm{y}_2):= \int_{\mathcal{S}_{2,d'}} \int_{\mathcal{S}_{1,d'}^\perp} k\left(\bm{y}_1, \begin{pmatrix}
(\vartheta')^{-1}(\bm{z}_{d'}) \\ \bm{z}_{d'}^\perp		
\end{pmatrix} \right) k\left(\begin{pmatrix}
(\vartheta')^{-1}(\bm{z}_{d'}) \\ \bm{z}_{d'}^\perp		
\end{pmatrix} , \bm{y}_2\right) \widetilde{\mathsf{P}}_{1,d'}^\perp(\mathrm{d} \bm{z}_{d'}^\perp) \widetilde{\mathsf{P}}_{2,d'}(\mathrm{d} \bm{z}_{d'}). 
	\end{equation}
	We shall call the kernels $k_1(\cdot,\cdot)$ and $k_2(\cdot,\cdot)$ the convolutional landmark kernels. 
\end{defn}
}

{
{The above definitions hold only if $d' > 0$, meaning that the datasets share at least some partially common structures.} {As will be seen from the discussions in Section \ref{manifold.sec} of our supplement, the kernel functions can be directly related to the manifolds. The following proposition shows that if $d'>0,$ both kernels in (\ref{eq_kernelone}) and (\ref{eq_kerneltwo}) are positive definite so that the integral operators and reproducing kernel Hilbert spaces (RKHSs; see Section \ref{append_RMRKHS} for some background) can be constructed based on them accordingly. } Its proof can be found in Section \ref{sec_31proof} of the supplement.

{
\begin{prop}\label{prop_pdfkernel}
	Suppose Assumptions \ref{assum_signal} and \ref{assum_commonstructureassumption} hold. Recall (\ref{eq_reducedmapping}). If the kernels in Definition \ref{defn_clmd} are properly defined in the sense that $d'>0,$ then we can rewrite 
	\begin{equation}\label{eq_kernelone1111}
		k_1(\bm{x}_1, \bm{x}_2)=\int_{\mathcal{S}_2} k(\bm{x}_1, \bm{z}) k(\bm{z}, \bm{x}_2) \widetilde{\mathsf{P}}_2(\mathrm{d} \bm{z}), \ \
		k_2(\bm{y}_1, \bm{y}_2)=\int_{\mathcal{S}_1} k(\bm{y}_1, \bm{z}) k(\bm{z}, \bm{y}_2) \widetilde{\mathsf{P}}_1(\mathrm{d} \bm{z}). 
	\end{equation}
Moreover, the above kernels are bounded and positive definite. 
\end{prop}}

{ The above proposition shows that when $d'>0$, $k_1(\cdot, \cdot)$ can be understood as a convolution of two kernels using $\iota_2(\mathcal{M}_2)$ as a landmark population. Similar discussion applies to $k_2(\cdot,\cdot).$ }
Based on Proposition \ref{prop_pdfkernel}, we now introduce the duo-landmark integral operators.

\begin{defn}[Duo-landmark integral operators]\label{defn_landmarkintegral} { Suppose the assumptions of Proposition \ref{prop_pdfkernel} hold,} for $\ell\in\{1,2\}$ and $g_\ell \in \mathcal{L}_2(\mathcal{S}_\ell, \widetilde{\mathsf{P}}_\ell),$ we define the duo-landmark integral operators $\mathcal{K}_\ell$ over $\mathcal{L}_2(\mathcal{S}_\ell, \widetilde{\mathsf{P}}_\ell)$ as 
		\begin{equation*}
		\mathcal{K}_\ell g_\ell(\bm{z}):=\int_{\mathcal{S}_{\ell}} k_\ell(\bm{z}, \bm{y}) g_\ell(\bm{y}) \widetilde{\mathsf{P}}_\ell(\mathrm{d} \bm{y}).
	\end{equation*} 
\end{defn}
	Since $k(\cdot, \cdot)$ in (\ref{eq_originalkernel}) is a positive-define kernel, it is well-known  \citep{scholkopf2002learning} that  $\mathcal{L}_2(\mathcal{S}_1, \widetilde{\mathsf{P}}_1)$ and $\mathcal{L}_2(\mathcal{S}_2, \widetilde{\mathsf{P}}_2)$ equipped  with the kernel (\ref{eq_originalkernel}) can be made into two RKHSs, respectively, denoted as $\mathcal{H}_1$ and $\mathcal{H}_2$. In addition, the two RKHSs are closely related to two other integral operators $\mathcal{G}_\ell, \ell=1,2,$ given by 
	\begin{equation}\label{eq_initialintegraloperators}
		\mathcal{G}_\ell g_\ell(\bm{x}):=\int_{\mathcal{S}_\ell} k(\bm{x},\bm{y}) g_\ell(\bm{y}) \widetilde{\mathsf{P}}_\ell(\mathrm{d} \bm{y}),  \qquad  g_\ell \in \mathcal{L}_2(\mathcal{S}_\ell, \widetilde{\mathsf{P}}_\ell). 
	\end{equation}
	Furthermore, $k(\cdot,\cdot)$ can be decomposed using Mercer's theorem (see Theorem 2.10 of \cite{scholkopf2002learning}) in the sense that for $\ell=1,2,$ $k(\bm{x}, \bm{y})= \sum_{j} \lambda^\ell_j \psi_j^\ell(\bm{x}) \psi_j^\ell(\bm{y}),$ where $\{\lambda_j^\ell\}$ and $\{\psi_j^\ell(\cdot)\}$ are the eigenvalues and eigenfunctions of the integral operators $\mathcal{G}_\ell$ in (\ref{eq_initialintegraloperators}). 
	
	Similarly, as in Definition \ref{defn_landmarkintegral}, one can build RKHSs using the kernels (\ref{eq_kernelone}) and (\ref{eq_kerneltwo}) on $\mathcal{L}_2(\mathcal{S}_1, \widetilde{\mathsf{P}}_1)$ and $\mathcal{L}_2(\mathcal{S}_2, \widetilde{\mathsf{P}}_2).$  Using Mercer's theorem, one can also expand the kernels $k_1(\cdot, \cdot)$ and $k_2(\cdot, \cdot).$ {As demonstrated in the proof of Proposition \ref{prop_pdfkernel}, the kernel  $k_1$ is closely related to $\mathcal{G}_2$ in the sense that one can construct an RKHS for $\mathcal{G}_2$ such that the eigenfunctions of $k_1$ are identical to those of the operator $\mathcal{G}_2$, while the eigenvalues of $k_1$ are squares of those of $\mathcal{G}_2$. The kernel $k_2$ is closely related to the operator $\mathcal{G}_1$ in the same way. From the above discussion, we can see the difference between {the manifold-specific operators $\mathcal{G}_\ell, \ell=1,2,$ in (\ref{eq_initialintegraloperators}) and the duo-landmark operators $\mathcal{K}_\ell, \ell=1,2,$ in Definition \ref{defn_landmarkintegral} incorporating both manifolds} --  the former operators only utilize the information from one dataset whereas the latter operators aim at integrating the information from two datasets. We also refer to Remark \ref{rem_algo1} of our supplement for more discussions.}

{	 Finally, we point out that the validity of Definitions \ref{defn_clmd} and \ref{defn_landmarkintegral} and Proposition \ref{prop_pdfkernel} rely on the existence of common structures between the two datasets, i.e., $d'>0$ in (\ref{eq_reducedmapping}).  In particular, Step 1 of our Algorithm \ref{al0} is designed to practically to check such an assumption. }

\subsection{Spectral convergence analysis}\label{sec_spectralanalysis}

In this section, we  show that when both datasets are clean signals (i.e., $\xb_i=\xb_i^0$ and $\yb_j=\yb_j^0$ in (\ref{model})), the embeddings obtained from Algorithm \ref{al0} have close relations with the eigenfunctions of the operators $\mathcal{K}_1$ and $\mathcal{K}_2$ in Definition \ref{defn_landmarkintegral}. More specifically,  the singular values and vectors of the matrix $\bK^0$ in (\ref{eq_clearnsignalkernel}) are closely related to the eigenvalues and eigenfunctions of the operators. Before stating our results, we first provide an important result on the eigenvalues of the operators whose proof can be found in Section \ref{sec_32proof} of the supplement. Recall that  $\mathcal{H}_1$ and $\mathcal{H}_2$ are the RKHSs built on  $\mathcal{L}_2(\mathcal{S}_1, \widetilde{\mathsf{P}}_1)$ and $\mathcal{L}_2(\mathcal{S}_2, \widetilde{\mathsf{P}}_2)$ separately  using the kernel (\ref{eq_originalkernel}).

\begin{prop} \label{eigenvalue.prop}
{	Suppose the assumptions of Proposition \ref{prop_pdfkernel} hold. Moreover, we assume that there exists  some common measurable space $\mathcal{N} \subset \mathbb{R}^p$ that $\mathcal{H}_i, i=1,2,$ are compactly embedded  into $\mathcal{L}_2(\mathcal{N}, \widetilde{\mathsf{P}}_i).$  Then we have that operators $\mathcal{K}_1$ and $\mathcal{K}_2$ have the same set of nonzero eigenvalues.} 
\end{prop}

\begin{rem}
We provide two remarks on Proposition \ref{eigenvalue.prop}. First, our proposed Algorithm \ref{al0} provides the joint embeddings by using the singular values and vectors of the matrix in (\ref{eigen_decomp}). Therefore, if the two embeddings are associated with some operators, it is necessary that these two operators must share the same nonzero eigenvalues. In this sense, Proposition \ref{eigenvalue.prop} is the basis of our theoretical analysis. Second, the assumption that $\mathcal{H}_i, i=1,2,$ are compactly embedded into the $\mathcal{L}_2$ spaces is a mild assumption and used frequently in the literature  \citep{hou2023sparse, JMLR:v17:14-023, steinwart2012mercer}. It guarantees that the orthonormal basis of the RKHSs and the $\mathcal{L}_2$ spaces are closely related; see Lemma \ref{lem_RKHSextension} for more precise statements. { Roughly speaking, it requires that the two manifolds have common structures so that the RKHSs built on them share common information (i.e., $d'>0$ in (\ref{eq_reducedmapping})). }
\end{rem}

Based on Proposition \ref{eigenvalue.prop}, we denote $\{\gamma_i\}$ as the nonzero eigenvalues of $\mathcal{K}_1$ and $\mathcal{K}_2$ in the decreasing order. Moreover, we denote $\{\phi_i(\bm{x})\}$ and $\{\psi_i(\bm{y})\}$ as the eigenfunctions of $\mathcal{K}_1$ and $\mathcal{K}_2,$ respectively. That is for $\bm{x} \in \mathcal{S}_1$ and $\bm{y} \in \mathcal{S}_2$
\begin{equation}\label{eq_eigefunctiondefinition}
	\mathcal{K}_1 \phi_i(\bm{x})=\gamma_i \phi_i(\bm{x}),  \quad  \ \mathcal{K}_2 \psi_i(\bm{y})=\gamma_i \psi_i(\bm{y}). 
\end{equation}
Moreover, for $\bK^0=(K^0(i,j))$ defined in (\ref{eq_clearnsignalkernel}), we denote 
\begin{equation}\label{eq_twomatrices}
	\bN_{01}=\frac{1}{n_1n_2} \bK^0 (\bK^0)^\top, \quad \bN_{02}=\frac{1}{n_2n_1} (\bK^0)^\top\bK^0. 
\end{equation}
It is clear that $\bN_{01}$ and $\bN_{02}$ share the same nonzero eigenvalues. We denote them as $\{\lambda_i\}$.  Moreover, we denote the associated eigenvectors of $\bN_{01}$ and $\bN_{02}$ as $\{\bu_i^0\}$ and $\{\bv_i^0\},$ respectively. As such, $\{\lambda^{1/2}_i
\}$ are the nonzero singular values of $(n_1 n_2)^{-1/2} \bK^0$, whereas $\{\bu_i^0\}$ and $\{\bv_i^0\}$  are  the left and right singular vectors of $(n_1 n_2)^{-1/2} \bK^0.$ To better describe the convergence of $\{\bu_i^0\}$ and $\{\bv_i^0\}$, for the kernel function defined in (\ref{eq_originalkernel}), we further denote 
\begin{equation}\label{eq_k01k02definition}
	\widehat{k}^0_1(\bm{x}_1, \bm{x}_2)=\frac{1}{n_2}\sum_{s=1}^{n_2} k(\bm{x}_1, \yb^0_s) k(\yb^0_s, \bm{x}_2), \qquad \widehat{k}^0_2(\bm{y}_1, \bm{y}_2)=\frac{1}{n_1}\sum_{s=1}^{n_1} k(\bm{y}_1, \xb^0_s) k(\xb^0_s, \bm{y}_2). 
\end{equation}
and define the functions
\begin{equation}\label{eq_empericialeigenfunction}
	\widehat{\phi}^0_{i}(\bm{x})=\frac{1}{\lambda_i\sqrt{n_1}} \sum_{j=1}^{n_1} \widehat{k}^0_1(\bm{x}, \xb^0_j) u^0_{ij}, \qquad \widehat{\psi}^0_i(\bm{y})=\frac{1}{\lambda_i\sqrt{n_2}} \sum_{j=1}^{n_2}\widehat{k}^0_2(\bm{y}, \yb^0_j) v^0_{ij}, 
\end{equation} 
where $\ub^0_i=(u^0_{i1}, \cdots, u^0_{in_1})^\top$ and $\vb^0_i=(v^0_{i1}, \cdots, v^0_{in_2})^\top.$ By the above construction, we have that $\widehat{\phi}^0_i(\xb^0_j)=\sqrt{n_1}u^0_{ij}, \qquad \widehat{\psi}^0_i(\yb^0_j)=\sqrt{n_2}v^0_{ij}.$

 {In what follows, we always assume that $\{\gamma_i\}$ and $\{\lambda_i\}$ are sorted in a nonincreasing order with multiplicity}. 
For some constant $\delta>0,$ we define 
\begin{equation}\label{eq_defni}
	\mathsf{K} \equiv \mathsf{K}(\delta):=\operatorname{arg \ max}\left\{ i: \gamma_i \geq \delta \right\},
\end{equation}  \vspace{-3pt}
and {for each $1 \leq i \leq \mathsf{K}$, we define the index set ${\sf I}\equiv {\sf I}(i)$ such that $i\in {\sf I} \subset \{1,2,\cdots, \min\{n_1,n_2\}\}$ and for any $t\in {\sf I}$ we have $\gamma_t \equiv \gamma_i$ and 
	\begin{equation}\label{defn_sfr}
		\mathsf{r}_i:= \min_{j\in{\sf I}^c}|\gamma_i-\gamma_j|,
	\end{equation}
	 where ${\sf I}^c=\mathbb{N}\setminus {\sf I}$. We define the vector-valued functions $\bm{\phi}_{\sf I}(\bm{x}):=(\phi_i(\bm{x}))_{i\in \sf I}$ and $\widehat{\bm{\phi}}^0_{\sf I}(\bm{x}):=(\widehat{\phi}^0_i(\bm{x}))_{i\in \sf I}$, and similarly $\bm{\psi}_{\sf I}(\bm{x})$ and $\widehat{\bm{\psi}}^0_{\sf I}(\bm{x})$. 

\begin{thm}\label{thm_cleanconvergence}
	Suppose the assumptions of Proposition \ref{eigenvalue.prop} hold. Then we have that when $n_1$ and $n_2$ are sufficiently large,
	\begin{equation}\label{v.conv0}
		\sup_i|\lambda_i-\gamma_i| \prec n_1^{-1/2}+n_2^{-1/2}. 
	\end{equation}
	Moreover,  for each $1\le i\le 	\mathsf{K}$ and $\sfr_i$ in (\ref{defn_sfr}) satisfying that for some small constant $\tau>0$ 
	\begin{equation}\label{eq_ricondition}
		n_1^{-1/2+\tau}+n_2^{-1/2+\tau}=\mathrm{o}(\sfr_i),
	\end{equation}
	we have that, for any $\bm{x}\in \mathcal{S}_1$, 
	\beq\label{eigenvector_leftspectral}
	\inf_{\bO\in O({\sf I})}\left\| \sqrt{\gamma_i} {\bm \phi}_{\sf I}(\bm{x}) -\sqrt{\lambda_i}\bO \widehat{\bm \phi}_{\sf I}^{0}(\bm{x}) \right\|_\infty	
	\prec \frac{1}{\sfr_i}\bigg(\frac{1}{\sqrt{n_1}}+\frac{1}{\sqrt{n_2}}\bigg)+\frac{1}{\sqrt{n_2}},
	\eeq
	and for any  $\bm{y}\in \mathcal{S}_2$,  
	\beqs
		\inf_{\bO\in O({\sf I})}\left\| \sqrt{\gamma_i} {\bm \psi}_{\sf I}(\bm{y}) -\sqrt{\lambda_i}\bO \widehat{\bm \psi}_{\sf I}^{0}(\bm{y}) \right\|_\infty
\prec \frac{1}{\sfr_i}\bigg(\frac{1}{\sqrt{n_1}}+\frac{1}{\sqrt{n_2}}\bigg)+\frac{1}{\sqrt{n_1}},
	\eeqs
	where $\mathcal{S}_1$ and $\mathcal{S}_2$ are defined in (\ref{eq_setstwo}). 
\end{thm}
{The technical proof is provided in Section \ref{sec_32proof} of the supplement. For further discussion of the technical details, see Remark \ref{rem_proofremark} therein.} We provide several remarks on Theorem \ref{thm_cleanconvergence}. First, it establishes the spectral convergence results of the matrices in (\ref{eq_twomatrices}) to the duo-landmark integral operators in Definition \ref{defn_landmarkintegral} under the clean signals. We now explain the high-level heuristics using the embeddings for $\{\xb_i^0\}$ as an example. Note that for all $1 \leq i, j \leq n_1,$ by law of large number, roughly speaking, we have that for the kernel in (\ref{eq_kernelone})
\begin{align}\label{eq_convulutionpart}
	n_1(\bN_{01})_{ij} =\frac{1}{n_2} \sum_{s=1}^{n_2} k(\xb^0_i, \yb^0_s) k(\yb^0_s, \xb^0_j) 
	\approx \int_{\mathcal{S}_2} k(\xb_i^0, \bm{z}) k(\bm{z}, \xb_j^0) \widetilde{\mathsf{P}}_2(\mathrm{d} \bm{z})\equiv k_1(\xb_i^0, \xb_j^0).
\end{align}
Consequently, if one defines $\bW_{01}=(W_{01}(i,j)) \in \mathbb{R}^{n_1 \times n_1}$ with $W_{01}(i,j)=k_1(\xb_i^0, \xb_j^0),$ similar to the argument in \cite{BJKO,JMLR:v11:rosasco10a}, one can see that {$n_1^{-1} \bW_{01}$ is a finite-sample matrix version of the operator $\mathcal{K}_1$ defined via the kernel $k_1$ in (\ref{eq_kernelone})}. This explains the closeness between the matrix $\bN_{01}$ in (\ref{eq_twomatrices}) and the operator $\mathcal{K}_1.$ It is worth pointing out that the scaling $n_2^{-1}$ is needed for the convergence of the convolution part in (\ref{eq_convulutionpart}) whereas $n_1^{-1}$ is related to the dimension of $\bW_{01}$, essentially ensuring the convergence of $\bW_{01}.$

{Second, we provide some insights on the differences in constructing embeddings between only using a single dataset and with the help of the landmark dataset. We use the dataset $\{\xb_i^0\}$ as an example. Note that after some simple algebraic manipulation, one can rewrite (\ref{eq_convulutionpart}) as follows 
\begin{equation*}
	k_1(\xb_i^0, \xb_j^0)=\mathsf{g}(\xb^0_i, \xb^0_j, \{\yb_i^0\})k(\xb_i^0, \xb_j^0),
\end{equation*}
for some function $\mathsf{g}.$ On the one hand, when one only uses the single dataset $\{\xb_i^0\}$, it is easy to see that $\mathsf{g} \equiv 1$ so that $\xb_i^0$ and $\xb_j^0$ are connected using the kernel $k.$ On the other hand, our proposed algorithm utilizes both $\{\xb_i^0\}$ and the landmark dataset $\{\yb_i^0\}$ so that $\xb_i^0$ and $\xb_j^0$ will be connected via the kernel $k_1,$ which is related to the global kernel $k$ but  with local modifications.}

{ Finally,  even though the convergence rates in Theorem \ref{thm_cleanconvergence} depend on both the sample sizes $n_1$ and $n_2$, our theoretical results do not rely on any particular assumptions on the relations between $n_1$ and $n_2$, allowing for sample imbalance. Moreover, we do not require any particular structures and relations between the two underlying manifolds.  This justifies the flexibility of our algorithm for analyzing real-world single-cell omics data, as demonstrated in Section \ref{sec_realdataanalysis}.} }   

\subsection{Performance in presence of high-dimensional noise}\label{sec_robustness}

\subsubsection{Robustness analysis}
In this section, we  study the convergence of our algorithm for the noisy datasets (\ref{model}). Throughout, we make the following assumptions on the random noise $\{\bm{\xi}_i\}$ and $\{\bm{\zeta}_i\}$. 

\begin{assu} \label{assum_mainassumption}
	For some constants $\sigma_1, \sigma_2>0,$ we assume $\{\bm{\xi}_i\}$ and $\{\bm{\zeta}_j\}$ are  independent  sub-Gaussian random vectors {containing independent components} such that
	\begin{equation}\label{eq_noisevariance}
		\mathbb{E}\bm{\xi}_i=\mathbb{E} \bm{\zeta}_j=\bm{0}; \qquad \operatorname{Cov}(\bm{\xi}_i)=\sigma_1^2 \mathbf{I}, \qquad  \operatorname{Cov}(\bm{\zeta}_j)=\sigma_2^2 \mathbf{I},
	\end{equation}
	for all $1\le i\le n_1$ and $1\le j\le n_2$.
\end{assu}

{The above assumption requires that the noise vectors are independent centered sub-Gaussian random vectors with isotropic covariance, with the respective noise magnitude controlled by parameters $\sigma_1^2$ and $\sigma_2^2$. Now by the same argument around (\ref{eq_datareducedstructure}), analyzing the duo-landmark kernel matrix $\bK$ defined in (\ref{eq_kernelmatrix}) under noisy observations (\ref{model}) can be simplified if we replace $(\xb_i-\yb_j)$'s in $\bK$ by their equivalent, orthogonal transformed vectors $
{\bf OR}(\xb_i-\yb_j)={\bf OR}(\xb_i^0-\yb_j^0)+{\bf OR}(\bxi_i-\bzeta_j)$.
By Assumption \ref{assum_mainassumption}, the noise component $\{{\bf OR}(\bxi_i-\bzeta_j)\}$ of $\{{\bf OR}(\xb_i-\yb_j)\}$ are also centered sub-Gaussian random vectors with covariance matrix $\sigma^2{\bf I}_p$, where $\sigma^2=\sigma_1^2+\sigma_2^2$. Note that by the property of sub-Gaussian vectors (c.f. Lemma \ref{sg.bnd.lem} of the supplement), the norm of ${\bf OR}(\bxi_i-\bzeta_j)$ is bounded by $\mathrm{O}(p\sigma^2)$ with high probability.}
We will also need the following assumption on the high dimensionality and global SNRs.

\begin{assu}\label{assum_dimensionalityandsnr}
	For $n=\min\{n_1, n_2\},$ we assume that $n$ is sufficiently large. Moreover,  we assume that there exist some positive constants $\beta_1$ and $\beta_2$ so that $n_1 \asymp n^{\beta_1}$ and $n_2 \asymp n^{\beta_2}.$ Additionally, we suppose that there exist some nonnegative constants  
	$\beta, \upsilon_1, \upsilon_2$ so that  $p \asymp n^{\beta}$ and for $i=1,2,$  $\sigma_i^2 \asymp n^{\upsilon_i}.$
	Moreover, for $\sigma^2=\sigma_1^2+\sigma_2^2$ and $\{\theta_i \equiv \theta_i(n)\}$ in (\ref{eq_covstructure}),  we assume that 
	\begin{equation}\label{eq_sigmaimagnititude}
		\frac{p \sigma^2}{\sum_{i=1}^r \theta_i }=\mathrm{o}\left(1\right).
	\end{equation}
\end{assu}

\begin{rem}\label{rem_snsnsnsnsndiscussion}
{	We provide two remarks. First, based on Assumption \ref{assum_mainassumption}, we can interpret $p \sigma^2$ as the overall noise level. Moreover, according to the discussions around (\ref{eq_datareducedstructure}) and (\ref{eq_covstructure}), $\sum_{i=1}^r \theta_i$ should be understood as the signal strength. Consequently, (\ref{eq_sigmaimagnititude}) imposes a condition on the SNR. That is, we require the overall signals are relatively stronger than the noise. Second, we explain the potential advantage of using our algorithm with two datasets as compared to using only a single dataset. For simplicity, we assume that $r_1=r_2=1$ in Assumption \ref{assum_signal}. Moreover, we assume that in (\ref{eq_covsetting}), $\operatorname{Var}(X)=\mathsf{a}_1$ and $\operatorname{Var}(Y)=\mathsf{a}_2.$ According to \cite{DW2}, if one uses a single dataset, say $\{\xb_i\},$ it requires that $\mathsf{a}_1 \gg p \sigma_1^2.$ However, if we utilize both datasets, the condition in (\ref{eq_sigmaimagnititude}) reads that $\mathsf{a}_1+\mathsf{a}_2 \gg p(\sigma_1^2+\sigma_2^2).$ We consider a setting that $\sigma_1^2 \asymp \sigma_2^2 \asymp 1. $ On the one hand, when $\mathsf{a}_1=\sqrt{p},$ we will not be able to learn useful information from the dataset $\{\xb_i\}$ with the commonly used integral or differential operators as in \cite{ding2022learning, DW2}. On the other hand, if we are able to find a dataset $\{\yb_i\}$ with better quality, for instance $\mathsf{a}_2=p^{1+\delta}-\sqrt{p}$ for some constant $\delta>0,$ then $\mathsf{a}_1+\mathsf{a}_2=p^{1+\delta} \gg p$ will enable us to better understand $\{\xb_i\}$ via the operator $\mathcal{K}_1$ and the dataset $\{\yb_i\}.$ This is empirically observed in Section \ref{sec_simulation} and Appendix \ref{supp.sec.nu}.}
\end{rem}

Analogous to (\ref{eq_twomatrices}), we denote  
\begin{equation}\label{eq_twomatricesnoisy}
	\bN_{1}=\frac{1}{n_1n_2} \bK \bK^\top, \ \bN_{2}=\frac{1}{n_2n_1} \bK^\top\bK. 
\end{equation}
We denote the nonzero eigenvalues of $\bN_1$ and $\bN_2$ in the decreasing order as $\{\mu_i\}$ (note that $\mu_i=s_i^2$), the associated eigenvectors of $\bN_{1}$ and $\bN_{2}$ as $\{\bu_i\}$ and $\{\bv_i\},$ which are also the left and right singular vectors of $(n_1 n_2)^{-1/2} \bK.$ Denote the control parameter $\eta \equiv \eta_n$ as follows 
\begin{equation}\label{psi}
	\eta \equiv \eta_n:=\frac{p \sigma^2}{\sum_{i=1}^r \theta_i}+\frac{\sigma}{\left(\sum_{i=1}^r \theta_i \right)^{1/2}}. 
\end{equation}

\begin{thm}\label{thm_noiseconvergence}
	Suppose Assumptions \ref{assum_signal}, \ref{assum_mainassumption} and \ref{assum_dimensionalityandsnr} hold. Recall that $\{\lambda_i\}, \{\ub_i^0\}, \{\vb_i^0\}$ are defined according to the matrices in (\ref{eq_twomatrices}).
	For the eigenvalues, we have that for $\eta$ defined in (\ref{psi})
	\begin{equation}\label{eq_eigs_v}
		\sup_i|\mu_i-\lambda_i| \prec \eta. 
	\end{equation}
	For the eigenvectors, if we further assume $\sfr_i$ in (\ref{defn_sfr})
	satisfies that for some small constant $\tau>0$
	\begin{equation}\label{eq_assumption}
		n^{\tau} \left( 	n_1^{-1/2}+n_2^{-1/2}+\eta \right)=\mathrm{o}(\sfr_i),
	\end{equation} 
{	then, for $\bU_{\sf I}=(\bu_i)_{i\in \sf I}$, $\bU^0_{\sf I}=(\bu^0_i)_{i\in \sf I}\in\R^{n_1\times |\sf I|}$, and $\bV_{\sf I}=(\bv_i)_{i\in \sf I}$, $\bV^0_{\sf I}=(\bv^0_i)_{i\in \sf I}\in\R^{n_2\times |\sf I|}$, we have 
	\begin{equation}\label{eq_projectionbound}
		\max\left\{	\inf_{\bO\in O(|\sf I|)}\| \bU_{\sf I}-\bU^0_{\sf I}\bO \|, \inf_{\bO\in O(|\sf I|)}\| \bV_{\sf I}-\bV^0_{\sf I}\bO \|\right\}=\OO_{\prec}\left( \frac{\eta}{\sfr_i}\right).
	\end{equation} }
\end{thm}
{The technical proof is provided in Section \ref{sec_proof33} of the supplement. For further discussion of the technical details, see Remark \ref{rem_proofremark} therein.} Theorem \ref{thm_noiseconvergence} implies that our proposed Algorithm \ref{al0} is robust against high-dimensional noise in the sense that $(n_1n_2)^{-1/2} \bK$ is close to $(n_1 n_2)^{-1/2} \bK^0,$ once (\ref{eq_sigmaimagnititude}) is satisfied. Combining Theorem \ref{thm_noiseconvergence} with Theorem \ref{thm_cleanconvergence}, we can establish the results concerning the convergence of the matrices (\ref{eq_twomatricesnoisy}) to the duo-landmark integral operators. We first prepare some notations. For $\bm{x}_1, \bm{x}_2 \in \mathcal{S}_1$ and $\bm{y}_1, \bm{y}_2 \in \mathcal{S}_2,$ we denote
\begin{equation}\label{eq_approximatekernels}
	\widehat{k}_1({\bm x}_1,{\bm x}_2)=\frac{1}{n_2}\sum_{s=1}^{n_2} k_n({\bm x}_1, \yb_s) k_n(\yb_s, {\bm x}_2), \  \widehat{k}_2({\bm y}_1, {\bm y}_2)=\frac{1}{n_1}\sum_{s=1}^{n_1} k_n({\bm y}_1, \xb_s) k_n(\xb_s,{\bm y}_2),
\end{equation}
where $k_n$ is the same as the kernel function in (\ref{eq_originalkernel}) except that $h_n^0$ is replaced by $h_n.$ Based on the above kernels, we further denote
\begin{equation*}
	\widehat{\phi}_{i}(\bm{x})=\frac{1}{\mu_i\sqrt{n_1}} \sum_{j=1}^{n_1} \widehat{k}_1(\bm{x}, \xb_j) u_{ij}, \ \widehat{\psi}_i(\bm{y})=\frac{1}{\mu_i\sqrt{n_2}} \sum_{j=1}^{n_2}\widehat{k}_2(\bm{y}, \yb_j) v_{ij}, \
\end{equation*} 
where $\ub_i=(u_{i1}, \cdots, u_{in_1})^\top$ and $\vb_i=(v_{i1}, \cdots, v_{in_2})^\top.$ By the above construction, we have that
\begin{equation}\label{eq_interptation}
	\widehat{\phi}_i(\xb_j)=\sqrt{n_1}u_{ij}, \ \widehat{\psi}_i(\yb_j)=\sqrt{n_2}v_{ij}. 
\end{equation}

{\begin{cor}\label{cor_finalconvergence}
	Suppose the assumptions of Theorems \ref{thm_cleanconvergence} and \ref{thm_noiseconvergence} hold. {Moreover, we assume that neither $\mathbf{x}_i$ nor $\mathbf{y}_j$ in (\ref{model}) is purely noise.} Then we have that for the eigenvalues
	\begin{equation} \label{v.conv}
		\sup_i |\mu_i-\gamma_i|={ \OO_{\prec}\left( \frac{1}{\sqrt{n_1}}+\frac{1}{\sqrt{n_2}}+\eta\right).}
	\end{equation}
	Furthermore, for the eigenfunctions, we have that for $\mathsf{K}$ in (\ref{eq_defni}) and $ 1 \leq i \leq \mathsf{K}$ and any  $\bm{x} \in \mathcal{S}_1$, 
	\begin{equation}\label{eq_lefteigenfunction}
			\inf_{\bO\in O({\sf I})}\left\| \sqrt{\gamma_i} {\bm \phi}_{\sf I}(\bm{x}) -\sqrt{\mu_i}\bO \widehat{\bm \phi}_{\sf I}(\bm{x}) \right\|_\infty
			\prec \frac{1}{\sfr_i}\bigg(\frac{1}{\sqrt{n_1}}+\frac{1}{\sqrt{n_2}} \bigg)+\frac{1}{\sqrt{n_2}}+\eta+\sqrt{\frac{\eta}{\sfr_i}}, 
	\end{equation}
	and for any $\bm{y} \in \mathcal{S}_2$, 
	\begin{equation*}
			\inf_{\bO\in O({\sf I})}\left\| \sqrt{\gamma_i} {\bm \psi}_{\sf I}(\bm{y}) -\sqrt{\mu_i}\bO \widehat{\bm \psi}_{\sf I}(\bm{y}) \right\|_\infty
			\prec\frac{1}{\sfr_i}\bigg(\frac{1}{\sqrt{n_1}}+\frac{1}{\sqrt{n_2}} \bigg)+\frac{1}{\sqrt{n_1}}+\eta+\sqrt{\frac{\eta}{\sfr_i}}.
	\end{equation*} 
\end{cor}}
Corollary \ref{cor_finalconvergence} establishes the convergence of noisy matrices in (\ref{eq_twomatricesnoisy}) to the duo-landmark integral operators. {Our analysis suggests that the performance of the proposed method may be primarily determined by the sample size of the smaller dataset; see Figure \ref{supp.gamma} in Section \ref{supp.sec.nu} of the supplement for numerical evidence.  Especially, for the eigenvectors, together with 
(\ref{eq_interptation}) and (\ref{eigen_decomp}), we see that for any $1 \leq i \leq \mathsf{K}$, if $\gamma_i$ has multiplicity one, then  up to a possible sign change of the whole vector,
\begin{equation*}
	s_i \ub_i \approx \frac{\sqrt{\gamma_i}}{\sqrt{n_1}}(\phi_i(\xb_1), \cdots, \phi_i(\xb_{n_1}))^\top, \quad	s_i \vb_i \approx \frac{\sqrt{\gamma_i}}{\sqrt{n_2}}(\psi_i(\yb_1), \cdots, \psi_{i}(\yb_{n_2}))^\top,\quad s_i=\sqrt{\mu}_i.
\end{equation*}
This implies that as long as the index sets $\Gamma_1, \Gamma_2 \subset \{1,2,\cdots, \mathsf{K}\},$ the joint embeddings of $(\sqrt{n_1}\bU_{\Gamma_1} \Lambda_{\Gamma_1}, \sqrt{n_2}\bU_{\Gamma_2} \Lambda_{\Gamma_2})$ is approximately the eigenfunctions $(\phi_i, \psi_i)$ evaluated at the rotated data points and weighted by the eigenvalues.} In other words, the coordinates of final embeddings are nonlinear transforms of the original datasets, related to the operators in Definition \ref{defn_landmarkintegral}. 

{ \subsubsection{Phase transition with SNR and practical consideration}\label{sec_phasetransition}
As discussed in the previous subsection, our proposed algorithm will be robust to the high dimensional noise when the SNR satisfies the assumption of (\ref{eq_sigmaimagnititude}), i.e., $\sum_{i=1}^r \theta_i \gg p \sigma^2.$ In this section, we discuss how to check this assumption via our algorithm using the spectrum of $\bN_1$ (or $\bN_2$) in (\ref{eq_twomatricesnoisy}). Roughly speaking, in the robust regime when (\ref{eq_sigmaimagnititude}) holds that the signal dominates the noise, the spectrum of the outputs of our algorithm are closely related to those of the duo-landmark integral operators. In contrast, when the noise dominates the signal that $\sum_{i=1}^r \theta_i \ll p \sigma^2,$ the spectrum can be explained via random matrix theory using the free multiplicative convolution of two Marchenco-Pastur type laws which differs significantly from the duo-landmark integral operators. Especially, the bulk eigenvalues are rather rigid. These information can be potentially used to check whether the robustness assumption holds.

We now state our results whose proof can be found in Section \ref{sec_proof33} of the supplement. Since in our applications, the dimension $p$ (i.e., the number of features) is usually comparable to or much larger than the sample sizes (i.e., the number of cells), we focus on the setup that $p  \succsim \max\{n_1, n_2\}.$ That is, for $\beta_1, \beta_2, \beta$ in Assumption \ref{assum_dimensionalityandsnr}, we focus on 
\begin{equation}\label{eq_detaileddimensionregime}
\beta \geq \max\{\beta_1, \beta_2\}. 
\end{equation} 
Denote 
\begin{equation}\label{eq_sunknown}
\mathsf{s}:=\frac{\sqrt{n_1 n_2}}{4p \sigma_1^2 \sigma_2^2} \exp\left(-\frac{2 p\sigma^2}{h_n}\right),
\end{equation}
and the nonzero eigenvalues of $\mathsf{s}\bN_1$ (or equivalently $\mathsf{s} \bN_2$) as $\{w_i\},$ and  
$\mathsf{n}=\min\{n_1, n_2, p\}.$ Denote the empirical spectral distribution (ESD) of $\{w_i\}$ be 
\begin{equation}\label{eq_muesddefn}
\mu_n(x)=\frac{1}{\mathsf{n}}\sum_{i=1}^{\mathsf{n}} \mathbf{1}(w_i \leq x). 
\end{equation}

Let $\mu_{\SMP_k}, k=1,2,$ be  the scaled Marchenco-Pastur laws with parameters $\phi_k=p/n_k$ in the sense that \citep{bloemendal2014isotropic} 
\begin{equation}\label{eq_MPlawforms}
\mathrm{d} \mu_{\SMP_k}(x)= \frac{\sqrt{\phi_k}}{2 \pi} \frac{\sqrt{[(x-\gamma_{k,-})(\gamma_{k,+}-x)]_+}}{x} \mathrm{d} x, \ \gamma_{k,\pm}=\sqrt{\phi_k}+\frac{1}{\sqrt{\phi_k}}\pm 2. 
\end{equation}
Moreover, we denote $\mu_{\SMP_1 \boxtimes \SMP_2}$ as the free multiplicative convolution of $\mu_{\SMP_1}$ and $\mu_{\SMP_2};$ see Section \ref{sec_backgroundinrmt} of our supplement for some background in random matrix theory and Definition \ref{defn_freemultiplicativeconvuliton}  for the definition of free multiplicative convolution of two measures. 
\begin{thm}\label{thm_noisededuction}
Suppose Assumption \ref{assum_mainassumption}  and the assumption of (\ref{eq_detaileddimensionregime}) hold. Moreover, we assume
	\begin{equation}\label{eq_weaksnr}
		\frac{\sum_{i=1}^r \theta_i }{p \sigma^2}=\mathrm{o}\left(1\right).
	\end{equation}
%
Then we have that for the ESD in (\ref{eq_muesddefn}), we have that $\mu_n \Rightarrow \mu_{\SMP_1 \boxtimes \SMP_2},$ where $\Rightarrow$ means converge weakly in distribution. 
\end{thm}

Several remarks are in order. First, unlike in the strong SNR regime described in Corollary \ref{cor_finalconvergence}, where the significant eigenvalues can be characterized by the spectrum of duo-landmark integral operators, Theorem \ref{thm_noisededuction} shows that when noise dominates the signal, the bulk of the eigenvalues follows the free multiplicative convolution of two Marchenko–Pastur (MP) laws. Second, as seen in the proof, $\mathbf{N}_1$ (or $\mathbf{N}_2$) can be viewed as finite-rank perturbations of products of two sample covariance matrices constructed purely from noise. The low-rank component arises from the kernel effect, which does not encode any meaningful information from the data, while the bulk eigenvalues are determined by the product of the sample covariance matrices. This ensures that the algorithms do not produce artificial structures. Finally, combined with the rigidity property established in Lemma \ref{lem_free}, Theorem \ref{thm_noisededuction} can be used to verify the assumption in (\ref{eq_sigmaimagnititude}) via a nonparametric approach, without estimating the unknown quantity in (\ref{eq_sunknown}). The key idea is that the bulk eigenvalues—corresponding to the quantiles of $\mu_{\SMP_1 \boxtimes \SMP_2}$—exhibit strong rigidity, meaning that the gaps between consecutive eigenvalues are uniformly small; see Section \ref{sec_tuningparametersection} of the supplement for more details.

By combining the design of Algorithm \ref{al0} with the insights from Corollary \ref{cor_finalconvergence} and Theorem \ref{thm_noisededuction}, our algorithm is tailored to accommodate varying SNR levels across datasets; see also Remark \ref{rem_noise} in the supplement.
}

\section{Numerical simulation studies}\label{sec_simulation}

We carry out simulation studies to evaluate the empirical performance of our proposed Algorithm \ref{al0}. We consider two types of tasks. {The first one is the simultaneous clustering of two datasets, where there exists certain latent correspondence between the clusters in the two datasets \citep{petegrosso2020machine}. The second task is to learn the low-dimensional structure of a high-dimensional noisy dataset based on an external, less noisy dataset containing some shared (but not necessarily identical) low-dimensional structures. We show that, compared with the existing methods, our method achieves higher clustering accuracy in the former task, and a higher manifold reconstruction accuracy in the latter task. Additional numerical results can be found in the supplement.} 

\subsection{Simultaneous clustering}\label{sec_biclustering} 
In the first study, we consider the simultaneous clustering problem, that is, obtaining cluster memberships simultaneously for each datasets. Recall that for $\xb \in \mathbb{R}^p$ is from a Gaussian mixture model (GMM) if its density function follows that 
\begin{equation}\label{eq_GMMdensity}
	p(\xb)=\sum_{j=1}^K \pi_j \Phi(\xb; \bm{\mu}_j, \Sigma_j),
\end{equation}
where $K$ is the number of clusters, $\pi_j$'s are the mixing coefficients and  $\Phi(\xb; \bm{\mu}_j, \Sigma_j)$ is a multivariate Gaussian density with parameters $(\bm{\mu}_j, \Sigma_j).$ Under our model (\ref{model}), we consider the following two simulation settings. For concreteness, we set  $n_1=n_2=600$ and $p=800$.

\noindent {\bf Setting 1.} The data $\xb_i, 1 \leq i \leq n_1,$ are generated from (\ref{model}), where the signals $\xb_i^0$ are drawn from (\ref{eq_GMMdensity}) with $K=6, \ \pi_j \equiv 1/6, \ \bm{\mu}_j=15 \mathbf{e}_j, \ \Sigma_j\equiv 9{\bf I}_p, 1 \leq j \leq 6,$ where $\{\mathbf{e}_j\}$ are the standard Euclidean basis in $\mathbb{R}^p$, and the noises $\bxi_i$ are generated from multivariate Gaussian satisfying (\ref{eq_noisevariance}) with $\sigma_1^2=0.25$. 
The data $\yb_i, 1 \leq i\leq n_2,$ are generated from (\ref{model}), where the noises $\bzeta_i$ are multivariate Gaussian vectors satisfying (\ref{eq_noisevariance}) with $\sigma_2^2=1$, and the signals $\yb_i^0$ satisfy
$\yb_i^0=\zb^0_i+\bw^0_i$, with $\zb^0_i$ generated
from (\ref{eq_GMMdensity}) with $K=6, \ \pi_j \equiv 1/6, \ \bm{\mu}_j=15 \mathbf{e}_j, \ \Sigma_j\equiv 9{\bf I}_p, 1 \leq j \leq 6$, 
and $\bw^0_i$ having all their components being zero except for its 6-th to 25-th components which are independently drawn from a uniform distribution between $[-3\tau, \tau]$. In this way, both datasets contain a low-dimensional Gaussian mixture cluster structure with the common cluster centers, whereas the second dataset $\{\yb_i\}_{1\le i\le n_2}$ is much noisier and contains an additional data-specific signal structure in the subspace nearly orthogonal to the cluster centers $\{\bmu_j\}$. 

\noindent{\bf Setting 2.} The data $\xb_i, 1 \leq i \leq n_1,$ are generated according to (\ref{model}) and noise setup (\ref{eq_noisevariance}), where $\bxi_i$'s are multivariate Gaussian vectors with $\sigma_1^2=0.25$, and the clean signals $\xb_i^0$ are generated from (\ref{eq_GMMdensity}) with $K=4, \ \pi_j \equiv 1/4, \ \bm{\mu}_j=15 \mathbf{e}_{j+2}, \ \Sigma_j= 9{\bf I}_p, 1 \leq j \leq 4.$ The data $\yb_i, 1 \leq i \leq n_2,$ are generated in the same way as in the Setting 1. Compared with Setting 1, in additional to the  structural differences introduced by $\bw^0_i$, here the low-dimensional Gaussian mixture cluster structure is only partially shared across the two datasets.  

We examine the performance of our proposed algorithm on both settings, and compare the performance with six alternative methods: (1). pca: apply PCA to each dataset, and then perform k-means method to the $\mathsf r$-dimensional embedding based on the first $\mathsf r$ principle components (i.e., the first $\mathsf r$ eigenvectors of the sample covariance matrices of each dataset); (2). kpca: apply kernel-based PCA to each dataset, and then perform k-means to the $\mathsf r$-dimensional embedding based on the eigenvectors of some kernel matrices \citep{ding2022learning}; {(3). j-pca: first concatenate the two datasets, and then apply PCA to $\{\xb_1,...,\xb_{n_1}, \yb_1,...,\yb_{n_2}\}$, and finally perform k-means to the first $\mathsf r$ principle components;} (4). j-kpca: first concatenate the two datasets, and then apply kernel-based PCA to $\{\xb_1,...,\xb_{n_1}, \yb_1,...,\yb_{n_2}\}$, and finally perform k-means to the $\mathsf r$-dimensional embedding \citep{reverter2014kernel}; (5). lbdm: the LBDM bi-clustering algorithm \citep{pham2018large}; (6). rl:  k-means applied to the $\mathsf r$-dimensional embeddings produced by the Roseland algorithm \citep{10.1093/imaiai/iaac013} ; (7). prop:   k-means applied to the $\mathsf r$-dimensional embeddings produced by our proposed Algorithm \ref{al0}. For all the kernel-based methods, we set $\Gamma_1=\Gamma_2=\{2,3,...,\mathsf r+1\}$ as the first eigenvector is noninformative in general (nearly constant vector). 
In what follows, we set $\mathsf r=6$. For each of the settings, we vary the structural discrepancy parameter $\tau$ and evaluate the clustering performance of each method using the Rand index \citep{rand1971objective} between the estimated and true cluster memberships. Specifically, suppose for each dataset $\ell\in\{1,2\}$, $\hat M^{(\ell)}$ and $M^{(\ell)}$ are two partitions of $\{1,2,...,n_{\ell}\}$ samples according to the estimated and true cluster memberships, the overall Rand index (RI) is defined as 
\beq \label{rand}
\operatorname{RI} = \frac{\operatorname{RI}^{(1)}+\operatorname{RI}^{(2)}}{2},\qquad \ \text{where}\ \operatorname{RI}^{(\ell)}=\frac{a^{(\ell)}+b^{(\ell)}}{{n_{\ell} \choose 2}},\ \ell=1,2,
\eeq
where $a^{(\ell)}$ is the number of pairs of elements in $\{1,2,...,n_{\ell}\}$ that are in the same subset in $\hat M^{(\ell)}$ and $M^{(\ell)}$, and $b^{(\ell)}$ is the number of pairs of elements in $\{1,2,...,n_{\ell}\}$ that are in the different subsets in $\hat M^{(\ell)}$ and  $M^{(\ell)}$.
For each setting, we repeat the simulation 100 times to calculate the averaged Rand index. { See Section \ref{supp.sec.nu} for additional numerical results concerning the alignability screeening procedure, imbalanced sample sizes and unequal index sets $\Gamma_1$ and $\Gamma_2$.}

The results can be found in Figure \ref{simu.fig1} of the supplement. We find that in both settings, our proposed method has in general the highest Rand index values across all the cases, followed by ``lbdm" under most of the $\tau$ values. In particular, as the structural discrepancy between the two datasets increases (i.e., when $\tau$ increases), there is a performance decay for most methods. 
Our results indicate that the proposed method does better job in leveraging the (even partially) shared cluster patterns across datasets to improve clustering of each individual dataset, and achieves\ superior performance over the existing joint clustering methods.

\subsection{Nonlinear manifold learning}\label{sec_ml}
In the second study,  we consider enhancing the embeddings of low-dimensional manifold structures contained in a high-dimensional noisy dataset $\{\yb_j\}_{1\le j\le n_2}$, with the help of an external dataset $\{\xb_i\}_{1\le i\le n_1}$  which contains stronger and cleaner signals. For the data point $\yb_j$ from the nosier dataset, in terms of the model setup (\ref{model}) and noise setup (\ref{eq_noisevariance}), we let $\bm{\zeta}_j$ be generated from multivariate Gaussian distribution with $\sigma_2^2=1.$ For the clean signal part $\yb_j^0,$ its  first three coordinates are uniformly drawn from a torus in $\R^3$, i.e., for $u,v\in[0,2\pi)$, $\yb_{j1}^0=\theta (2+0.8\cos u) \cos v, \yb_{j2}^0=\theta(2+0.8\cos u)\sin v$ and $\yb_{j3}^0= 0.8\theta\sin u,$ where $\theta=0.2n^{1/2}$ characterizing the overall signal strength of the torus. For the subsequent coordinates of $\yb_i^0,$ the first 20 components (i.e., the 4-th to 23-th components) are independently drawn from a uniform distribution on $[-8, 8]$, whereas the rest components are zero. The dataset $\{\yb^0_j\}_{1\le j\le n_2}$ contains two mutually orthogonal low-dimensional signal structures. 
For the data point $\xb_i$ from the cleaner dataset, in terms of the model setup (\ref{model}) and noise setup (\ref{eq_noisevariance}), we generate $\bm{\xi}_i$ from multivariate Gaussian distribution with $\sigma_1^2=0.16.$ For the clean signal part $\xb_j^0,$ its  first three coordinates are uniformly drawn from the same torus in $\R^3$, 
and the rest of the coordinates are set as zeros. As such, $\{\xb_i\}_{1\le i\le n_1}$ is less noisy and only contains partially shared signal structure with $\{\yb_j\}_{1\le i\le n_2}$, i.e., the torus structure.

Our interest is to recover the torus structures of $\{\yb^0_j\}$ with the help of $\{\xb_i\}.$ Specifically, to evaluate the structural concordance between the obtained 3-dimensional embeddings of $\{\yb_i\}_{1\le i\le n_2}$ and the original noiseless signals $\{\yb_i^0\}_{1\le i\le n_2}$, we use Jaccard index to measure the preservation of neighborhood structures. 
More concretely, we consider the Jaccard index  between the set of 50 nearest neighbors of the embeddings and that of the clean signals $\{\yb_j^0\}$ and then calculate the average index value across all data points as the concordance score. In other words, for each $j\in \{1,2,...,n_2\}$, denote $S_j$ as the index set for the 50-nearest neighbors of $\yb^0_j$ and $W_j$ as the index set for the 50-nearest neighbors of its low-dimensional embedding, we calculate
\beq\label{jaccard}
\text{concordance}:= \frac{1}{n_2}\sum_{j=1}^{n_2}\text{[Jaccard Index]}_j,\qquad \text{[Jaccard Index]}_j:=\frac{|S_j\cap W_j|}{|S_j\cup W_j|}.
\eeq
We examine the performance of our proposed algorithm under the setup that $p=800$ and 
$n_1=n_2=n,$ where $n$ range from 400 to 1000. For each $n$, we repeat the Monte Carlo simulations 1,000 times and report the averaged concordance. We compare our methods with six other methods as discussed in Section \ref{sec_biclustering} but only assess their respective low-dimensional embeddings of the dataset $\{\yb_j\}_{1\le i\le n_2}$ in term of (\ref{jaccard}). Among them, two embedding methods (``pca" and ``kpca") only use $\{\yb_j\}_{1\le i\le n_2}$ whereas other embedding methods use both datasets. 

Our results are summarized in Figure \ref{simu.fig2} of the supplement. Compared with alternative methods, the proposed algorithm has superior performance across all the settings in retrieving the geometric (torus) structures from the noisy dataset $\{\yb_i\}_{1\le i\le n_2}$.  As the sample size increases, the proposed method along with ``j-kpca" and ``kpca" has improved performance, whereas ``rl" and ``lbdm", although achieving moderate performance under smaller sample sizes, have slightly decreased performance, possibly due to insufficient account of the high dimensionality and diverging signal strength. Comparing with non-integrative methods such as ``kpca" and ``pca," our results suggest the advantages of integrative methods, such as ``prop", ``j-kpca", ``lbdm" and ``rl", that borrow the shared information from the external dataset $\{\xb_i\}$. {Comparing the running time with the two related kernel algorithms ``rl" and ``lbdm" over datasets with various sample sizes, our proposed method demonstrates similar  scalability as ``lbdm", which is better than ``rl" (Figure \ref{simu.fig2} of the supplement). In particular, the proposed method only requires less than 1.5 minutes to process datasets with $n>10^4$ samples. }

\section{Applications in integrative single-cell omics analysis}\label{sec_realdataanalysis}

We apply the proposed method and evaluate its performance for joint analysis of single-cell omics data. 
By leveraging and integrating datasets with some shared biological information, we hope to improve the understanding of individual datasets, such as enhancing the underlying biological signals and capturing higher-order biological variations \citep{butler2018integrating,luecken2022benchmarking}. Here we test our method by considering the important task of identifying distinct cell types based on multiple single-cell data of the same type. { The performance of different methods is compared and validated using benchmark single-cell datasets for which cell type labels are available from the original scientific publications. These annotations were treated as ground truth cluster labels, against which we assessed the embedding quality of each method.} In what follows,
we analyze two single-cell RNA-seq datasets of human peripheral blood mononuclear cells (PBMCs) generated under different experimental conditions. Moreover, in Section \ref{supp.sec.nu} of the supplement, we analyze two single-cell ATAC-seq datasets of mouse brain cells generated from different studies.


The following example concerns joint analysis of two single-cell RNA-seq datasets for human PBMCs \citep{kang2018multiplexed}. In this experiment, PBMCs were split into a stimulated group ($n_1=7451$ cells) and a control group ($n_2=6548$ cells), where the stimulated group was treated with interferon beta. The distinct experimental conditions and different sequencing batches necessarily introduced structural discrepancy, or batch effects, between the two datasets, making it difficult and problematic to directly combine the two datasets. { Here we are interested in identifying clusters of different cell types for each dataset, in the same spirit as our simulation studies in Section 4.1. We evaluate the embedding quality using the Rand index, comparing the estimated clusters to the true labels, with higher values indicating better performance.} We preprocess and normalize each dataset by following the standard pipeline implemented in the Seurat R package \citep{stuart2019comprehensive}, and select the $p=1000$ most variable genes for subsequent analysis. We apply our proposed method, the Seurat  integration method \citep{butler2018integrating,stuart2019comprehensive} (``seurat") -- arguably the most popular single-cell integration method -- as well as  six other methods evaluated in our simulation, to obtain clustering of both datasets. 

For each method, for various choices of $\mathsf r$ (from 5 to 20), we first obtain $\mathsf r$-dimensional embeddings of both datasets, and then apply the hierarchical clustering algorithm 
to assign the cells into $\mathsf r+1$ clusters. As in Section \ref{sec_biclustering}, for each $\mathsf r$, the clustering results are evaluated using the Rand index (\ref{rand}), averaged across the two datasets, by comparing with the cell type annotations obtained from the original scientific publication \citep{kang2018multiplexed,stuart2019comprehensive}. Figure \ref{realdata.fig} of our supplement (Left) contains the evaluation results, where for each method, we present a boxplot of the Rand indices obtained based on different $\mathsf r$ values. 

The proposed method achieves overall the best performance in identifying the distinct cell types. Its improvement over the methods based on separate analysis (``pca" and ``kpca") demonstrates the advantage of integrative analysis of multiple datasets. { Moreover, comparing with other integrative approaches such as ``j-kpca" and ``seurat", the proposed method not only achieves better clustering accuracy but shows least variability over different values of $\mathsf r$, demonstrating its general robustness against the choice of embedding dimensions and cluster numbers. }


 \clearpage
\begin{center}
{\bf \huge Supplementary materials}
\end{center}

 In the supplement, we prove all the main theorems and  the technical lemmas. Some additional discussions about theoretical background and numerical implementation are also included.

{The supplement is organized as follows. Section \ref{supp.sec.nu} presents additional numerical results. Section \ref{appendix_additionalremark} provides several remarks on our methodology and theoretical findings. In Section \ref{appendix_preliminaryresults}, we state and prove technical lemmas concerning convergence in RKHSs, random matrix theory, and high-dimensional concentration inequalities. Section \ref{sec_techinicalproofappendix} is devoted to the proofs of the main results. Finally, Section \ref{sec_additionaladalall} offers further discussions, technical background, and auxiliary lemmas.}

Throughout the paper, the following conventions and notations will be consistently used. We interchangeably use the notation $\mathsf{X}=\OO_{\prec}(\mathsf{Y})$, $\mathsf{X} \prec \mathsf{Y}$ or $\mathsf{Y}\succ \mathsf{X}$  if $\mathsf{X}$ is stochastically dominated by $\mathsf{Y}$, uniformly in $u\in\mathsf{U}^{(n)}$, when there is no risk of confusion. We say that an $n$-dependent event $\Psi \equiv \Psi(n)$ holds {with high probability} if for any large $D>1$, there exists $n_0=n_0(D)>0$ so that $\mathbb{P}(\Psi) \geq 1-n^{-D},$
for all $n \geq n_0.$  For two sequences of deterministic positive values $\{a_n\}$ and $\{b_n\},$ we write $a_n=\OO(b_n)$ if $a_n \leq C b_n$ for some positive constant $C>0.$ In addition, if both $a_n=\OO(b_n)$ and $b_n=\OO(a_n),$ we write $a_n \asymp b_n.$ Moreover, we write $a_n=\oo(b_n)$ if $a_n \leq c_n b_n$ for some positive sequence $c_n \to 0.$
For any probability measure $\sfP$ over $\Omega$, we denote $\mathcal{L}_2(\Omega, \sfP)$ as the  collection of $L_2$-integrable functions with respect to $\sfP$, that is, for any $f\in \mathcal{L}_2(\Omega, \sfP)$, we have $\|f\|_{\sfP}=\sqrt{\int_{\Omega} |f(y)|^2\sfP(dy)}<\infty$. 
For a vector $\bold{a} = (a_1,...,a_n)^\top \in \mathbb{R}^{n}$, we define its $\ell_p$ norm as $\| \bold{a} \|_p = \big(\sum_{i=1}^n |a_i|^p\big)^{1/p}$.  
We denote $\text{diag}(a_1,...,a_n)\in\R^{n\times n}$ as the diagonal matrix whose $i$-th diagonal entry is $a_i$.
For a matrix $ \bold{A}=(a_{ij})\in \R^{n\times n}$,  we define its Frobenius norm as $\| \bold{A}\|_F = \sqrt{ \sum_{i=1}^{n}\sum_{j=1}^{n} a^2_{ij}}$, 
and its operator norm as $\| \bold{A} \| =\sup_{\|\bold{x}\|_2\le 1}\|\bold{A}\bold{x}\|_2 $.
For any integer $n>0$, we denote the set $[n]=\{1,2,...,n\}$. 
For a random vector $\mathbf{g},$ we say it is sub-Gaussian if $
\mathbb{E} \exp(\mathbf{a}^\top \mathbf{g}) \leq \exp\left( \| \mathbf{a} \|_2^2/2 \right)$ 
for any deterministic vector $\mathbf{a}.$
Throughout, $C,C_1,C_2,...$ are universal constants independent of $n$, and can vary from line to line.
We denote ${\bf I}_n$ as an $n$-dimensional identity matrix.


\setcounter{section}{0}
\renewcommand{\thesection}{\Alph{section}}
\renewcommand{\thetable}{\Alph{table}}
\renewcommand{\thefigure}{\Alph{figure}}  

 \section{Additional numerical results} \label{supp.sec.nu}

The R codes that reproduce our analyses are available at our GitHub repository \url{https://github.com/rongstat/KSJE}.

\paragraph{Numerical results for the simulations conducted in  Section \ref{sec_simulation}.} The numerical results for the simulations conducted in Section \ref{sec_simulation} are presented below in Figures \ref{simu.fig1} and \ref{simu.fig2}.

\begin{figure}
	\centering
	\includegraphics[angle=0,width=14cm]{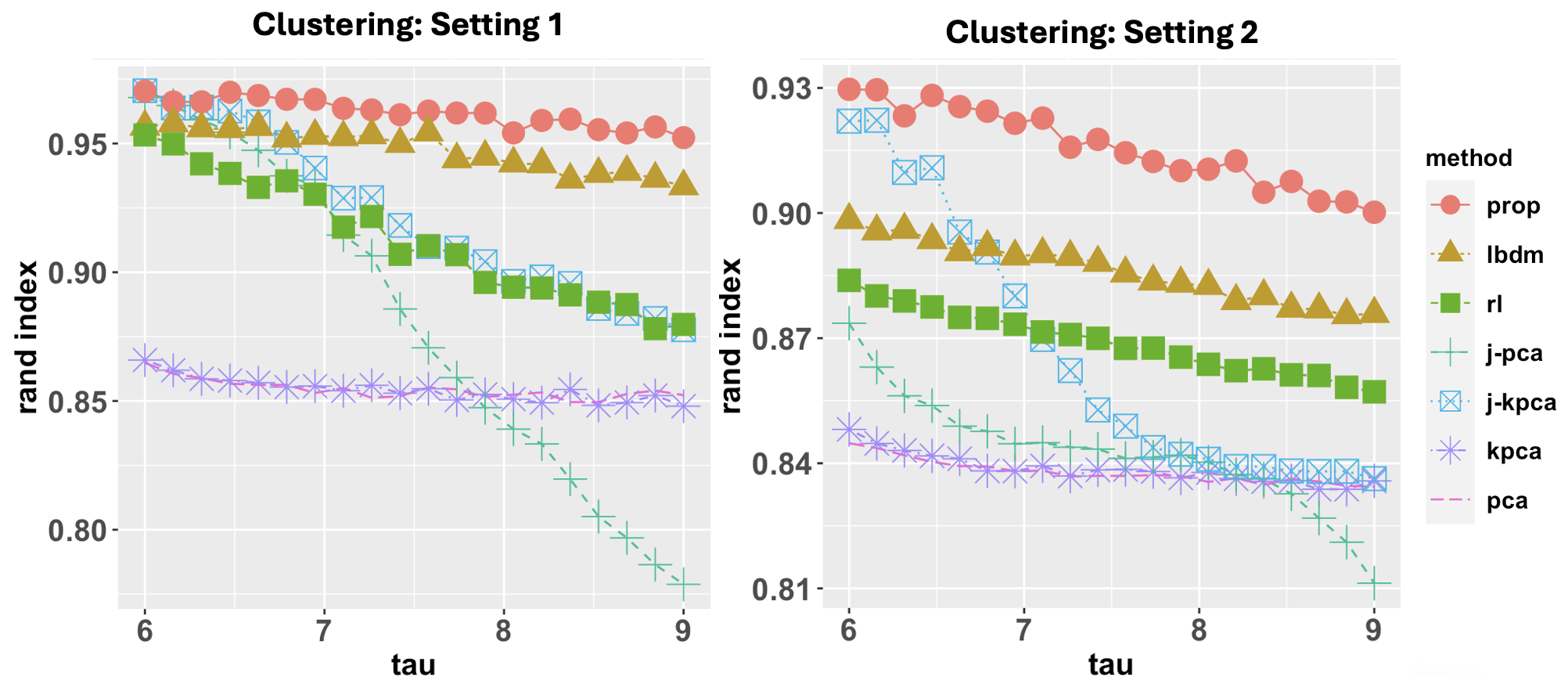}
	\caption{Comparison of simultaneous clustering performance of 7 different approaches using Rand index. The parameter $\tau$ indicates strength of added structural discrepancy between the two datasets. Left: simulation Setting 1 with identical cluster structures. Right: simulation Setting 2 with partially overlapping clusters. The proposed method does the best in leveraging the (even partially) shared cluster patterns across datasets to improve clustering of each individual dataset.} 
	\label{simu.fig1}
\end{figure}

\begin{figure}
	\centering
	\includegraphics[angle=0,width=14cm]{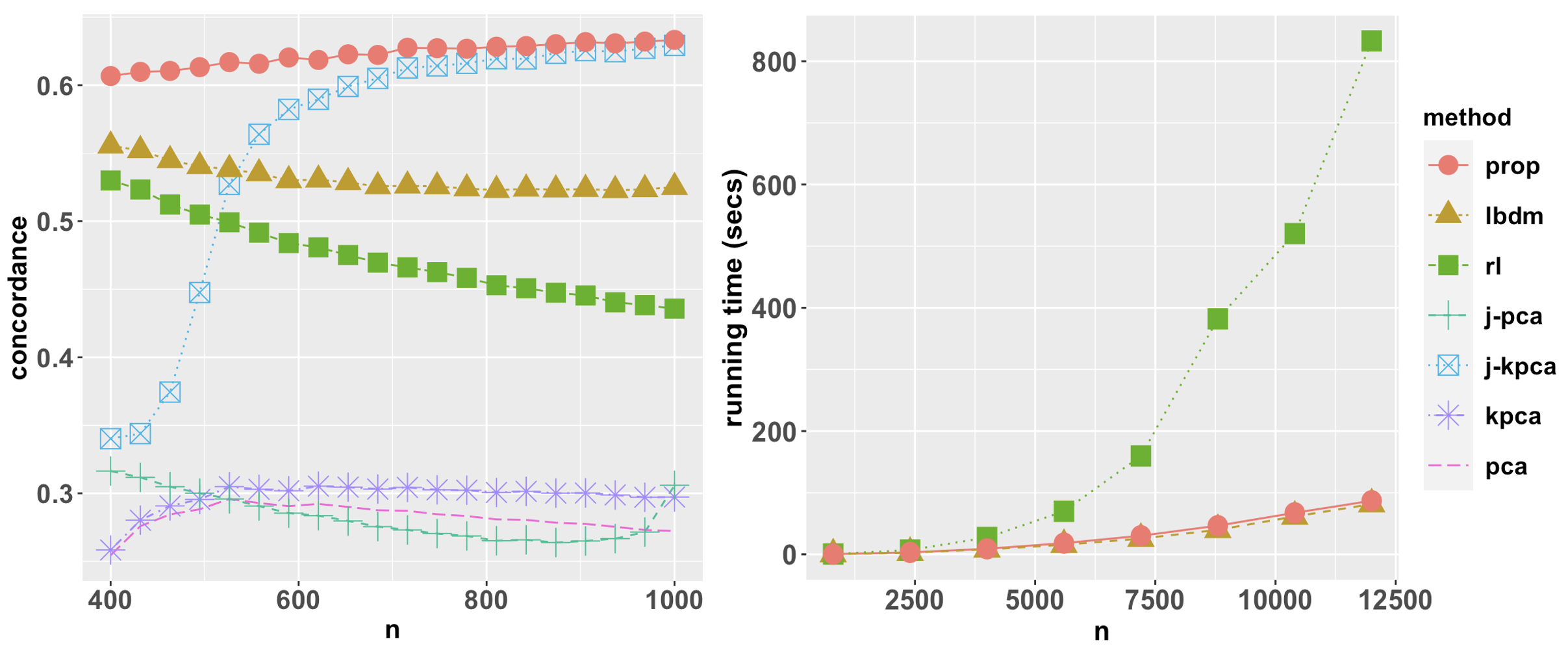}
	\caption{Comparison of nonlinear manifold learning performance of 7 different approaches. Left: concordance measures under various sample sizes. Right: running time (minutes) comparison of three related algorithms ``prop", ``rl" and ``lbdm," that achieved relatively better performance, showing competent scalability of ``prop." The proposed algorithm has overall the best performance in retrieving the torus structures from the noisy dataset $\{\yb_i\}$.  Our results suggest the advantages of integrative embedding methods (e.g., ``prop") over non-integrative embedding method (``pca" and ``kpca"), by leveraging the shared information from the external, cleaner dataset $\{\bx_i\}$.  
	} 
	\label{simu.fig2}
\end{figure}

\paragraph{Mouse brain single-cell ATAC-seq data.}
The second real data analysis example concerns integrative analysis of two single-cell ATAC-seq datasets for mouse brain cells ($n_1=3618$ and $n_2=3715$) generated from different studies \cite{luecken2022benchmarking}. ATAC-seq is a biotechnology that quantifies the genome-wide chromatin accessibility, which contains important information about epigenome dynamics and gene regulations. Each dataset contains a matrix of ATAC-seq gene activity scores, characterizing gene-specific chromatin accessibility for individual cells. Again, we preprocess and normalize the datasets following the standard pipeline in the Seurat package, and select the $p=1,000$ most variable genes for subsequent analysis. We evaluate the performance of the above eight methods in clustering the cells in both datasets under various embedding dimensions $\mathsf r$ (from 5 to 20), using the Rand index based on the cell type annotations provided in the original scientific publication  \cite{luecken2022benchmarking}. Our results in Figure \ref{realdata.fig} (Right) again indicate  the overall superior performance of the proposed method, demonstrating the benefits of integrative analysis of single-cell data, and its robustness with respect to different choices of embedding dimensions.

\begin{figure}
	\centering
	\includegraphics[angle=0,width=13cm]{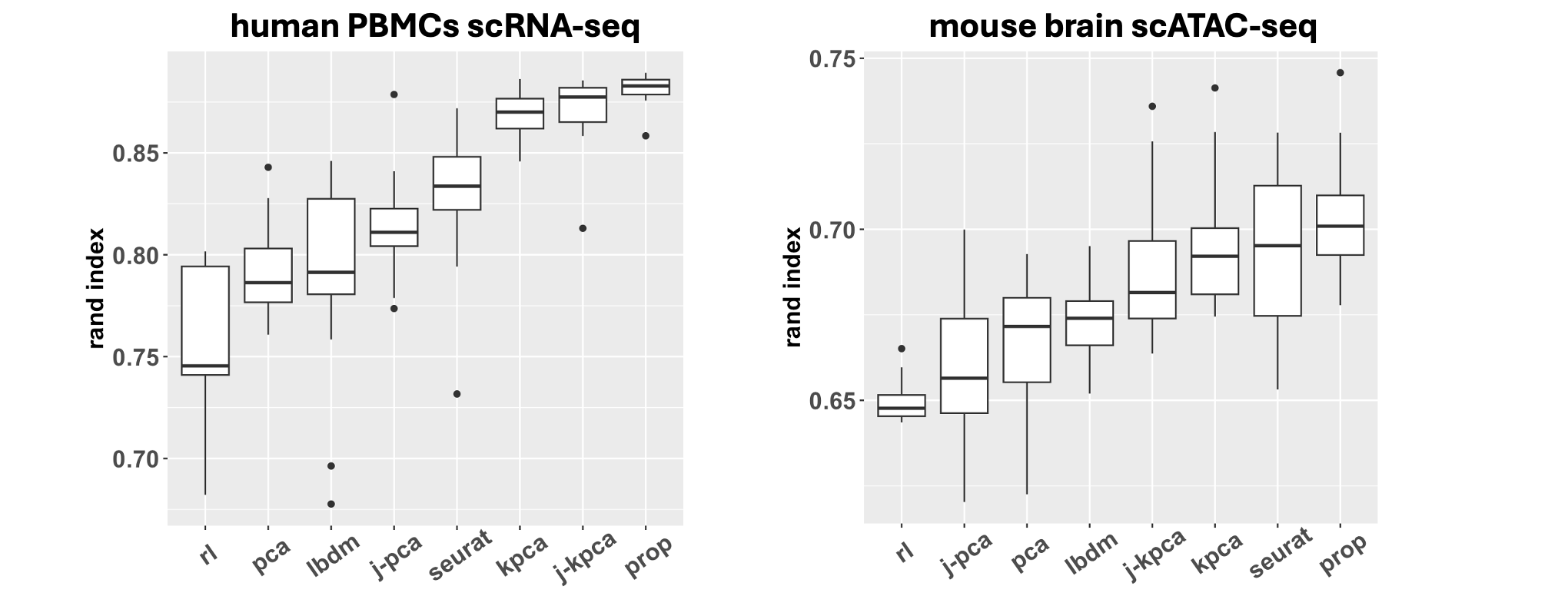}
	\caption{Comparison of eight methods for simultaneous biclustering of single-cell datasets. Each boxplot contains the  Rand index for clustering accuracy obtained under various embedding dimensions ($\mathsf r$ from 3 to 20). Left: single-cell RNA-seq data for human peripheral blood mononuclear cells \cite{kang2018multiplexed}. Right: single-cell ATAC-seq gene activity data for mouse brain cells \cite{luecken2022benchmarking}. The proposed method not only achieves better clustering accuracy compared with other methods, but also shows smaller variability and therefore robustness with respect to different choices of embedding dimensions. 
	} 
	\label{realdata.fig}
\end{figure}

	{
	\paragraph{Integrating alignable datasets and alignability screening.} For the numerical simulation setup considered in our main text, where in each case the two datasets contain some shared structures, we demonstrate the use of our proposed alignability screening procedure. {In particular, throughout this paper, we use $k=30$ for the nearest neighbor parameter in the alignability screening step}. For both the simultaneous clustering problem and the nonlinear manifold learning problem, in almost all the settings ($>99\%$), the simulated data pairs passed our alignability screening. Figure \ref{supp.simu.fig1} contains the boxplots of the median KNN purity scores within each simulation setting. 	
	Our results indicate that the proposed alignability screening does not have a negative effect on the integration of datasets containing (partially) shared structures.
	
\begin{figure}
	\centering
	\includegraphics[angle=0,width=14cm]{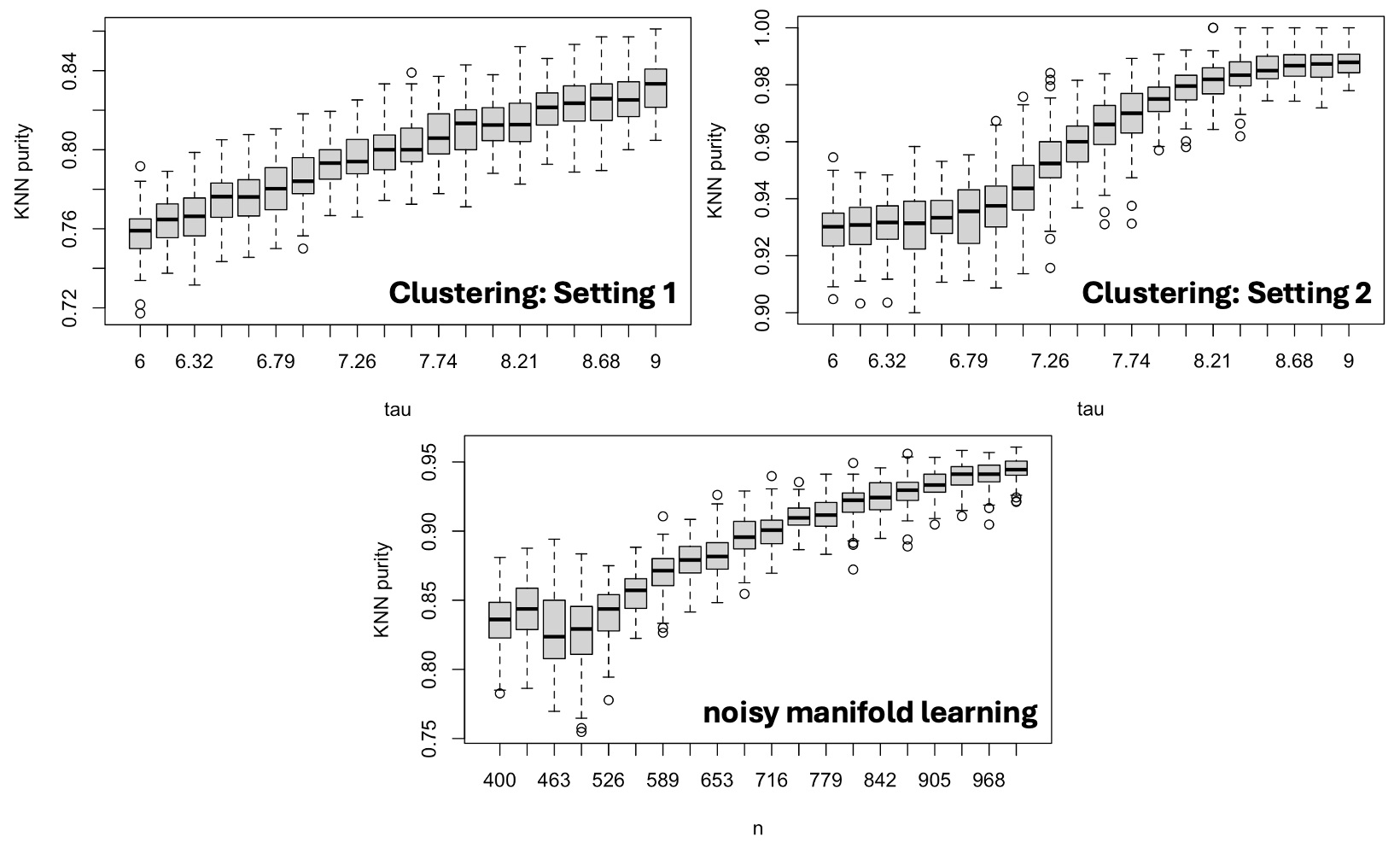}
	\caption{Boxplots of the median KNN purity scores within each simulation setting.} 
	\label{supp.simu.fig1}
\end{figure}

	\paragraph{Detection of non-alignable datasets.} When the pair of datasets do not contain overlapping signals, integrating these datasets using our proposed duo-landmark joint embedding procedure may introduce unwanted distortions to the datasets, creating artificial alignment patterns that may lead to misinterpretations and false conclusions. To demonstrate this potential risk and the power of the proposed alignability screening procedure, we generate a pair of negative control datasets, each containing a distinct manifold structure and there is no (measurable) overlap between the two manifolds. Specifically, the first dataset contains $n_1=3000$ samples, uniformly drawn from a ``Klein bottle" manifold in 4-dimensional space; see Figure \ref{supp.klein} for an illustration. The second dataset contains $n_2=2000$ samples in $\R^4$, whose first coordinates are uniformly drawn from $[-1,1]$, and the other coordinates are set as 0. If we directly apply our duo-landmark kernel embedding, we observe an artifical alignment between the two datasets (Figure \ref{supp.ce.prop}); in contrast,  the joint embeddings based on the j-kpca  do not display similar artificial alignment patterns (Figure \ref{supp.ce.kpca}). Now we apply our proposed alignability screening procedure (see Algorithm 1, Step 1) and compute the KNN purity score for each pair of generated data, finding that in all cases (across 100 rounds of simulations), our  alignability screening procedure can consistently detect the non-alignability between the two datasets. This suggests that in practice this procedure can effectively avoid generating misleading joint embedding of the two datasets. 
	
	\begin{figure}
		\centering
		\includegraphics[angle=0,width=14cm]{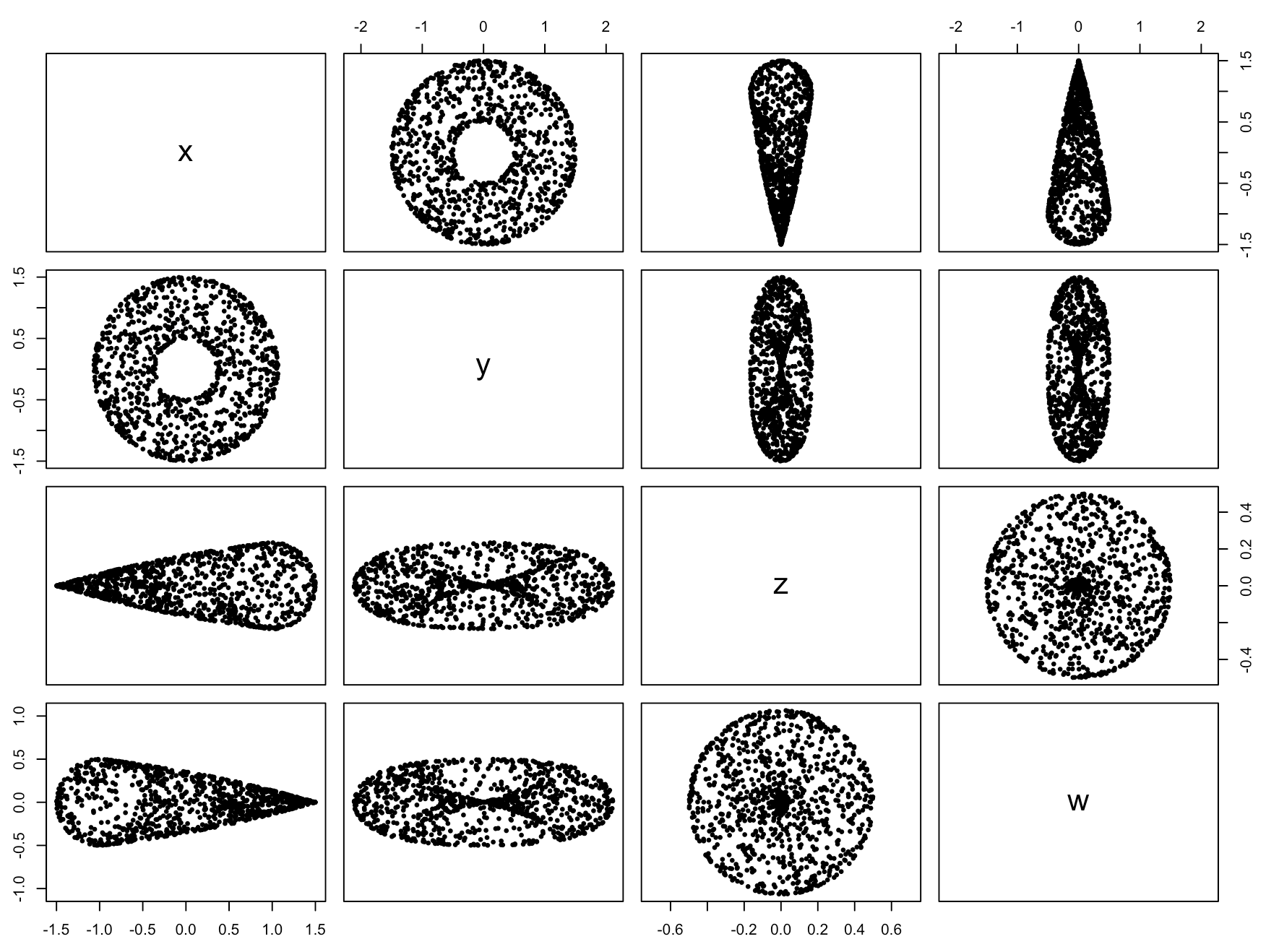}
		\caption{Pairwise coordinates scatter plots of random samples from the ``Klein bottle".} 
		\label{supp.klein}
	\end{figure}
	
		\begin{figure}
		\centering
		\includegraphics[angle=0,width=14cm]{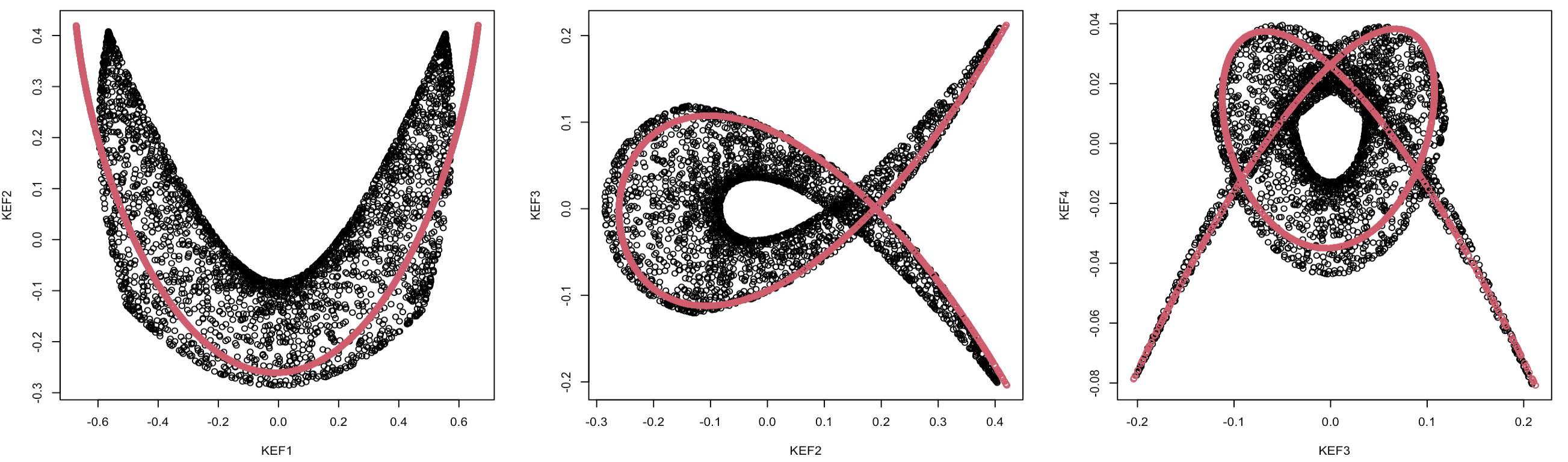}
		\caption{Pairwise coordinates scatter plots for the proposed joint embeddings. Black points: embeddings of the  3000 samples from the ``Klein bottle" manifold. Red points: embeddings of the  2000 samples from the straight line segment. KEF indicates the coordinates of the proposed kernel eigenfunction embeddings.} 
		\label{supp.ce.prop}
	\end{figure}
	
			\begin{figure}
		\centering
		\includegraphics[angle=0,width=14cm]{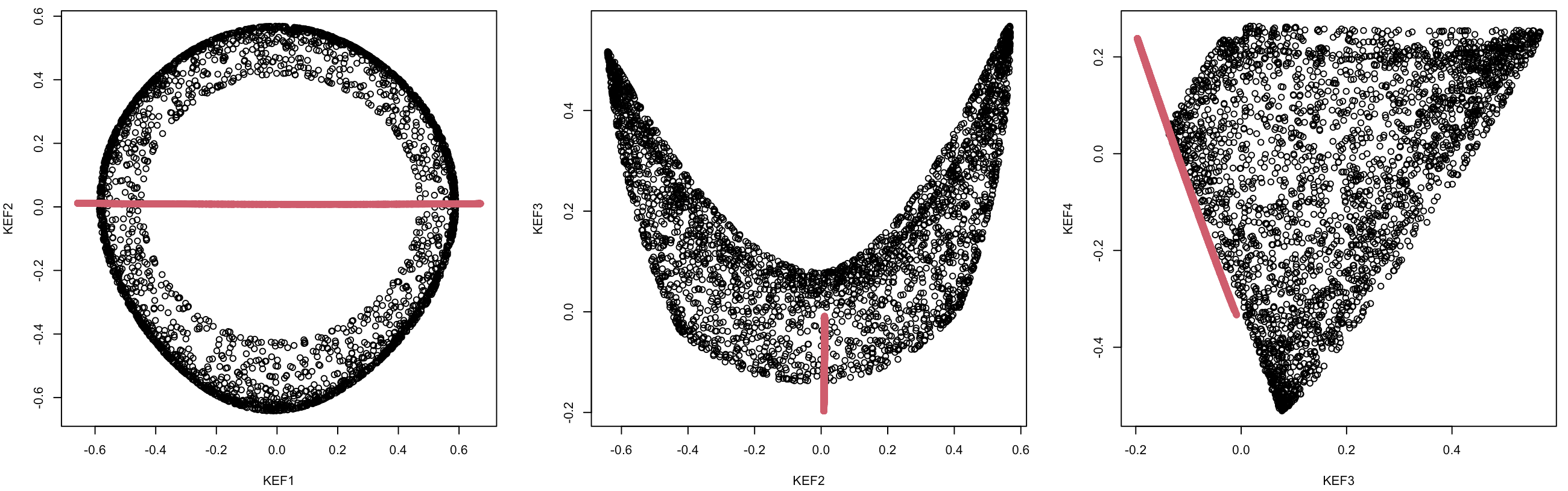}
		\caption{Pairwise coordinates scatter plots for the ``j-kpca"-based embeddings. Black points: embeddings of the  3000 samples from the ``Klein bottle" manifold. Red points: embeddings of the  2000 samples from the straight line segment. KEF indicates the coordinates of the proposed duo-landmark kernel  embeddings.} 
		\label{supp.ce.kpca}
	\end{figure}
	
	

	\paragraph{Integrating a noisy signal dataset with a pure noise dataset.} when one dataset contains relatively strong signals while the other consists purely of noise, similar to the previous case, we find our proposed duo-landmark kernel embedding may still introduce artifacts, forcing alignment between the low-dimensional embeddings of the two datasets. In contrast, j-KPCA does not appear to suffer from this issue. Nonetheless, in this scenario, our alignability screening procedure remains valuable in preventing the forced integration of these two datasets. To demonstrate this point, we generate a pair of datasets, one containing $n_1=3000$ samples uniformly drawn from the torus manifold $\mathbb{T}^2$ in 3-dimensional space, and the other dataset containing $n_2=2000$ samples in $\R^3$, whose coordinates are all drawn from $\mathcal{N}(0,1)$. Applying our  alignability screening procedure, we find that across all 100 rounds of simulations, our method successfully detects the non-alignability between these two datasets, again avoiding the misleading joint embeddings of the two datasets (see Figures \ref{supp.sn.prop} and \ref{supp.sn.jkpca}).

		\begin{figure}
		\centering
		\includegraphics[angle=0,width=14cm]{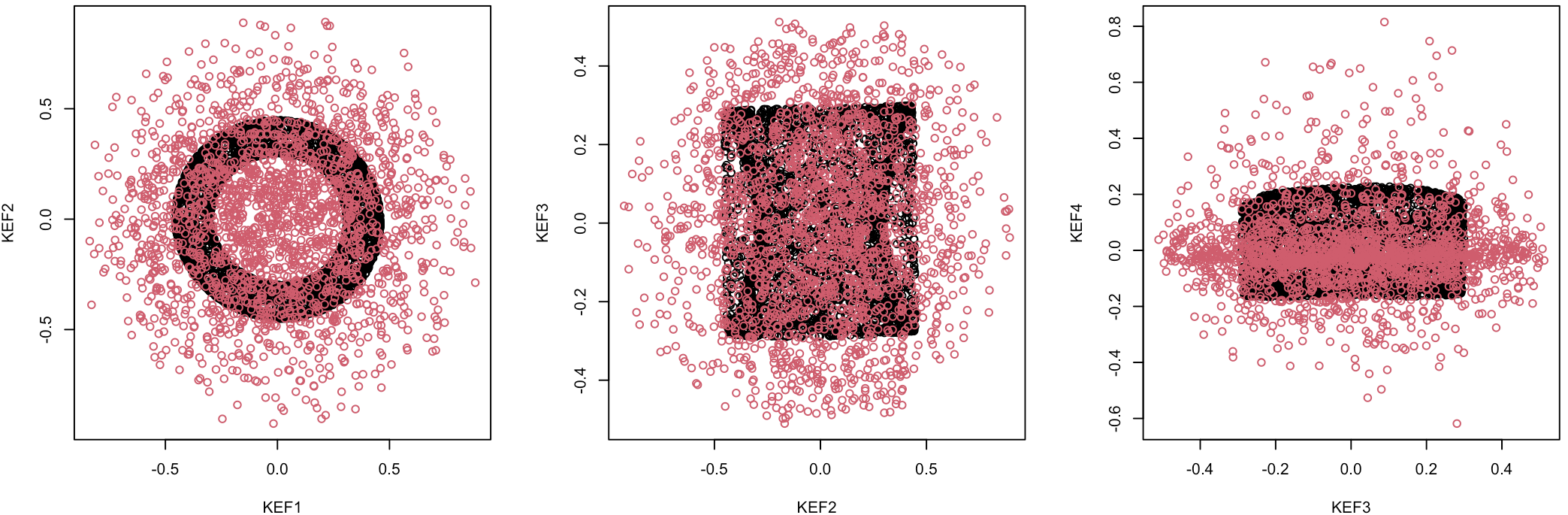}
		\caption{Pairwise coordinates scatter plots for the proposed joint embeddings. Black points: embeddings of the 3000 samples from the torus $\mathbb{T}^2$ manifold embedded in $\R^3$. Red points: embeddings of the 2000 samples from standard multivariate Gaussian distribution in $\R^3$. KEF indicates the coordinates of the proposed duo-landmark kernel embeddings.} 
		\label{supp.sn.prop}
	\end{figure}
	
		\begin{figure}
		\centering
		\includegraphics[angle=0,width=14cm]{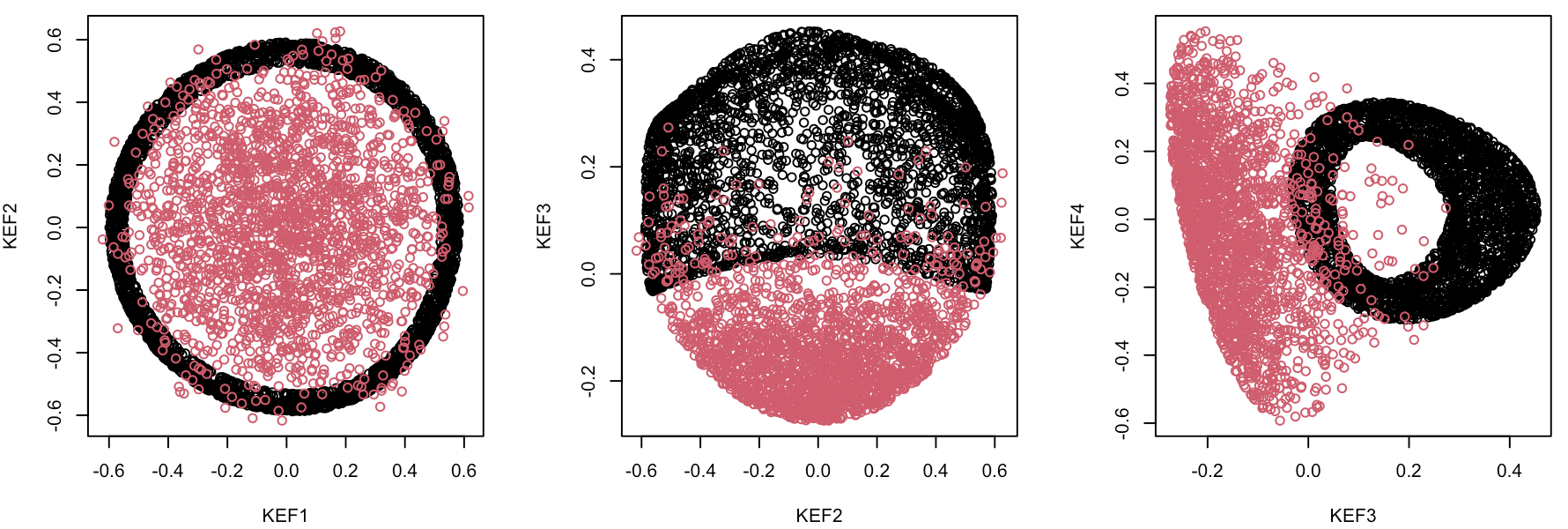}
		\caption{Pairwise coordinates scatter plots for the ``j-kpca" joint embeddings. Black points: embeddings of the 3000 samples from the torus $\mathbb{T}^2$ manifold embedded in $\R^3$. Red points: embeddings of the 2000 samples from standard multivariate Gaussian distribution in $\R^3$. KEF indicates the coordinates of the proposed duo-landmark kernel embeddings.} 
		\label{supp.sn.jkpca}
	\end{figure}
		
	\paragraph{Integrating a pair of pure noise datasets.} When both datasets consist of pure noise -- for example, when both datasets are drawn from a standard multivariate Gaussian distribution in $\R^4$ -- we find that our  alignability screening procedure would often indicate alignability between the datasets. Figure \ref{supp.nn.prop} contains a histogram of the median KNN purity score across 100 rounds of simulations, whose values vary around 0.5,  showing that the alignability screening algorithm cannot effectively distinguish such pathological cases from the favorable cases where two datasets contain shared signal structures.
{ Alternatively, in Section \ref{sec_tuningparametersection} of the Supplement, we propose an algorithm based on our theoretical results  under the low-SNR regime presented in Section \ref{sec_phasetransition}, that can help detect such pathological cases.}
	
	\begin{figure}
		\centering
		\includegraphics[angle=0,width=7cm]{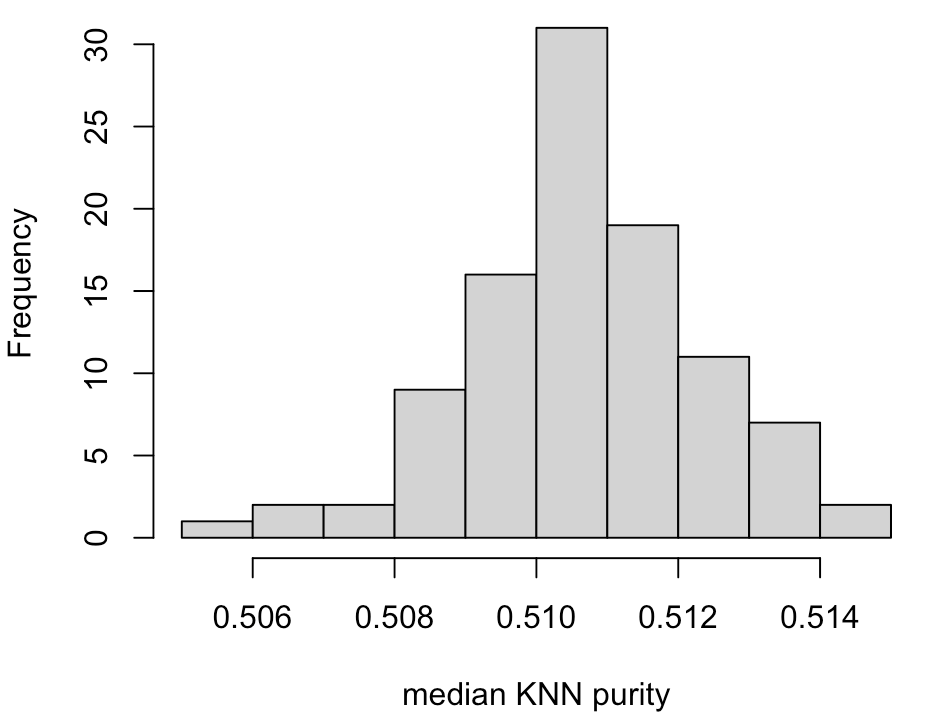}
		\caption{Histogram of the median KNN purity scores when both datasets consist of pure noise.} 
		\label{supp.nn.prop}
	\end{figure}

	\paragraph{Impact of sample-size imbalance.} We conduct simulations following the setup in Section 4.1 of the main text to assess the impact of unequal sample sizes. Specifically, under the same simulation framework, we consider scenarios where  $n_1=600$ and the ratio $n_2/n_1$ varying from 1.2 to 3.  We then evaluate the performance of different methods using the overall Rand index. We find that, on the one hand, even if the two datasets have very different sample sizes, the proposed method still performs well and demonstrates advantages over existing methods.  On the other hand, as suggested by our theoretical analysis (Theorem \ref{thm_cleanconvergence} and Corollary \ref{cor_finalconvergence}) as long as the smaller of the two dataset has sufficiently large sample size, the performance of our method is primarily determined by the sample size of the smaller dataset; in particular, increasing the sample size of the larger dataset does not further improve the performance of our method.
	
		\begin{figure}
		\centering
		\includegraphics[angle=0,width=7cm]{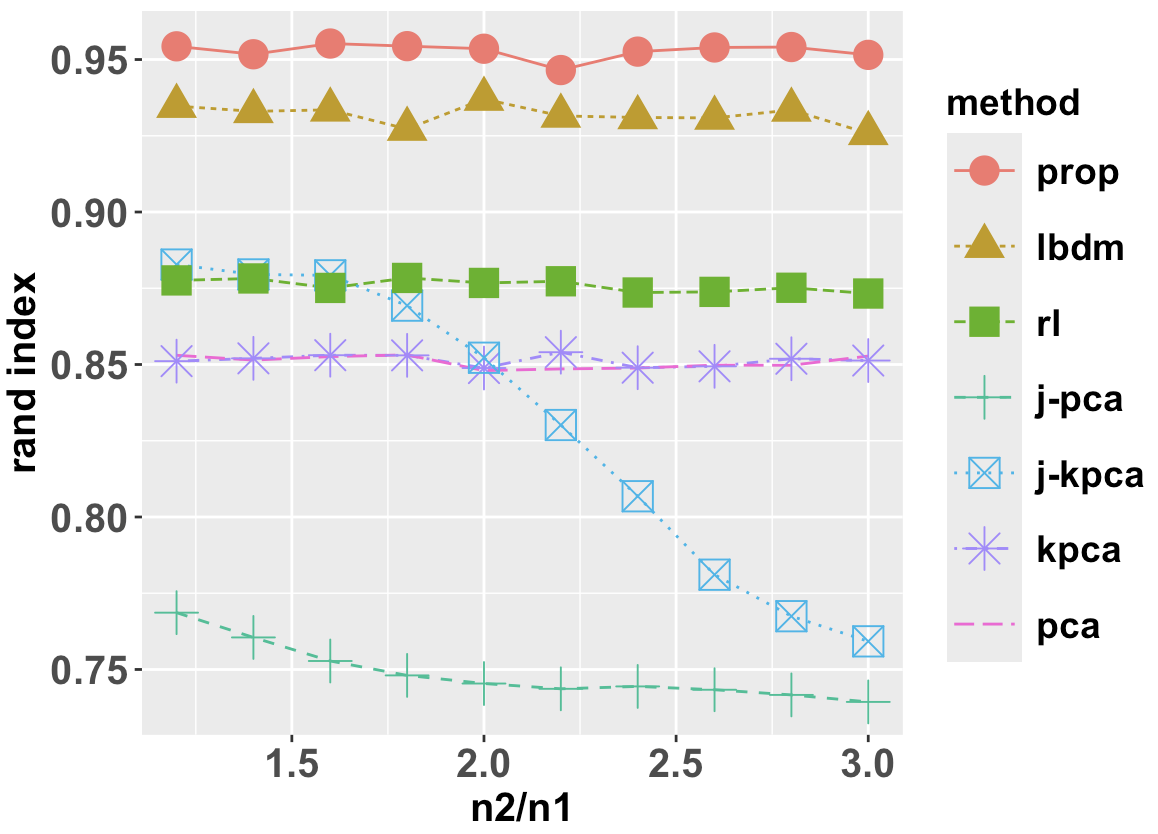}
		\caption{Comparison of simultaneous clustering performance of 7 different approaches using Rand index, averaged across 100 rounds of simulations for each setting. We fix $n_1=600$ and $p=800$, and let $n_2/n_1$ vary from 1.2 to 3.} 
		\label{supp.unbalance}
	\end{figure}

	\paragraph{Choice of $\Gamma_1$ and $\Gamma_2$.}  The choice of $\Gamma_1$ and $\Gamma_2$ will depend on the specific downstream applications of our spectral embeddings. If the goal is to obtain a joint low-dimensional representation of both datasets (as in our single-cell applications in Section 5), we acknowledge that it is typically required that $\Gamma_1=\Gamma_2$. However, if the goal is to conduct clustering analysis of the samples within each datasets, leveraging the possibly shared structures to improve the overall clustering performance on both datasets, the choice of $\Gamma_1$ and $\Gamma_2$ may depends on the number of clusters contained in each dataset. In setting 2 of the simulation study considered in Section 4.1 of the manuscript, we generated two datasets, one containing 6 clusters ($\mathcal{Y}$) and the other containing 4 clusters ($\mathcal{X}$), whose cluster centers align with some clusters in the first dataset. Although in the manuscript we used $\Gamma_1=\Gamma_2=\{2,3,4,5,6,7\}$ for simplicity and practicality, we show further in the supplement that if the number of clusters is known in advance, choosing $\Gamma_2=\{2,3,4,5,6,7\}$ and $\Gamma_1=\{2,3,4,5\}$ will in fact lead to better clustering of the samples, especially for the dataset $\mathcal{X}$ containing 4 clusters (see Figure \ref{supp.gamma}). This demonstrates the potential benefits of specifying different index sets for the two datasets.

	\begin{figure}
		\centering
		\includegraphics[angle=0,width=12cm]{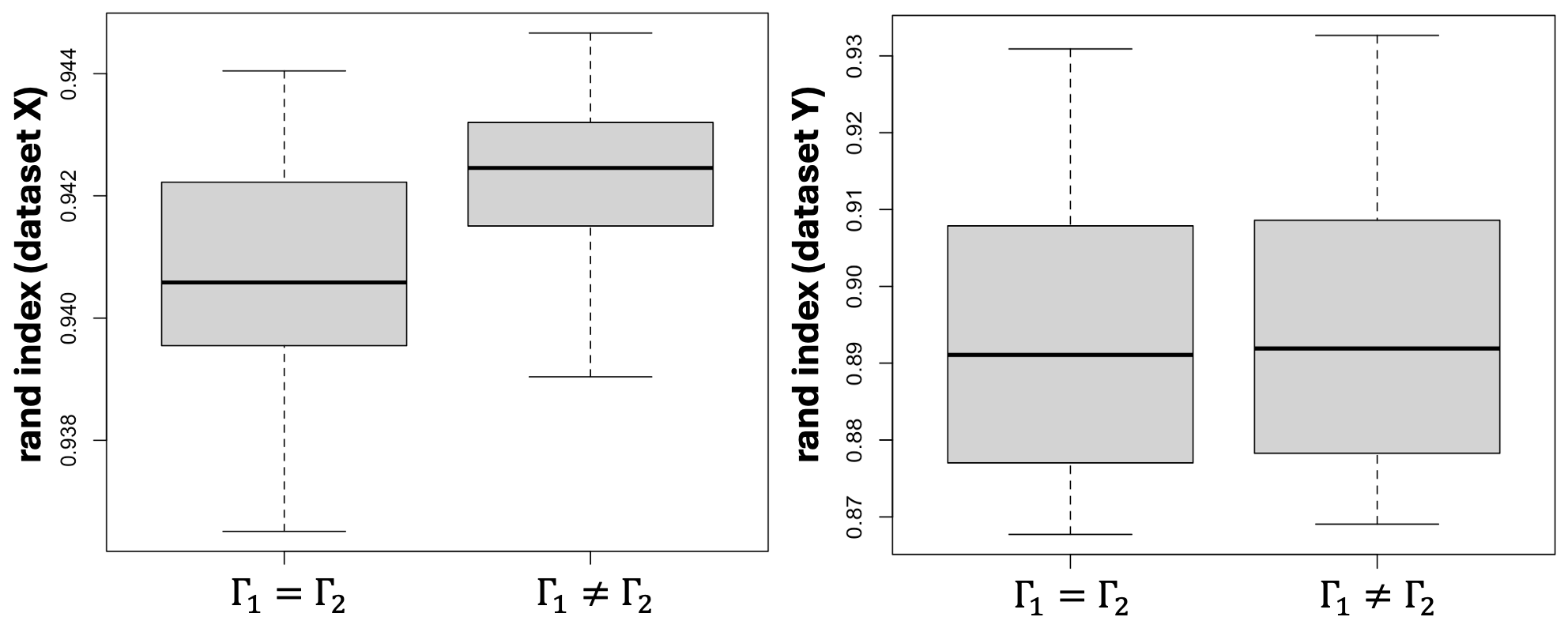}
		\caption{Comparison of clustering performance (Rand index) of the proposed method on the dataset $\mathcal{X}$ containing 4 clusters (left), and on the dataset $\mathcal{Y}$ containing 6 clusters (right), under the second simulation setup considered in Section 4.1. In each panel, the left boxplot corresponds to the embedding with $\Gamma_1=\Gamma_2= \{2,3,4,5,6,7\}$; the right boxplot corresponds to the embedding with $\Gamma_1=\{2,3,4,5\}$ and $\Gamma_2=\{2,3,4,5,6,7\}$.} 
		\label{supp.gamma}
	\end{figure}
	
	{
	\paragraph{Additional sensitivity analyses.} We conducted additional sensitivity analyses about the kernel bandwidth percentile parameter $\omega$. Under the noisy manifold learning task, we compared our methods with different choices of $\omega\in\{0.2,0.35,0.5,0.65,0.8\}$ for the bandwidth selection over different sample sizes. We found that overall the performance of the proposed method remains similar and stable across all the settings, with the methods under larger $\omega$ displaying slightly better performance compared with the others. See the left panel of Figure  \ref{simu.fig.3}.
	Similarly, we have conducted additional experiments to evaluate the sensitivity against the choice of the embedding index set $\Gamma_1=\Gamma_2=\{1,2,...,r\}$, particularly with different latent dimensions $r\in\{2,3,...,6\}$. Our analysis suggested that, as expected, when the underlying manifold has dimension $r^*=3$, the proposed method with underspecified $r$ achieved worse performance than that with correctly specified or over-specified $r$. In particular, the performance of the proposed method is less sensitive to the exact choice of $r$ whenever $r\ge r^*=3$. See the right panel of Figure  \ref{simu.fig.3}.}

	\begin{figure}
		\centering
		\includegraphics[angle=0,width=14cm]{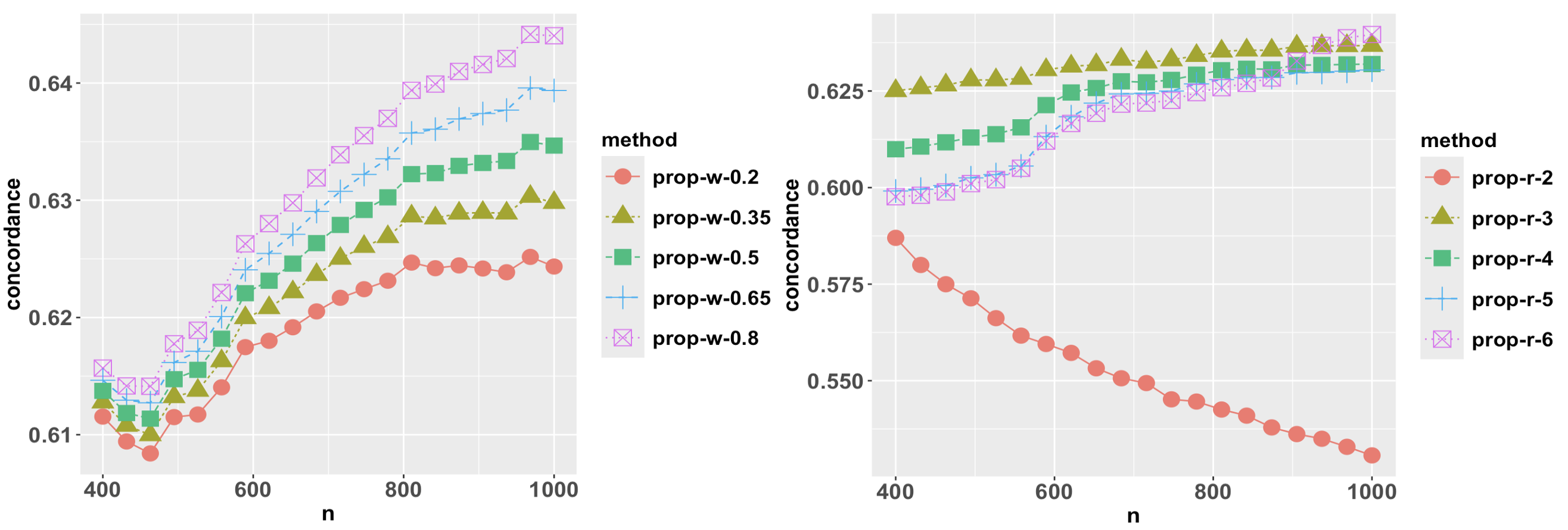}
		\caption{Comparison of the proposed method under different choices of $\omega\in \{0.2,0.35,0.5,0.65,0.8\}$ (Left) and of the embedding index set $\Gamma_1=\Gamma_2=\{1,2,...,r\}$, with dimension $r\in\{2,3,...,6\}$ (Right). The simulation is carried out under the noisy manifold learning setup described in Section 4.2.}  
		\label{simu.fig.3}
	\end{figure}

			\paragraph{Singular values of $\bK$ in real data.} Regarding the separation of the leading eigenvalues of the integral operators in real applications, by checking their empirical counterparts, that is, the leading singular values of the duo-landmark kernel matrix $\bK$ obtained from each single-cell dataset analyzed in Section 5 of the main text, we found that, for both human PBMCs scRNA-seq and and mouse brain scATAC-seq datasets, the leading (e.g., first 3 or 4) singular values appear to be well separated, with  the magnitude of the ``eigen gap" of the same order as the singular value themselves. This observation suggests that in single-cell applications, the leading eigenfunctions of the integral operators are often  distinguishable and can be individually estimated by our method.
		
		\begin{figure}
			\centering
			\includegraphics[angle=0,width=13cm]{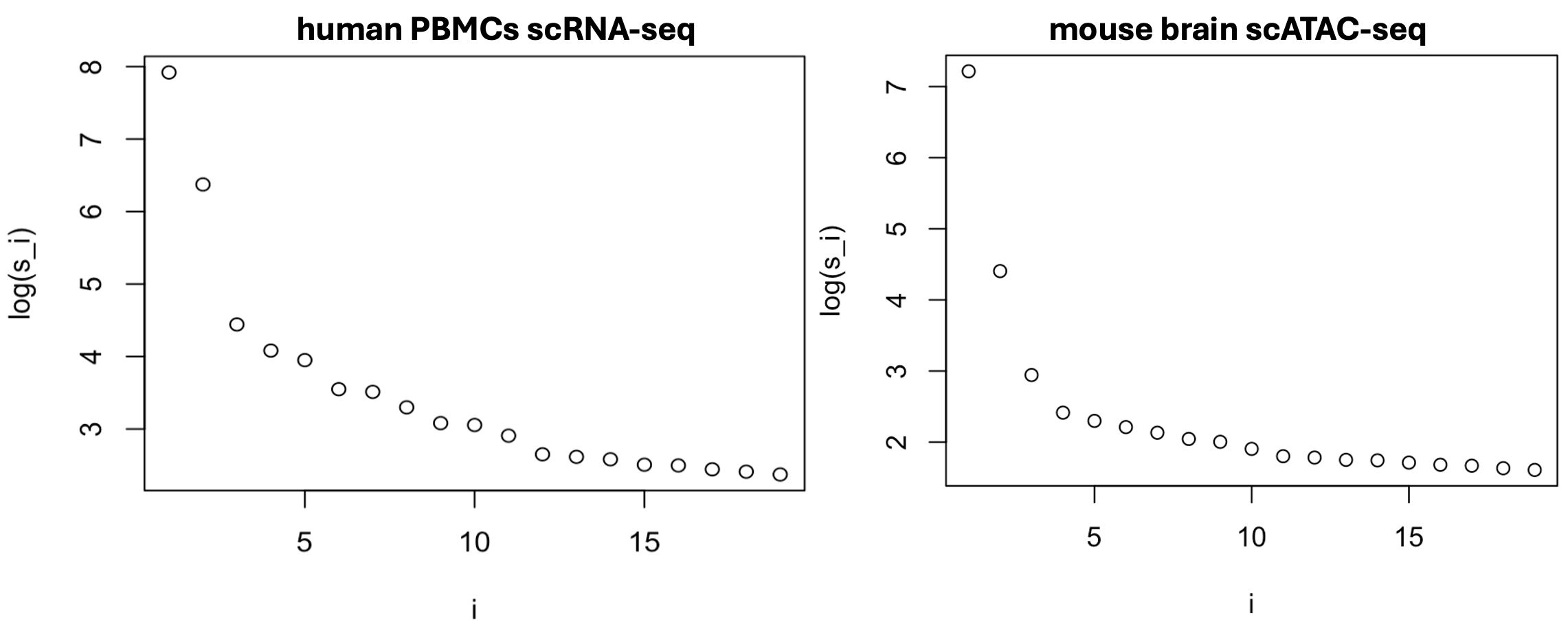}
			\caption{Summary of the leading 20 singular values (under log-scale) of $\bK$ in the two real single-cell datasets analyzed in the main text. } 
			\label{supp.eig.data}
		\end{figure}
		
}

{ \section{Some additional remarks}\label{appendix_additionalremark}
\begin{rem}[Choice of kernel functions] \label{rem_remarkkernel}
	Regarding the choices of kernel functions, in the current paper, for simplicity, we focus on the Gaussian kernel function as in (\ref{eq_clearnsignalkernel}). However, our results may be extended to a more general class of kernel functions, as discussed in Section 3.4 of \cite{ding2022learning}. This class includes Laplacian kernels, rational polynomial kernels, Matern kernels, and truncated kernels as special examples. The arguments are similar to Theorem 4 of \cite{ding2022learning}. Since this is not the focus of the current paper, we will pursue this direction in a subsequent work.  
\end{rem}

\begin{rem}[Advantages of duo-landmark embedding]
	{ Built upon these duo-landmark integral operators defined on the two manifolds, our algorithm is capable of generating spectral embeddings that encode the underlying geometric structures of both datasets. A key advantage is that our algorithm and theory can capture these structures without requiring the explicit form of the manifolds. All essential information about the manifolds, such as the Riemannian metric, curvature, and the first and second fundamental forms, is incorporated into the algorithms and the final embeddings through these proposed operators;} {see Section \ref{manifold.sec} for more details on this aspect.}
	\end{rem}
}

{
\begin{rem}[Key differences from the method in \cite{ding2022learning}]\label{rem_algo}
First, Algorithm 1 in \cite{ding2022learning} is designed to provide a low-dimensional embedding for a {single} dataset, with the goal of extracting useful underlying geometric information. In contrast, Algorithm 1 in the current paper aims to learn the {commonly shared} information between {two} datasets. Thus, the goals of the two algorithms are fundamentally different.  Although both algorithms involve the analysis of kernel affinity matrices, Algorithm 1 in \cite{ding2022learning} focuses on a {symmetric} kernel affinity matrix constructed from a complete graph based on a single dataset, aiming to produce a single set of embeddings. In contrast, our Algorithm 1 analyzes an {asymmetric} kernel affinity matrix derived from a graph built using two datasets, with the goal of producing two sets of joint embeddings. At a high level, Algorithm 1 in \cite{ding2022learning} utilizes a variant of the {adjacency matrix} of a single graph, whereas our current algorithm investigates a {biadjacency matrix} of a bipartite graph.

Second, although the bandwidth selection procedures in \cite{ding2022learning} and our current algorithm may appear similar, they are fundamentally different. In \cite{ding2022learning}, the bandwidth is selected based on the {pairwise distances within a single dataset} to ensure that the underlying structure can be effectively captured. In contrast, our bandwidth is chosen based on the {cross-pair distances between the two datasets} to ensure that the {commonly shared information} can be properly extracted. 

Third, as discussed above, our goal is to extract meaningful {commonly shared information} between two datasets for various downstream tasks. To this end, it is important to first assess whether the two datasets potentially share overlapping underlying geometric structures. In other words, we need to verify whether the datasets are {alignable} to avoid introducing potential artifacts. Consequently, unlike the algorithm in \cite{ding2022learning}, our method requires an {alignability screening procedure} applied to the two datasets at the outset. In the current paper, we propose such a step based on the {purity}  of the local neighborhoods across the datasets. This screening ensures that the datasets indeed contain shared structures, so that the resulting joint embeddings are meaningful for studying these commonalities.

Finally, we point out that, since the goals and algorithms are fundamentally different, their {theoretical interpretations} also differ significantly. The embeddings produced by Algorithm 1 in \cite{ding2022learning} are related to the eigenfunctions of {integral operators} associated with an underlying manifold model. In contrast, the joint embeddings produced by our current algorithm can be interpreted in terms of {pairs of eigenfunctions of duo-landmark integral operators} associated with joint manifold models that capture the {commonly shared structures}.
\end{rem}

\begin{rem}[Comparison with joint kernel PCA]\label{rem_algo1}	
Armed with Remark \ref{rem_algo}, we now explain why we prefer to work with the asymmetric $n_1 \times n_2$ kernel affinity matrix constructed from the two datasets $\mathcal{X}$ and $\mathcal{Y}$, rather than directly using the symmetric $(n_1 + n_2) \times (n_1 + n_2)$ kernel affinity matrix formed from the combined dataset $\mathcal{X} \cup \mathcal{Y}$.

First, on the modeling side, choosing to directly work with the $(n_1 + n_2) \times (n_1 + n_2)$ kernel affinity matrix formed from the combined dataset essentially assumes a strong premise: that both datasets arise from {the same underlying nonlinear manifold}. However, motivated by practical applications, our goal is to identify and exploit the shared structures between the two datasets under a joint manifold framework. Specifically, each dataset may follow its own {distinct nonlinear manifold model, but these models may partially overlap}. Our objective is to prioritize these shared structures for downstream analysis.    

Consequently, the interpretation of the embeddings differs significantly. Embeddings derived from the combined dataset rely on the eigenfunctions of a single integral operator associated with one underlying manifold model. In contrast, our proposed embeddings are based on a pair of closely related operators, with each dataset using the eigenfunctions of its respective operator. Since these operators are tightly connected, the resulting embeddings are naturally aligned through them, thereby emphasizing the shared structures between the datasets.  

In summary, when the two datasets share only partial structures, treating them as if they originate from identical structures—by using the combined dataset—may lead to poor approximations. For example, as demonstrated in our simulation studies, when $\mathcal{X}$ and $\mathcal{Y}$ only partially overlap, using the combined dataset can introduce artifacts.

Second, on the methodological side, building on the above model discussion, the two approaches require different prescreening or validation procedures. For our problem of interest, it suffices to check whether the two datasets are alignable—that is, whether they share some overlapping structures. This check can be efficiently performed using the first step of our algorithm. In contrast, using the combined dataset requires verifying that the two datasets have identical structures, which is a much more challenging—if not infeasible—task.

Third, despite the key differences in modeling and interpretation, our method offers additional advantages. For instance, when combining two datasets, the bandwidth selection procedure in \cite{ding2022learning} requires that both datasets have similar signal-to-noise ratios (SNRs)—a condition that is often unmet in practice. In contrast, our method, particularly the bandwidth selection step, accommodates differing SNRs across datasets. This flexibility enables us to extract meaningful information from a noisier dataset by pairing it with a cleaner one that shares common structures.

 Nevertheless, if the two datasets originate from the same manifold model and have identical SNRs, then combining them can indeed be numerically beneficial due to the increased sample size. In this case, our method still applies—the shared structure corresponds to the common manifold model, and the pair of operators becomes identical, though the convergence rate may be slower. However, verifying this strong assumption is a much more challenging, if not infeasible, task. In summary, considering all the above factors, we prefer to work with the asymmetric matrix, as it better aligns with the nature of our problem and offers several methodological advantages.
\end{rem}

\begin{rem}[Technical innovations]\label{rem_proofremark} 
First, before stating the key technical novelties, we explain the theoretical novelties of the proposed joint model, the common structures, and the associated pairs of operators used to extract these structures and characterize their properties, as they form the foundation for interpreting our proposed algorithms. As discussed earlier, our goal is to learn the shared common structures between the two datasets. Our algorithm is fundamentally based on analyzing the singular vectors of a single asymmetric rectangular kernel random matrix. Consequently, to provide a theoretical interpretation, we must construct a pair of integral operators that share the same spectrum but have different, yet related, eigenfunctions.
To achieve this, we first introduce a pair of { convolutional landmark kernels} in Definition \ref{defn_clmd}. Based on these kernels, we construct two integral operators. Through careful analysis of these operators and the underlying kernels, we prove in Proposition \ref{eigenvalue.prop} that when the two datasets possess common structures and are alignable, the resulting integral operators { have identical spectra, and their eigenfunctions are closely related} and can inform each other. {To the best of our knowledge}, the existing RKHS literature focuses on single-view settings, where only one integral operator is typically used to explain the embeddings. If related work does exist, we are among the first to introduce a pair of integral operators by leveraging pairs of convolutional kernels in the multi-view setting.

Second, we explain the key technical novelties of Theorem \ref{thm_cleanconvergence} from two perspectives. {On the one hand}, unlike the existing literature on RKHS, which typically deals with symmetric kernel matrices, our analysis involves Gram matrices derived from an asymmetric matrix. In classical RKHS settings, the entries of the symmetric kernel matrix can be directly linked to the underlying kernel function, allowing spectral convergence to the associated integral operator to be established via the law of large numbers. In contrast, in our setting, the matrices $\mathbf{N}_{01}$ and $\mathbf{N}_{02}$ in (\ref{eq_twomatrices}) of the revised manuscript are {Gram matrices formed by the inner products of columns or rows of an asymmetric matrix} (see equation (\ref{eq_cite}). As a result, their entries cannot be directly associated with the given kernel function. Instead, {we must relate them to a newly defined convolution kernel, leveraging concentration inequalities and certain auxiliary quantities} (cf. equation (\ref{eq_intermediatematrix})). Based on this connection, we then relate the Gram matrices to the corresponding landmark integral operators. {Our analysis does not rely on particular assumptions on $n_1$ and $n_2$ to allow flexibility on the batch sizes.} 
On the other hand, compared to the existing literature—which typically imposes stronger assumptions on the spectral separation of the operators—{we relax such assumptions}. In particular, we allow for the presence of repeated eigenvalues and still establish convergence for the corresponding eigenspaces. This is achieved through a more refined analysis of the resolvent of the operators.

Third, for Theorem \ref{thm_noiseconvergence}, which serves as a counterpart to Theorem \ref{thm_cleanconvergence} in the setting where the signals are corrupted by high-dimensional noise, we address similar challenges, including the complicated asymmetric signal-plus-noise kernel matrix, weaker spectral separation, and potentially differing magnitudes of $n_1$, $n_2$, and $p$. In addition, due to the presence of noise across different datasets, there are additional cross terms. {These terms illustrate how varying signal-to-noise ratios (SNRs) across datasets can affect the performance of our algorithms}. To manage these terms, we leverage recent advances in random matrix theory, which allow us to control their effect without requiring restrictive assumptions on the relative sizes of $n_1$, $n_2$, and $p$, as long as they are sufficiently large.

{Fourth, we have added a new section—Section \ref{sec_phasetransition}—and a new result—Theorem \ref{thm_noisededuction}—to demonstrate that our proposed algorithm exhibits an (almost) sharp phase transition with respect to the signal-to-noise ratios.} In the weak signal regime (i.e., when Assumption \ref{assum_dimensionalityandsnr} does not hold), the spectrum can be described by the free multiplicative convolution of two Marchenko–Pastur laws. This result provides a practical tool to assess whether both datasets are dominated by noise, enabling a pre-check to avoid generating artifacts—i.e., to verify the applicability of Theorem \ref{thm_noiseconvergence}. The proof relies on extending results from kernel random matrix theory and leveraging tools from free probability theory, which, while well-established in probability, have been seldom applied in statistics. Unlike classical random matrix theory results, which typically require comparable dimensions for $n_1$, $n_2$, and $p$, our approach avoids such restrictive assumptions.

Finally, since the kernel functions are all bounded, the convergence results in Theorem \ref{thm_cleanconvergence} do not rely on the dimensionality $r_1$ or $r_2.$ Moreover, as in Step 1 of Algorithm \ref{al0}, we always choose the bandwidth using the same scheme regardless of whether the data is clean or noisy. As can be seen in Proposition \ref{lem_bandwidthconcentration}, such a bandwidth is useful in the sense that meaningful information associated with the operators $\mathcal{K}_1$ and $\mathcal{K}_2$ can be recovered from the noisy datasets.
\end{rem}

\begin{rem}[Different SNRs treatment]\label{rem_noise}
Our algorithm is designed to accommodate differing SNRs across datasets. This flexibility allows it not only to learn shared structures from two datasets with relatively strong SNRs but also to extract meaningful information from a noisier dataset by leveraging a cleaner one with common structures. 

We begin by discussing the case where at least one of the datasets exhibits a strong signal-to-noise ratio (SNR), meaning the signal dominates the noise. Step One of our proposed algorithm  is designed to determine whether the two datasets are alignable. In scenarios where one dataset consists purely of noise or where the noise completely overwhelms the signal, while the other dataset has a strong signal, our algorithm recommends against jointly analyzing the two datasets, as they are unlikely to share any meaningful common structures. If Step One is passed, we proceed to extract shared structures between the datasets. Our SNR assumption in equation (\ref{eq_sigmaimagnititude}) requires that the combined signal strength across the two datasets dominates the combined noise. Beyond the straightforward case where both datasets have strong signals, our algorithm is also capable of handling situations where one dataset has relatively weaker signals, provided these can be learned through the information in the stronger dataset; see Remark \ref{rem_snsnsnsnsndiscussion} for more discussions.

Nevertheless, we point out that when both datasets consist purely of noise or when noise completely dominates the signals, it is still possible for such cases to pass Step One of our algorithm. However, in this scenario, the resulting spectrum cannot be described by our duo-landmark integral operators; instead, it aligns with distributions predicted by random matrix theory, as characterized in the newly added Theorem \ref{thm_noisededuction}. To avoid artifacts introduced by such cases, we can use a screening procedure based on random matrix theory to filter them out prior to applying our algorithm 

In summary, our proposed algorithm can accommodate various combinations of SNR levels across the two datasets, allowing for flexible and robust applications.
\end{rem}

\begin{rem}[Generalization to more than two datasets]\label{rem_morethanthree}
A potential advantage of our proposed method indeed lies in its flexibility to be generalized to handle three or more datasets. Our proposed algorithm is based on analyzing the singular values and vectors of a normalized $n_1 \times n_2$ kernel matrix $\frac{1}{\sqrt{n_1 n_2}}\mathbf{K}$. Specifically, we are interested in the eigenvalues and eigenvectors of the following two matrices:
 \begin{equation*} 
 \mathbf{N}_1 = \frac{1}{n_1 n_2} \mathbf{K} \mathbf{K}^\top, \quad \mathbf{N}_2 = \frac{1}{n_1 n_2} \mathbf{K}^\top \mathbf{K}. 
 \end{equation*} 
 By the classical trick of Hermitian dilation, it suffices to study the following $(n_1 + n_2) \times (n_1 + n_2)$ matrix:  
\begin{equation*}
\mathcal{K}:=\begin{pmatrix}
\mathbf{0} & \frac{1}{\sqrt{n_1 n_2}}\mathbf{K} \\
\frac{1}{\sqrt{n_1 n_2}}\mathbf{K}^\top & \mathbf{0}
\end{pmatrix},
\end{equation*} 
 since the positive nonzero eigenvalues of $\mathcal{K}$ coincide with those of $\mathbf{N}_1$ and $\mathbf{N}_2$, and the eigenvectors of $\mathcal{K}$ are constructed from those of $\mathbf{N}_1$ and $\mathbf{N}_2$.

Based on the above discussion, suppose we observe $\mathsf{r}$ datasets $\mathcal{X}_1, \cdots, \mathcal{X}_{\mathsf{r}}$ with corresponding sample sizes $n_1, n_2, \cdots, n_{\mathsf{r}}$. For each pair $1 \leq i \neq j \leq \mathsf{r}$, let $\mathbf{K}_{ij} \in \mathbb{R}^{n_i \times n_j}$ denote the rectangular kernel matrix constructed from the datasets $\mathcal{X}_i$ and $\mathcal{X}_j$. Using these, we construct a matrix of size $\left( \sum_{i=1}^{\mathsf{r}} n_i \right) \times \left( \sum_{i=1}^{\mathsf{r}} n_i \right)$ as follows:
\begin{equation*}
\mathcal{K}':=\begin{pmatrix}
\mathbf{0} & \frac{1}{\sqrt{n_1 n_2}}\mathbf{K}_{12} & \ldots & \frac{1}{\sqrt{n_1 n_\mathsf{r}}}\mathbf{K}_{1\mathsf{r}}  \\
\frac{1}{\sqrt{n_1 n_2}}\mathbf{K}_{21} & \mathbf{0} & \ldots & \frac{1}{\sqrt{n_2 n_\mathsf{r}}}\mathbf{K}_{2\mathsf{r}} \\
\vdots & \vdots & \ddots & \vdots \\ 
 \frac{1}{\sqrt{n_1 n_\mathsf{r}}}\mathbf{K}_{\mathsf{r} 1}  & \frac{1}{\sqrt{n_2 n_\mathsf{r}}}\mathbf{K}_{\mathsf{r}2} & \ldots &  \mathbf{0} 
\end{pmatrix} ,
\end{equation*}
 where we note that $\mathbf{K}_{ij}=\mathbf{K}_{ji}^\top$. Therefore, the simultaneous embeddings can be obtained by analyzing the matrix constructed above.

Based on the structure of $\mathcal{K}'$, when $\mathsf{r}$ is finite, each block corresponds to a two-view case, and the overall structure aggregates these blocks in a nonlinear manner via the matrix $\mathcal{K}'$. On the modeling and theoretical side, the joint manifold model can be naturally extended from two datasets to $\mathsf{r}$ datasets, leading to a sequence of convolutional kernels and, consequently, a collection of pairwise duo-landmark integral operators. The theoretical integration of these components yields eigenfunctions of a matrix-valued operator $\Omega$, where $\Omega \in \mathbb{R}^{\mathsf{r} \times \mathsf{r}}$ and each entry is itself an integral operator. Specifically, for $1 \leq i \neq j \leq \mathsf{r}$, the entries $\Omega_{ij}$ and $\Omega_{ji}$ are defined using the duo-landmark integral operators associated with the datasets $\mathcal{X}_i$ and $\mathcal{X}_j$, as described in Definition \ref{defn_landmarkintegral}. Similarly, with regard to algorithmic robustness, a model reduction scheme can still be employed, now based on a block matrix structure. Since this extension is beyond the main focus of the present work, we leave a detailed investigation for future research.
\end{rem}
}

\section{Technical preparation}\label{appendix_preliminaryresults}

\subsection{Preliminary results on reproducing kernel Hilbert space}\label{append_RMRKHS}
In this section, we summarize some results regarding the reproducing kernel Hilbert space (RKHS). Most of the results can be found in \cite{JMLR:v7:braun06a, BJKO, JMLR:v11:rosasco10a, ICMLko, AOSko,Smale2007LearningTE, MR2558684}.
Consider that we observe $n$ i.i.d. samples $\{\xb_i\}_{1\le i\le n}$ drawn from some probability distribution $\sfP$ in $\mathbb{R}^{p}$. Then the population integral operator $\mathcal{K}$ with respect to  $\sfP$ and the reproducing kernel
$
k({\bm x},{\bm y})$, ${\bm x}, {\bm y} \in \text{supp}(\sfP),$
is defined by 
\begin{equation}  \label{K}
	{\mathcal{K}} g({\bm x})=\int k(\bx,\by)g(\by) {\sfP}(\mathrm{d} \by) , \qquad \bx, \by\in \text{supp}(\sfP),
\end{equation}
and its empirical counterpart $\mathcal{K}_n$ is defined by
\begin{equation} \label{K_n}
	{\mathcal{K}_n} g(\bx)=\int k(\bx,\by) g( \by) {\sfP}_n(\mathrm{d} \by)=\frac{1}{n}\sum_{i=1}^n k(\bx,\xb_i) g(\xb_i), \qquad \bx \in \text{supp}(\sfP),
\end{equation}
where $\sfP_n$ is the empirical CDF of $\{\xb_i\}.$ The  RKHS $\mathcal{H}_K$ associated with the kernel function $k(\bx,\by)$	
in (\ref{K}) and (\ref{K_n}) is defined as the completion of the linear span of the set of functions $\{k_{\bx}=k(\bx,\cdot): \xb \in \text{supp}(\sfP)\}$ with the inner product denoted as $\langle \cdot,\cdot \rangle_K$ satisfying $\langle k(\bx, \cdot), k(\by, \cdot)\rangle_K=k(\bx,\by)$ and the reproducing property $\langle k(\bx, \cdot), g\rangle_K=g(\bx)$ for any $g\in \mathcal{H}_K$.
Note that $\mathcal{K}$ and $\mathcal{K}_n$ may be considered as self-adjoint operators on $\mathcal{H}_K$, or on their respective $\mathcal{L}_2$ spaces (that is, $\mathcal{L}_2(\Omega,\sfP)$ and $\mathcal{L}_2(\Omega,\sfP_n)$, $\Omega=\text{supp}(\sfP)$). 

It is easy to see that (for example, Section 2.2 of \cite{AOSko}) the eigenvalues of $\mathcal{K}_n$ coincide with $n^{-1} \bK$, where $\bK=(k(\xb_i,\xb_j))_{1\le i,j\le n}$. The eigenfunctions $\{\widehat\phi_i\}$ associated with nonzero eigenvalues $\{\mu_i\}$ of $\mathcal{K}_n$ or $n^{-1}\bK$ satisfy
\beq \label{eigenvec}
\widehat{\phi}_i(\bx)=\frac{1}{\mu_i\sqrt{n}} \sum_{j=1}^n k\left(\bx,\xb_j\right) v_{i j}, \qquad \bx \in \text{supp}(\sfP),
\eeq
where $\vb_i=(v_{i1},v_{i2},...,v_{in})^\top$ is the $i$-th eigenvector of $n^{-1} \bK$, and that  $\|\widehat\phi_i\|_{\sfP_n}=1$.
The following lemma concerns the convergence of the eigenvalues and eigenfunctions of $\mathcal{K}_n$ to those of $\mathcal{K}$. 

{

\begin{lem}\label{eigenvalue.lem1} For the self-adjoint operators $\mathcal{K}$ and $\mathcal{K}_n$ on $\mathcal{H}_K$, defined by (\ref{K}) and (\ref{K_n}),  we have 
	\begin{equation}\label{eq_evlimitclose}
		\| \mathcal{K}-\mathcal{K}_n \| \prec \frac{1}{\sqrt{n}}.
	\end{equation}
	Let $\{\gamma_i\}$ be the nonincreasing eigenvalues of $\mathcal{K}$ with multiplicity and $\{\phi_i\}$ be the corresponding eigenfunctions. For each nonzero eigenvalue $\gamma_i$, we define the index set ${\sf I}\equiv {\sf I}(i)$ such that $i\in {\sf I}$ and for any $t\in {\sf I}$ we have $\gamma_t=\gamma_i$ and $\mathsf{r}_i:= \min_{j\in{\sf I}^c}|\gamma_i-\gamma_j|$, where ${\sf I}^c=\mathbb{N}\setminus {\sf I}$. We define the projection operators $\mathcal{P}_{\sf I}=\sum_{i\in {\sf I}}\gamma_i\phi_i\phi_i^*$ and $\mathcal{P}^{(n)}_{\sf I}=\sum_{i\in {\sf I}}\mu_i\phi^{(n)}_i(\phi^{(n)}_i)^*$ where $\{\phi_i^{(n)}\}$ are the eigenfunctions of $\mathcal{K}_n$ associated to the nonincreasing eigenvalues $\{\mu_i\}$ with multiplicity. Then, if $n^{-1/2}=o(\delta_i)$, it follows that
	\begin{equation}\label{eq_eigenfunction_conv}
	\|\mathcal{P}_{\sf I}-\mathcal{P}^{(n)}_{\sf I}\|_{\mathcal{H}_K\to \mathcal{H}_K} \prec \frac{1}{ \mathsf{r}_i \sqrt{n}},
	\end{equation}
	where $\| \sqrt{\gamma_i}{\phi_i}\|_K=\|\sqrt{\mu_i}\phi_i^{(n)} \|_K=1$, and $\|\phi_i^{(n)} \|_{\sfP_n}=1$ for each $i\in {\sf I}$. Moreover, let ${\bm \phi}_{\sf I}$ and ${\bm \phi}^{(n)}_{\sf I}$ be $|{\sf I}|$-dimensional vector-valued functions consisting of component functions $(\phi_i)_{i\in{\sf I}}$ and $(\phi_i^{(n)})_{i\in \sf I}$, respectively. It then follows that
		\begin{equation}\label{eq_eigenfunction2_conv}
		\inf_{\bO\in O(|{\sf I}|)}\|\sqrt{\gamma_i}{\bm \phi}_{\sf I}-\sqrt{\mu_i}\bO{\bm \phi}^{(n)}_{\sf I}\|_{K,\infty} \prec \frac{1}{ \mathsf{r}_i \sqrt{n}},
	\end{equation}
	where $\|(f_1,f_2,...,f_m)\|_{K,\infty}:= \max_{i\in m}\|f_i\|_K$.
\end{lem}
\begin{proof}
	Equation (\ref{eq_evlimitclose}) is proven in \cite[Proposition 1]{MR2558684}, \cite[Theorem 7]{JMLR:v11:rosasco10a} and \cite[Lemma 1]{ding2022learning} for the more general class of bounded and positive definite kernel functions. As a result, the convergence (\ref{eq_eigenfunction_conv}) of the eigenspace projection operator then follows from Lemma \ref{eigenvector.lem}. Note that in $\mathcal{H}_K$, the normalized eigenfunctions of $\mathcal{K}$ and $\mathcal{K}_n$ are $\{\sqrt{\gamma_i}\phi_i\}$ and $\{\sqrt{\mu_i}\phi_i^{(n)}\}$, respectively. Finally, we prove (\ref{eq_eigenfunction2_conv}).  Let $(\langle \sqrt{\gamma_i}\phi_j,\sqrt{\mu_i}\phi_k^{(n)} \rangle_K)_{j,k\in{\sf I}}=\bA\bSig\bB^\top$ be the SVD of the matrix $(\langle \sqrt{\gamma_i}\phi_j,\sqrt{\mu_i}\phi_k^{(n)} \rangle_K)_{j,k\in{\sf I}}$. Note that
	\begin{align*}
			&\inf_{\bO\in O(|{\sf I}|)}\|\sqrt{\gamma_i}{\bm \phi}_{\sf I}-\sqrt{\mu_i}\bO{\bm \phi}^{(n)}_{\sf I}\|_{K,\infty} \\
			&\le 	\|\sqrt{\gamma_i}{\bm \phi}_{\sf I}-\sqrt{\mu_i}\bA\bB^\top {\bm \phi}^{(n)}_{\sf I}\|_{K,\infty} \\
			&\le \sqrt{	\|\mathcal{P}^{(n)}_{\sf I}(\sqrt{\gamma_i}{\bm \phi}_{\sf I}-\sqrt{\mu_i}\bA\bB^\top {\bm \phi}^{(n)}_{\sf I})\|_{K,\infty}^2+	\|(\text{id}-\mathcal{P}^{(n)}_{\sf I})(\sqrt{\gamma_i}{\bm \phi}_{\sf I}-\sqrt{\mu_i}\bA\bB^\top {\bm \phi}^{(n)}_{\sf I})\|_{K,\infty}^2}\\
			&\le \sqrt{	\|\sqrt{\gamma_i}\mathcal{P}^{(n)}_{\sf I}{\bm \phi}_{\sf I}-\sqrt{\mu_i}\bA\bB^\top {\bm \phi}^{(n)}_{\sf I}\|_{K,\infty}^2+	\|\sqrt{\gamma_i}(\text{id}-\mathcal{P}^{(n)}_{\sf I}){\bm \phi}_{\sf I}\|_{K,\infty}^2}.
				\end{align*}
				Since
				\[
				\sqrt{\gamma_i}\mathcal{P}^{(n)}_{\sf I}{\bm \phi}_{\sf I}=\sqrt{\gamma_i}\sum_{j\in \sf I}\mu_i\phi_j^{(n)}(\phi_j^{(n)})^* {\bm \phi}_{\sf I}=\bigg(
					\sqrt{\mu_i}\sum_{j\in \sf I}\phi_j^{(n)} \langle \sqrt{\mu_i}\phi_j^{(n)},\sqrt{\gamma_i}\phi_\ell\rangle_K\bigg)_{\ell\in \sf I}=\sqrt{\mu_i}\bA\bSig\bB^\top{\bm \phi}_{\sf I}^{(n)},
				\]
				we also have
				\begin{align*}
					\inf_{\bO\in O(|{\sf I}|)}\|\sqrt{\gamma_i}{\bm \phi}_{\sf I}-\sqrt{\mu_i}\bO{\bm \phi}^{(n)}_{\sf I}\|_{K,\infty} &\le \sqrt{	\|\sqrt{\mu_i}\bA(\bSig-{\bm I})\bB^\top{\bm \phi}_{\sf I}^{(n)}\|_{K,\infty}^2+	\|\sqrt{\gamma_i}(\text{id}-\mathcal{P}^{(n)}_{\sf I}){\bm \phi}_{\sf I}\|_{K,\infty}^2} \\
					&\le  \sqrt{\max_{j}\| [\bA(\bSig-{\bm I})\bB^\top]_{j.}\|_2^2+	\|\sqrt{\gamma_i}(\text{id}-\mathcal{P}^{(n)}_{\sf I}){\bm \phi}_{\sf I}\|_{K,\infty}^2}\\
					&\le \sqrt{	1-\sigma_{\min}^2+	\|\sqrt{\gamma_i}(\text{id}-\mathcal{P}^{(n)}_{\sf I}){\bm \phi}_{\sf I}\|_{K,\infty}^2},
					\end{align*}
					where $\sigma_{\min}$ is the smallest diagonal element in $\bSig$, and in the last
%
					 inequality we used the fact that $\max_i\|{\bm T}_{i.}\|_2\le \|{\bm T}\|$ for any matrix ${\bm T}$. Now since
					\begin{align*}
						\sigma_{\min}^2:&=\inf_{x:\|x\|_2=1}\|(\langle \sqrt{\gamma_i}\phi_j,\sqrt{\mu_i}\phi_k^{(n)} \rangle_K)_{j,k\in{\sf I}}x\|_2^2=\inf_{x:\|x\|_2=1}\| \sqrt{\gamma_i}x^\top \mathcal{P}_{\sf I}^{(n)}{\bm\phi}_{\sf I}\|_K^2\\
						&=\inf_{x:\|x\|_2=1}\bigg[\| \sqrt{\gamma_i}x^\top{\bm\phi}_{\sf I}\|_K^2-\| \sqrt{\gamma_i}(\text{id}- \mathcal{P}_{\sf I}^{(n)})x^\top{\bm\phi}_{\sf I}\|_K^2\bigg]\\
						&=\inf_{x:\|x\|_2=1}\| \sqrt{\gamma_i}x^\top{\bm\phi}_{\sf I}\|_K^2-\sup_{x:\|x\|_2=1}\| \sqrt{\gamma_i}(\text{id}- \mathcal{P}_{\sf I}^{(n)})\mathcal{P}_{\sf I}x^\top{\bm\phi}_{\sf I}\|_K^2\\
						&\ge 1-\| \sqrt{\gamma_i}(\text{id}- \mathcal{P}_{\sf I}^{(n)})\mathcal{P}_{\sf I}\|_{\mathcal{H}_K\to \mathcal{H}_K}^2,
						\end{align*}
						we have
						\[
						1-\sigma_{\min}^2\le \| \sqrt{\gamma_i}(\text{id}- \mathcal{P}_{\sf I}^{(n)})\mathcal{P}_{\sf I}\|_{\mathcal{H}_K\to \mathcal{H}_K}^2=\|(\text{id}-\mathcal{P}_{\sf I}^{(n)})(\mathcal{P}_{\sf I}-\mathcal{P}_{\sf I}^{(n)})\|_{\mathcal{H}_K\to \mathcal{H}_K}^2\le 	\|\mathcal{P}_{\sf I}-\mathcal{P}_{\sf I}^{(n)}\|_{\mathcal{H}_K\to \mathcal{H}_K}^2.
						\]
					On the other hand, we have
					\begin{align*}
						\|\mathcal{P}_{\sf I}-\mathcal{P}_{\sf I}^{(n)}\|_{\mathcal{H}_K\to \mathcal{H}_K}&\ge 	\|(\text{id}-\mathcal{P}_{\sf I}^{(n)})(\mathcal{P}_{\sf I}-\mathcal{P}_{\sf I}^{(n)})\|_{\mathcal{H}_K\to \mathcal{H}_K}\\
						&=\|(\text{id}-\mathcal{P}_{\sf I}^{(n)})\mathcal{P}_{\sf I}\|_{\mathcal{H}_K\to \mathcal{H}_K}\\
						&\ge \|\sqrt{\gamma_i}(\text{id}-\mathcal{P}_{\sf I}^{(n)})\mathcal{P}_{\sf I}{\bm \phi}_{\sf I}\|_{K,\infty}.
					\end{align*}
					Combining the above inequalities, we have
					\[
						\inf_{\bO\in O(|{\sf I}|)}\|\sqrt{\gamma_i}{\bm \phi}_{\sf I}-\sqrt{\mu_i}\bO{\bm \phi}^{(n)}_{\sf I}\|_{K,\infty} \le \sqrt{2}\|\mathcal{P}_{\sf I}-\mathcal{P}_{\sf I}^{(n)}\|_{\mathcal{H}_K\to \mathcal{H}_K}.
					\] 
					This along with (\ref{eq_eigenfunction_conv}) leads to (\ref{eq_eigenfunction2_conv}).
\end{proof}

\begin{lem}\label{eigenvector.lem}
	Let $\mathcal{A}$ and $\widehat{\mathcal{A}}$ be two compact self-adjoint operators on a Hilbert space $H$, with nonincreasing eigenvalues $\{\lambda_i\}$ and $\{\widehat\lambda_j\}$ with multiplicity. Then we have $\max_{j\ge 1}|\lambda_j-\widehat\lambda_j|\le \|\mathcal{A}-\widehat{\mathcal{A}}\|$. Moreover, for any nonzero eigenvalue $\lambda_i$, there exists a (unique and finite) index set ${\sf I}\equiv {\sf I}(i)$ such that $i\in {\sf I}$ and for any $t\in {\sf I}$ we have $\lambda_t=\lambda_i$ and $\delta_i:=\min_{j\in {\sf I}^c}|\lambda_i-\lambda_j|$, where we denote ${\sf I}^c=\mathbb{N}\setminus {\sf I}$. Let $\{w_i\}_{i\in {\sf I}}$ and $\{\widehat w_i\}_{i\in {\sf I}}$ be the sets of normalized eigenvectors of $\mathcal{A}$ and $\widehat{\mathcal{A}}$, associated with eigenvalues $\{\lambda_i\}_{i\in {\sf I}}$ and $\{\widehat \lambda_j\}_{j\in {\sf I}}$, respectively. Define the corresponding orthogonal projection operators $\mathcal{P}_{\sf I}=\sum_{i\in {\sf I}}w_iw_i^*$ and $\widehat{\mathcal{P}}_{\sf I}=\sum_{i\in {\sf I}}\widehat w_i\widehat w_i^*$. If $\ell>0$ satisfies that $\delta_i \ge \ell $ and $\|\mathcal{A}-\widehat{\mathcal{A}}\|< \ell/4$, then we have $\|\mathcal{P}_{\sf I}-\widehat{\mathcal{P}}_{\sf I}\|\le \frac{8}{\ell}\|\mathcal{A}-\widehat{\mathcal{A}}\|$.
\end{lem}

\begin{proof}
	The eigenvalue perturbation bound $\max_{j\ge 1}|\lambda_j-\widehat\lambda_j|\le \|\mathcal{A}-\widehat{\mathcal{A}}\|$ follows from Proposition 2 of \cite{MR2558684}. Below we prove the perturbation bound for the eigenspace projection operators. Recall that for compact self-adjoint operators $\mathcal{A}$ and $\widehat{\mathcal{A}}$, by the Spectral Theorem (e.g., Theorem 6.27 of \cite{einsiedler2017functional}), we can write
	\[
	\mathcal{A}=\sum_{i=1}^{\infty}\lambda_i w_j w_i^*,\qquad \widehat{\mathcal{A}}=\sum_{i=1}^{\infty}\widehat\lambda_i \widehat w_j \widehat{w}_i^*.
	\] 
	As a result, for some simply connected contour $\Gamma\equiv \Gamma_i\subset \mathbb{C}$ only containing $\lambda_i$ (or $\{\lambda_i\}_{i\in {\sf I}}$), but not other distinct eigenvalues of $\mathcal{A}$, it holds that
	\[
	\mathcal{P}_{\sf I}=\frac{1}{2\pi \ri} \oint_\Gamma (z{\bf I}-\mathcal{A})^{-1} \mathrm{d} z.
	\]
	Now if we choose the contour $\Gamma_i: ={\sf B}(\lambda_i,\ell/2)$ as the disk centered at $\lambda_i$ with radius $\ell/2$, since by the first part of the theorem $$\max_{i\in{\sf I}}|\lambda_i-\widehat\lambda_i|\le \|\mathcal{A}-\widehat{\mathcal{A}}\|< \ell/4,$$ all the eigenvalues in $\{\widehat\lambda_i\}_{i\in {\sf I}}$ also lie within the interior of $\Gamma_i$. On the other hand, since $\delta_i\ge \ell$, for any $j\in {\sf I}^c$, we have
	\[
	|\widehat \lambda_j-\lambda_i|\ge \big||\widehat \lambda_j-\lambda_j|-|\lambda_j-\lambda_i|\big|\ge \delta_i-\|\mathcal{A}-\widehat{\mathcal{A}}\|\ge \delta_i-\ell/4\ge 3\ell/4,
	\]
	so that all the eigenvalues in $\{\widehat \lambda_j\}_{j\in {\sf I}^c}$ lie strictly outside the contour $\Gamma_i$. Thus it follows that
	\[
		\widehat{\mathcal{P}}_{\sf I}=\frac{1}{2\pi \ri} \oint_\Gamma (z{\bf I}-\widehat{\mathcal{A}})^{-1} \mathrm{d} z.
	\]
	As a result, we can write
	\begin{align}
	\|	\mathcal{P}_{\sf I}-	\widehat{\mathcal{P}}_{\sf I}\|&=\bigg\| \frac{1}{2\pi \ri} \oint_\Gamma (z{\bf I}-\mathcal{A})^{-1} \mathrm{d} z-\frac{1}{2\pi \ri} \oint_\Gamma (z{\bf I}-\widehat{\mathcal{A}})^{-1} \mathrm{d} z\bigg\| \nonumber \\
	&=\sup_{ \|u\|_H=1} \bigg| \frac{1}{2\pi \ri} \oint_\Gamma \big( (z{\bf I}-\mathcal{A})^{-1} -(z{\bf I}-\widehat{\mathcal{A}})^{-1}\big)u \mathrm{d} z \bigg| \nonumber\\ 
	&\le \frac{1}{2\pi } \cdot  \oint_\Gamma  \sup_{ \|u\|_H=1}\sup_{z\in \Gamma}\bigg| \big((z{\bf I}-\mathcal{A})^{-1} -(z{\bf I}-\widehat{\mathcal{A}})^{-1} \big) u \bigg| \cdot |\mathrm{d} z|\nonumber \\ 
	&= \frac{\ell}{2}\sup_{ \|u\|_H=1}\sup_{z\in \Gamma}\bigg| \big((z{\bf I}-\mathcal{A})^{-1} -(z{\bf I}-\widehat{\mathcal{A}})^{-1} \big) u \bigg|,\label{eq1.lem}
	\end{align}
	where in the second equation we use the definition of operator norm, in the second last inequality we used triangle inequality for complex integrals.  By the resolvent identity, it follows that
	\[
	(z{\bf I}-\mathcal{A})^{-1} -(z{\bf I}-\widehat{\mathcal{A}})^{-1} =	(z{\bf I}-\mathcal{A})^{-1} (\mathcal{A}-\widehat{\mathcal{A}})(z{\bf I}-\widehat{\mathcal{A}})^{-1},
	\]
	so that
	\begin{align*}
	\sup_{ \|u\|_H=1}\sup_{z\in \Gamma}\bigg| \big((z{\bf I}-\mathcal{A})^{-1} -(z{\bf I}-\widehat{\mathcal{A}})^{-1} \big) u \bigg| &=	\sup_{ \|u\|_H=1}\sup_{z\in \Gamma}\bigg| \big(	(z{\bf I}-\mathcal{A})^{-1} (\mathcal{A}-\widehat{\mathcal{A}})(z{\bf I}-\widehat{\mathcal{A}})^{-1} \big) u \bigg|\\
	&\le 	\sup_{z\in \Gamma}\bigg\|	(z{\bf I}-\mathcal{A})^{-1} (\mathcal{A}-\widehat{\mathcal{A}})(z{\bf I}-\widehat{\mathcal{A}})^{-1}  \bigg\|\\
	&\le 	\sup_{z\in \Gamma}\|	(z{\bf I}-\mathcal{A})^{-1}\| \cdot \|\mathcal{A}-\widehat{\mathcal{A}}\| \cdot	\sup_{z\in \Gamma}\|(z{\bf I}-\widehat{\mathcal{A}})^{-1} \|.
	\end{align*}
	Now note that
	\[
	\sup_{z\in \Gamma}\|	(z{\bf I}-\mathcal{A})^{-1}\| \le \frac{4}{\ell},\qquad 	\sup_{z\in \Gamma}\|	(z{\bf I}-\widehat{\mathcal{A}})^{-1}\| \le \frac{4}{\ell}.
	\]
	Combining with (\ref{eq1.lem}), it follows that
	\[
	\|	\mathcal{P}_{\sf I}-	\widehat{\mathcal{P}}_{\sf I}\|\le \frac{8}{\ell}\|\mathcal{A}-\widehat{\mathcal{A}}\| .
	\]
	This completes the proof.
\end{proof}
}

\begin{lem}\label{lem_RKHSextension} Let $S$ be a measurable space, $\nu$ be a measure on $S$ and $k$ be a measurable kernel on $S$ with RKHS $\mathcal{H}_s.$ Assume that $\mathcal{H}_s$ is compactly embedded into $\mathcal{L}_2(\nu, S).$ Then if $\{\Psi_i\}$ be the eigenfunctions of $\mathcal{H}_s,$ then it is also an orthonormal basis for $\mathcal{L}_2(\nu, S).$  
\end{lem}
\begin{proof}
	See Theorem 3.1 of \cite{steinwart2012mercer}.
\end{proof}

{\subsection{Preliminary results in random matrix theory}\label{sec_backgroundinrmt} In this section, we summarize some results on random matrix theory which will be used in our technical proofs. Let $\mu_1$ and $\mu_2$ be two measures supported on $\mathbb{R}.$ Recall that the Stieltjes transforms for $\mu_k, k=1,2,$ are
\begin{equation}\label{eq_stieltjestransform}
m_k(z)=\int \frac{1}{x-z} \mu_k(\mathrm{d} x), \ z \in \mathbb{C} \backslash \mathbb{R}. 
\end{equation}
It is well-known that if $\mu_k$ has a continuous density function, denoted as $\varrho_k,$ it can be recovered from $m_k(z)$ using the famous inversion formula \cite{bai2010spectral}.

 We first introduce the definition of free multiplicative convolutions of two measures \cite{benaych2011eigenvalues,mingo2017free,voiculescu1991limit} which will be used to describe our results in Section \ref{sec_phasetransition}.  Based on the Stieltjes transforms (\ref{eq_stieltjestransform}),  the $S$-transforms for $\mu_k, k=1,2,$ are defined as 
\begin{equation*}
S_k(z):=-\frac{z+1}{z} \Omega_k^{-1}(z+1), \ \text{where} \ \ \Omega_k(z):=\frac{m_k(-1/z)}{z}. 
\end{equation*} 
The free multiplicative convolution of $\mu_1$ and $\mu_2$ is defined using the $S$-transform \cite{mingo2017free} as follows. 
\begin{defn}\label{defn_freemultiplicativeconvuliton} For two measures $\mu_1$ and $\mu_2$ supported on $\mathbb{R},$ the free multiplicative convolution of $\mu_1$ and $\mu_2,$ denoted as $\mu_1 \boxtimes \mu_2,$ is defined via its $S$-transform in the sense that  
\begin{equation*}
S_{\mu_1 \boxtimes \mu_2}(z)=S_1(z)S_2(z). 
\end{equation*}
\end{defn}

The following lemma shows the limiting spectral distribution (LSD) of  the product of two independent sample covariance matrices can be described as the free multiplicative convolution of two Marchenco-Pastur laws. Recall that $\mathbf{W}_1 \in \mathbb{R}^{n_1 \times p}$ is the purely noise matrix containing $\{\bm{\xi}_i\}$ and $\mathbf{W}_2 \in \mathbb{R}^{n_2 \times p}$ is the purely noise matrix containing $\{\bm{\zeta}_j\}.$  It is known that under the assumption of (\ref{eq_detaileddimensionregime}), the LSD of the nonzero eigenvalues of $\frac{1}{\sqrt{n_k p} \sigma_k^2} \mathbf{W}^\top_k \mathbf{W}_k$ \cite{bloemendal2014isotropic} convergences to the celebrated Marchenco-Pastur (MP) law whose density functions are recorded as in (\ref{eq_MPlawforms}). For more detailed discussion, the results can be found when $p, n_1$ and $n_2$ are comparably large in the monograph \cite{bai2010spectral}, and in the general setting when (\ref{eq_detaileddimensionregime}) holds in \cite{bloemendal2014isotropic,ding2023global}.  Based on this, we can establish the spectral convergence of the eigenvalues of the products of the two matrices 
\begin{equation*}
\mathbf{P}=\frac{1}{p\sigma_1^2 \sigma_2^2 \sqrt{n_1 n_2}} \mathbf{W}_1^\top \mathbf{W}_1 \mathbf{W}_2^\top \mathbf{W}_2.  
\end{equation*}      
Denote the empirical spectral distribution (ESD) of the nonzero eigenvalues of $\mathbf{P}$ as $\mu_n.$ Recall that $\mu_{\SMP_1 \boxtimes \SMP_2}$ is the free multiplicative convolution of the two MP laws in (\ref{eq_MPlawforms}). Moreover, we denote its typical locations as follows. Recall that $\mathsf{n}=\min\{n_1, n_2, p\}.$ For $1 \leq i \leq \mathsf{n},$ we denote  $\gamma_i$ in the sense that
\begin{equation*}
\int_{-\infty}^{\gamma_i} \mu_{\SMP_1 \boxtimes \SMP_2}(\mathrm{d} x)=\frac{i}{\mathsf{n}}. 
\end{equation*}   

\begin{lem}\label{lem_free}
We suppose Assumption \ref{assum_mainassumption}  and the assumption of (\ref{eq_detaileddimensionregime}) hold. Then when $\mathsf{n}$ is sufficiently large, we have that 
\begin{equation*}
\mu_n \Rightarrow \mu_{\SMP_1 \boxtimes \SMP_2}.
\end{equation*}
Moreover, for all $1 \leq i \leq \mathsf{n},$ let $\lambda_1 \geq \lambda_2 \geq \cdots$ be the eigenvalues of $\mathbf{P},$ we have that
\begin{equation*}
|\lambda_i-\gamma_i|=\mathrm{o}_{\prec}(1).
\end{equation*}
\end{lem}
\begin{proof}
The first part of the results follows from Chapter 4 of \cite{mingo2017free}. For the second part, when $p, n_1, n_2$ are all of comparable magnitude, the results have been established in \cite{ji2023local}. In the general setting (\ref{eq_detaileddimensionregime}), since each individual matrix follows the Marchenko–Pastur law as in \cite{bloemendal2014isotropic}, one can follow the arguments in \cite{ji2023local} verbatim to complete the proof.   
\end{proof}

}

\subsection{Asymptotic behavior of oracle and empirical bandwidths} \label{band.sec}

We can prove the following proposition concerning the convergence of the empirical bandwidth $h_n$ defined in (\ref{eq_bandwidthselection}) and $h_n^0$ defined in (\ref{eq_clearnsignalkernel}). Recall the parameters $\theta_1,...,\theta_r$ defined in (\ref{eq_covstructure}). {
\begin{prop}\label{lem_bandwidthconcentration}
	Suppose the assumptions of Theorem \ref{thm_noiseconvergence} hold and $h_n$ is selected according to Algorithm \ref{al0} for any fixed $\omega\in(0,1)$, and $h_n^0$ is defined through (\ref{eq_clearnsignalkernel}). Then we have $\sum_{i=1}^r\theta_i\prec h^0_n \prec \sum_{i=1}^r\theta_i$ and $|h_n/h_n^0-1|=\mathrm{O}_{\prec}(\eta)$. Moreover, if instead the assumption of Theorem \ref{thm_noisededuction} holds, we have that $p\sigma^2\prec h_n\prec p\sigma^2$.
\end{prop}

\begin{proof}
	The proof of the first part follows a similar argument as  \cite[Proposition 1]{ding2022learning}. Below we provide the proof of the second part. 

	Note that 
	\begin{equation}\label{eq_originaldecomposition}
		\|\xb_i-\yb_j\|_2^2=\|\xb^0_i-\yb^0_j\|_2^2+\|\bxi_i-\bzeta_j\|_2^2-2(\xb^0_i-\yb^0_j)^\top(\bxi_i-\bzeta_j).
	\end{equation}
	By Lemma \ref{sg.bnd.lem} and the setup in  (\ref{eq_covstructure}), we have
	\beq
	\max_{1\le i\le n_1,1\le j\le n_2}\|\xb^0_i-\yb^0_j\|_2^2\prec \sum_{i=1}^r\theta_i.
	\eeq
	By Lemma \ref{lem_concentrationinequality} and Assumption \ref{assum_mainassumption}, we have that
	\beq
	\max_{i,j}\|\bxi_i-\bzeta_j\|_2^2\prec \sigma^2p. \nonumber
	\eeq
	By Lemma \ref{sg.bnd.lem}, we have
	\beq \label{cross.pd0}
\max_{i,j}|({\bxi}_i-\bzeta_j)^\top(\xb_i^0-\yb_j^0)|\prec \sigma(\sum_{i=1}^r\theta_i)^{1/2}.
	\eeq
	Combining the above arguments, we find that 
	\begin{align}
		\max_{i,j}\|\xb_i-\yb_j\|_2^2\prec \sigma^2p+\sigma (\sum_{i=1}^r\theta_i)^{1/2}\prec \sigma^2 p.\label{y.bnd}
	\end{align}
	On the other hand, by Assumption \ref{assum_mainassumption} and Lemma \ref{HW.lem}, and the fact that $\E\|\bxi_i-\bzeta_j\|_2^2=2\sigma^2p$, we have
	\[
	\min_{i,j}\|\bxi_i-\bzeta_j\|_2^2\succ \sigma^2 p.
	\]
	This along with (\ref{cross.pd0}) implies for all pairs $(i,j)$
	\[
  \|\xb_i-\yb_j\|_2^2\ge 	\|\bxi_i-\bzeta_j\|_2^2-\|\xb^0_i-\yb^0_j\|_2^2-|2(\xb^0_i-\yb^0_j)^\top(\bxi_i-\bzeta_j)|\succ \sigma^2 p.
	\]
	This implies that any finite percentile $h_n$ of $\{\|\xb_i-\yb_j\|_2^2\}$, we have $p\sigma^2\prec h_n\prec p\sigma^2$. 
	\end{proof}
}

\section{Technical Proofs}\label{sec_techinicalproofappendix}

\subsection{Construction of auxiliary RKHSs}\label{sec_auxililaryincrement}
Note that in general $\mathcal{S}_1$ and $\mathcal{S}_2$ defined in (\ref{eq_setstwo}) may not be the same. In addition, the distributions $\widetilde{\mathsf{P}}_1$ and $\widetilde{\mathsf{P}}_2$ are essentially different so that for the same kernel function $k(\cdot,\cdot),$ it may result in different Mercer's expansions. To address this issue, we introduce two auxiliary RKHSs. 

Denote 
\begin{equation*}
	\mathcal{S}=\mathcal{S}_1 \cup \mathcal{S}_2. 
\end{equation*}
On $\mathcal{S},$ we define two probability measures, denoted as $\mathring{\mathtt{P}}_1$ and $\mathring{\mathtt{P}}_2$ as follows. For $\bm{w} \in \mathcal{S},$ we denote $\mathring{\mathtt{P}}_t, t=1,2,$ that
\begin{equation}\label{eq_probmainmeasure}
	\mathring{\mathtt{P}}_t(\bm{w})=
	\begin{cases}
		\widetilde{\mathsf{P}}_t(\bm{w}), & \bm{w} \in \mathcal{S}_t \\
		0, & \ \text{Otherwise}
	\end{cases}.
\end{equation} 
Based on the above two measures, we can construct RKHSs on $\mathcal{L}_2(\mathcal{S}).$ According to Mercer's theorem (see Theorem 2.10 of \cite{scholkopf2002learning}), for $(\mathcal{L}_2(\mathcal{S}), \mathring{\mathtt{P}}_t), \ t=1,2,$ we can obtain the follow decompositions for $k(\bm{w}_1, \bm{w}_2)$
\begin{equation}\label{eq_mercerresult}
	k(\bm{w}_1,\bm{w}_2)=\sum_j \lambda^{a,t}_j \psi^{a,t}_j(\bm{w}_1) \psi^{a,t}_j(\bm{w}_2),
\end{equation}
for some positive eigenvalue sequences $\{\lambda_j^{a,t}\}$ and eigenfunction sequences that 
\begin{equation*}
	\int_{\mathcal{S}} \psi_j^{a,t}(\bm{w}) \psi_i^{a,t}(\bm{w}) \mathring{\mathtt{P}}_t(\mathrm{d} \bm{w})= \delta_{ij}.   
\end{equation*}
For notional simplicity, we denote these two RKHSs as $\mathring{\mathcal{H}}_1$ and $\mathring{\mathcal{H}}_2.$

\subsection{Proof  of results in Section \ref{sec_manifoldmodelandlandmarkoperator}}\label{sec_31proof}

\begin{proof}[\bf Proof of Proposition \ref{prop_pdfkernel}] 
{ For (\ref{eq_kernelone1111}), due to similarly, we focus on $k_1(\bm{x}_1, \bm{x}_2).$ Under the assumption that $d'>0,$ together with (\ref{eq_reducedmapping}), we find that for all $\bm{z}_{d'} \in \mathcal{S}_{1,d'}$ sampled according to $\widetilde{\mathsf{P}}_{1,d'}$, $\vartheta'(\bm{z}_{d'}) \in \mathcal{S}_{2,d'}$ will be distributed according to $\widetilde{\mathsf{P}}_{2,d'}$. Consequently, we can rewrite 
	\begin{align*}
		k_1(\bm{x}_1, \bm{x}_2)& = \int_{\mathcal{S}_{2,d'}} \int_{\mathcal{S}_{2,d'}^\perp} k\left(\bm{x}_1, \begin{pmatrix}
\bm{w}_{d'} \\ \bm{z}_{d'}^\perp		
\end{pmatrix} \right) k\left(\begin{pmatrix}
\bm{w}_{d'} \\ \bm{z}_{d'}^\perp		
\end{pmatrix} , \bm{x}_2\right) \widetilde{\mathsf{P}}^\perp_{2,d'}(\mathrm{d} \bm{z}_{d'}^\perp) \widetilde{\mathsf{P}}_{1,d'}(\mathrm{d} (\vartheta')^{-1}(\bm{w}_{d'})) \\
&= \int_{\mathcal{S}_{2,d'}} \int_{\mathcal{S}_{2,d'}^\perp} k\left(\bm{x}_1, \begin{pmatrix}
\bm{w}_{d'} \\ \bm{z}_{d'}^\perp		
\end{pmatrix} \right) k\left(\begin{pmatrix}
\bm{w}_{d'} \\ \bm{z}_{d'}^\perp		
\end{pmatrix} , \bm{x}_2\right) \widetilde{\mathsf{P}}^\perp_{2,d'}(\mathrm{d} \bm{z}_{d'}^\perp) \widetilde{\mathsf{P}}_{2,d'}(\mathrm{d} \bm{w}_{d'}).
	\end{align*}
This concludes the proof. 
}

	The boundedness of the kernel follow directly from the first part of the result and the fact that $k(\cdot,\cdot)$ is bounded. For positive definiteness,  due to similarity, we focus our discussion on the kernel $k_1(\bm{x}_1, \bm{x}_2).$ 
	Using the definition (\ref{eq_kernelone}) and the conventions in Section \ref{sec_auxililaryincrement}, we can rewrite $k_1(\bm{x}_1, \bm{x}_2)$ as 
	\begin{equation}\label{eq_formone}
		k_1(\bm{x}_1, \bm{x}_2)=\int_{\mathcal{S}} k(\bm{x}_1, \bm{z}) k(\bm{z}, \bm{x}_2) \mathring{\mathtt{P}}_2 (\mathrm{d} \bm{z}). 
	\end{equation}
	Together with (\ref{eq_mercerresult}), we can write {
	\begin{align}\label{eq_k1proof}
		k_1(\bm{x}_1, \bm{x}_2) & =\sum_{j} \sum_{i} \lambda_j^{a,2} \lambda_i^{a,2} \psi_j^{a,2}(\bm{x}_1) \psi_i^{a,2}(\bm{x}_2) \int_{\mathcal{S}} \psi_j^{a,2}(\bm{z}) \psi_i^{a,2}(\bm{z}) \mathring{\mathtt{P}}_2 (\mathrm{d} \bm{z}) \nonumber  \\
		&=\sum_{j} (\lambda_j^{a,2})^2  \psi_j^{a,2}(\bm{x}_1) \psi_j^{a,2}(\bm{x}_2).
	\end{align}}
	This completes our proof for $k_1(\bm{x}_1, \bm{x}_2)$ using the reverse of Mercer's theorem (see Exercise 2.23 of \cite{scholkopf2002learning}).
	
	For $k_2(\bm{x}_1, \bm{x}_2),$ similar to the discussion of (\ref{eq_k1proof}), we can show that 
	\begin{equation*}
		k_2(\bm{x}_1, \bm{x}_2)=\sum_{j} (\lambda_j^{a,1})^2 \psi_j^{a,1}(\bm{x}_1) \psi_j^{a,1} (\bm{x}_2).  
	\end{equation*}  
	This completes the proof for $k_2(\bm{x}_1, \bm{x}_2).$ 
\end{proof}

{
\begin{rem}
The eigenfunctions used to interpret the joint embeddings for the two datasets $\mathcal{X}$ and $\mathcal{Y}$ are derived from two integral operators, $\mathcal{K}_1$ and $\mathcal{K}_2$, respectively. These operators are defined using two newly introduced convolutional landmark kernels, $k_1$ and $k_2.$

First, the construction of $k_1$ in (\ref{eq_kernelone1111}), which is used to analyze the dataset $\mathcal{X}$, incorporates information from $\mathcal{Y}$. Specifically, \begin{equation*} k_1(\bm{x}_1, \bm{x}_2)=\int_{\mathcal{S}_2} k(\bm{x}_1, \bm{z}) k(\bm{z}, \bm{x}_2) \widetilde{\mathsf{P}}_2(\mathrm{d} \bm{z}), 
\end{equation*} 
where we recall that $\mathcal{S}_2$ encodes the geometric information of the manifold associated with $\mathcal{Y}$, and $\widetilde{\mathsf{P}}_2$ denotes the corresponding sampling distribution. In this way, the construction of $k_1$ leverages the manifold structure of $\mathcal{Y}$ as a landmark population. A similar construction applies to $k_2$ using $\mathcal{X}$.

Second, regarding the eigenfunctions associated with the integral operators, (\ref{eq_k1proof}) shows that the eigenfunctions of $\mathcal{K}_1$ coincide with those of an integral operator defined on the joint space $\mathcal{S}_1 \cup \mathcal{S}_2$, where the sampling distribution is extended using $\widetilde{\mathsf{P}}_2$. Similarly, the eigenfunctions of $\mathcal{K}_2$ coincide with those of an operator on the same joint space, but with the sampling distribution extended using $\widetilde{\mathsf{P}}_1$.
At a high level, this formulation enables $\mathcal{K}_1$ to learn from $\widetilde{\mathsf{P}}_2$ and $\mathcal{K}_2$ to learn from $\widetilde{\mathsf{P}}_1$, thereby facilitating the integration of information over the overlapping domain $\mathcal{S}_1 \cap \mathcal{S}_2$. 
\end{rem}
}

\subsection{Proof  of results in Section \ref{sec_spectralanalysis}}\label{sec_32proof}


\begin{proof}[\bf Proof of Propositions \ref{eigenvalue.prop}] According to (\ref{eq_formone}) and (\ref{eq_k1proof}), we can also write that 
	\begin{equation}\label{eq_k1form}
		k_1(\bm{x}_1, \bm{x}_2)=\sum_{j} \sum_{i} \lambda_j^{a,2} \lambda_i^{a,1} \psi_j^{a,2}(\bm{x}_1) \psi_i^{a,1}(\bm{x}_2) \int_{\mathcal{S}} \psi_j^{a,2}(\bm{z}) \psi_i^{a,1}(\bm{z}) \mathring{\mathtt{P}}_2 (\mathrm{d} \bm{z}). 
	\end{equation}
	Similarly, we can also write 
	\begin{equation*}
		k_2(\bm{y}_1, \bm{y}_2)=\sum_{j} \sum_{i} \lambda_j^{a,2} \lambda_i^{a,1} \psi_j^{a,2}(\bm{y}_1) \psi_i^{a,1}(\bm{y}_2) \int_{\mathcal{S}} \psi_j^{a,2}(\bm{z}) \psi_i^{a,1}(\bm{z}) \mathring{\mathtt{P}}_1 (\mathrm{d} \bm{z}). 
	\end{equation*}
	Let $\ell_2(\mathbb{R}):=\{  \bm{a}= (a_i)_{i \in \mathbb{N}} \in \mathbb{R}^{\infty}| a_i \in \mathbb{R} \ \text{and} \ \sum_{i=1}^{\infty} a_i^2<\infty \}.$ We denote two operators $\mathcal{T}_1$ and $\mathcal{T}_2$ on $\ell_2(\mathbb{R})$ as follows 
	\begin{equation*}
		\mathcal{T}_1 \bm{a}=\left( \lambda_j^{a,2}\sum_{i=1}^{\infty} a_i  \int_{\mathcal{S}} \psi_j^{a,2}(\bm{z}) \psi_i^{a,1}(\bm{z}) \mathring{\mathtt{P}}_2 (\mathrm{d} \bm{z}) \right)_{j \in \mathbb{N}},
	\end{equation*} 
	and
	\begin{equation*}
		\mathcal{T}_2 \bm{a}=\left( \lambda_j^{a,1}\sum_{i=1}^{\infty} a_i \int_{\mathcal{S}} \psi_j^{a,1}(\bm{z}) \psi_i^{a,2}(\bm{z}) \mathring{\mathtt{P}}_1 (\mathrm{d} \bm{z}) \right)_{j \in \mathbb{N}}. 
	\end{equation*}
	It is straightforward to see that $\mathcal{T}_1$ and $\mathcal{T}_2$ commutes on $\ell_2(\mathbb{R}).$ Therefore, according to Lemma \ref{lem_Operatorcommuteshavesameeigenvalues}, $\mathcal{T}_1 \mathcal{T}_2$ and $\mathcal{T}_2 \mathcal{T}_1$ have the same nonzero eigenvalues. To conclude our proof, we will show that the nonzero eigenvalues of $\mathcal{K}_1$ coincide with those of $\mathcal{T}_1 \mathcal{T}_2,$ and those of $\mathcal{K}_2$ coincide with those of $\mathcal{T}_2 \mathcal{T}_1.$ 
	
	Under the assumption that $\mathcal{H}_1$ and $\mathcal{H}_2$ are compactly embedded into some common space $\mathcal{N},$  we can use the conventions in Section \ref{sec_auxililaryincrement} to conclude that    these exist some common measurable space $\mathring{\mathcal{N}}$ so that $\mathring{\mathcal{H}}_1$ and $\mathring{\mathcal{H}}_2$ are compactly embedded into $\mathcal{L}_2(\mathring{\mathcal{N}}, \mathring{\mathsf{P}}_1)$ and $\mathcal{L}_2(\mathring{\mathcal{N}}, \mathring{\mathsf{P}}_2),$ respectively. Suppose $\phi(\cdot)$ is an eigenfunction of $ \mathcal{K}_1$ associated with some nonzero eigenvalue $\lambda$. That is, 
	\beq \label{eigen.fun}
	\mathcal{K}_1\phi(\bm{x})=\lambda\phi(\bm{x}), \ \bm{x} \in \mathcal{S}_1.  
	\eeq
	By Lemma \ref{lem_RKHSextension} and the conventions of Section \ref{sec_auxililaryincrement}, we can expand  $\phi(\bm{x})$ on the common space $\mathring{\mathcal{N}}$ via $\mathcal{L}_2(\mathring{\mathcal{N}}, \mathring{\mathsf{P}}_2)$ using the basis of $\mathring{\mathcal{H}}_2.$ Consequently, we can write that for some constants $b_i \equiv b_i(\phi), \ i \in \mathbb{N}$ and $(b_i) \in \ell_2(\mathbb{R})$ so that 
	\begin{equation*}
		\phi(\bm{x}) =  \sum_{i=1}^\infty b_i  \psi^{a,2}_i(\bm{x}).
	\end{equation*}
	Together with (\ref{eq_k1form}) and Definition \ref{defn_landmarkintegral} as well as the conventions in Section \ref{sec_auxililaryincrement}, we have that 
	\begin{align*}
		\mathcal{K}_1\phi(\bm{x}) &= \int_{\mathcal{S}} k_1(\bm{x},\bm{y}) \sum_{i=1}^\infty b_i \psi^{a,2}_i(\bm{y})\mathring{\mathtt{P}}_1(\mathrm{d} \bm{y})\\
		&=\int_{\mathcal{S}} \sum_{j} \sum_{i} \lambda_j^{a,2} \lambda_i^{a,1} \psi_j^{a,2}(\bm{x}) \psi_i^{a,1}(\bm{y}) \int_{\mathcal{S}} \psi_j^{a,2}(\bm{z}) \psi_i^{a,1}(\bm{z}) \mathring{\mathtt{P}}_2 (\mathrm{d} \bm{z}) \sum_{t=1}^\infty b_t \psi_t^{a,2}(\bm{y})\mathring{\mathtt{P}}_1(\mathrm{d} \bm{y})\\
		&=\sum_{j=1}^{\infty} \alpha_j \psi_j^{a,2}(\bm{x}),
	\end{align*}
	where $\alpha_j$ is denoted as 
	\begin{equation*}
		\alpha_j:=\lambda_j^{a,2}\sum_i  \int_{\mathcal{S}} \psi_j^{a,2}(\bm{z}) \psi_i^{a,1}(\bm{z}) \mathring{\mathtt{P}}_2 (\mathrm{d} \bm{z}) \lambda_i^{a,1} \sum_{t=1}^{\infty} b_t   \int_{\mathcal{S}} \psi_i^{a,1}(\bm{y}) \psi_t^{a,2}(\bm{y}) \mathring{\mathtt{P}}_1(\mathrm{d} \bm{y}). 
	\end{equation*}
	Since the discussion holds for all $\bm{x},$ together with (\ref{eigen.fun}), we conclude that 
	\begin{equation}\label{eq_identity}
		\alpha_j=\lambda b_j, \ j \in \mathbb{N}
	\end{equation}
	Denote $\bm{b}=(b_j)_{j \in \mathbb{N}}$ and
	$\bm{\alpha}=(\alpha_j)_{j \in \mathbb{N}}.$ Using the definitions of $\mathcal{T}_1$ and $\mathcal{T}_2,$ we can rewrite (\ref{eq_identity}) as follows
	\begin{equation*}
		\mathcal{T}_1 \mathcal{T}_2 \bm{b}=\lambda \bm{b}.  
	\end{equation*}
	This shows that $\lambda$ is also an eigenvalues of $\mathcal{T}_1 \mathcal{T}_2$ and conclude the proof that the nonzero eigenvalues of $\mathcal{K}_1$ coincide with those of $\mathcal{T}_1 \mathcal{T}_2.$ Similarly, we can show that the eigenvalues of $\mathcal{K}_2$ coincide with those of $\mathcal{T}_2 \mathcal{T}_1.$ This completes our proof. 
	
\end{proof}

\begin{proof}[\bf Proof of Theorem \ref{thm_cleanconvergence}]
	
	We start with the convergence of the eigenvalues (\ref{eq_defni}). Since $\bN_{01}$ and $\bN_{02}$ in (\ref{eq_twomatrices}) have the same nonzero eigenvalues, we focus our discussion on $\bN_{01}.$ For $\bN_{01}=(N_{ij}^{01}),$ we see that 
	%
	%
	\begin{equation}\label{eq_cite}
		N^{01}_{ij}=\frac{1}{n_1n_2}\sum_{k=1}^{n_2}k(\xb^0_i,\yb^0_k)k(\xb^0_j,\yb^0_k).
	\end{equation}
	
	Using the conventions in Section \ref{sec_31proof}, we now introduce the following auxiliary matrix that $\bW_{01}=(W_{ij}^{01}) \in \mathbb{R}^{n_1 \times n_1},$ where
	\begin{equation}\label{eq_intermediatematrix}
		W_{ij}^{01} = \frac{1}{n_1}\int_{\mathcal{S}}  k(\xb^0_i,\by)k(\by,\xb^0_j)\mathring{\mathtt{P}}_{2}(\mathrm{d}\by) \equiv \frac{1}{n_1}k_1(\xb_i^0, \xb_j^0).
	\end{equation}
	Denote the eigenvalues of $\bW_{01}$ as $\{\lambda_i^{01}\}.$ Then according to Proposition \ref{prop_pdfkernel} and Lemma \ref{eigenvalue.lem1}, we have that 
	\begin{equation}\label{eq_partone}
		\sup_i|\lambda_i^{01}-\gamma_i| \prec \frac{1}{\sqrt{n_1}}. 
	\end{equation}

	Moreover, for each $1\le i\ne j\le n_1$, conditional on $\xb_i^0$ and $\xb_j^0$, $N_{ij}^{01}$ is essentially
	a sum of independent bounded random variables $k(\xb^0_i,\yb^0_k)k(\xb^0_j,\yb^0_k), k=1,2,...,n_2$, satisfying
	\begin{equation*}
		\E \left[k(\xb^0_i,\yb^0_k)k(\yb^0_k,\xb^0_j)| \xb^0_i,\xb^0_j\right]=n_1 W_{ij}^{01}.
	\end{equation*}
	Therefore, for each pair of $i,j,$  we can apply Lemma \ref{lem_boundedconcentration} to obtain that 
	\begin{equation}\label{eq_discussionsimilarargument}
		\left| \frac{1}{n_2}\sum_{k=1}^{n_2}k(\xb^0_i,\yb^0_k)k(\yb^0_k,\xb^0_j)- n_1 W_{ij}^{01}\right|\prec n_2^{-1/2}.
	\end{equation}
	Equivalently, we have that 
	\begin{equation*}
		n_1 N_{ij}^{01}=n_1 W_{ij}^{01}+\OO_{\prec}\left(n_2^{-1/2}\right).
	\end{equation*}
	By Gershgorin circle theorem, we readily have that 
	\begin{equation}\label{part_II}
		\| \bN_{01}-\bW_{01} \| \prec \frac{1}{\sqrt{n}_2}. 
	\end{equation}
	Combining (\ref{eq_partone}) and (\ref{part_II}), we can conclude the proof of (\ref{v.conv0}). 
	
	Then we proceed to the proof of the eigenfunctions. Due to similarity, we focus on the proof of (\ref{eigenvector_leftspectral}). For the matrix $\bW_{01}$ defined according to (\ref{eq_intermediatematrix}), we denote its eigenvectors associated with the eigenvalues $\{\lambda_i^{01}\}$ as $\{\bu_i^{01}\}.$ Then for $\bm{x} \in \mathcal{S}_1,$ we denote that 
	\begin{equation*}
		\phi^{01}_{i}(\bm{x})=\frac{1}{\lambda^{01}_i\sqrt{n_1}} \sum_{j=1}^{n_1} {k}_1(\bm{x}, \xb^0_j)  u^{01}_{ij}
	\end{equation*} 
	where $\ub^{01}_i=(u^{01}_{i1}, \cdots, u^{01}_{in_1})^\top.$ {According to (\ref{eq_partone}) and Lemma \ref{eigenvalue.lem1}, we find that 
	\begin{equation*}
			\inf_{\bO\in O(|\sf I|)}\|\sqrt{\lambda_i^{01}}{\bm \phi}^{01}_{\sf I}-\sqrt{\gamma_i}\bO{\bm \phi}_{\sf I}\|_{K_1,\infty} \prec \frac{1}{ \mathsf{r}_i \sqrt{n}},
	\end{equation*}
	where ${\bm \phi}^{01}_{\sf I}$ and ${\bm \phi}_{\sf I}$ are vector-valued functions containing $\{\phi_i^{01}\}_{i\in {\sf I}}$ and $\{\phi_i\}_{i\in {\sf I}}$, respectively, $\|(f_1,...,f_m) \|_{K_1,\infty}=\max_{1\le j\le m}\|f_j\|_{K_1}$, and $\|\cdot\|_{K_1}$ is the operator norm on the RKHS $\mathcal{H}_1$ built on $\mathcal{L}_2(\mathcal{S}_1, \widetilde{\mathsf{P}}_1)$ (or equivalently, on $\mathcal{L}_2(\mathcal{S}, \mathring{\mathtt{P}}_1)$ using the convention of Section \ref{sec_32proof}) using the kernel $k_1.$  Moreover, using the property of the reproducing kernel (see equation (2.31) of \cite{scholkopf2002learning}), we have that 
	\begin{equation}\label{eq_partoneeigenvector}
		\inf_{\bO\in O({\sf I})}\left\| \sqrt{\lambda_i^{01}} {\bm \phi}_{\sf I}^{01}(\bm{x})-\sqrt{\gamma_i} \bO{\bm \phi}_{\sf I}(\bm{x})  \right\|_\infty \leq \sqrt{k_1(\bm{x}, \bm{x})}\inf_{\bO\in O(|\sf I|)}\|\sqrt{\lambda_i^{01}}{\bm \phi}^{01}_{\sf I}-\sqrt{\gamma_i}\bO{\bm \phi}_{\sf I}\|_{K_1,\infty}  \prec \frac{1}{\sfr_i \sqrt{n}_1},
	\end{equation} 
	where we used the boundedness of the kernel $k_1$ in Proposition \ref{prop_pdfkernel}.  
	
	Moreover, we denote
	\[
\sqrt{\lambda_i^{01}} 	{\bm \phi}_{\sf I}^{01}(\bm{x})=\bigg(\frac{1}{\sqrt{n_1\lambda^{01}_i}} \sum_{j=1}^{n_1} {k}_1(\bm{x}, \xb^0_j) u^{01}_{kj}\bigg)_{k\in \sf I}: = \frac{1}{\sqrt{n_1\lambda^{01}_i}}\bm{k}_1(\bm{x})^\top\bU^{01},
	\]
	where $\bU^{01}\in\R^{n_1\times |{\sf I}|}$ and $\bm{k}_1(\bm{x})=\big( {k}_1(\bm{x}, \xb^0_j)\big)_{1\le j\le n_1}$, and denote
		\[
	\sqrt{\lambda_i}	\widehat{\bm \phi}_{\sf I}^{0}(\bm{x}):=\bigg(\frac{1}{\sqrt{n_1\lambda_i}} \sum_{j=1}^{n_1} \widehat{k}^0_1(\bm{x}, \xb^0_j) u^{0}_{kj}\bigg)_{k\in \sf I}: = \frac{1}{\sqrt{n_1\lambda_i}}\widehat{\bm{k}}^0_1(\bm{x})^\top\bU^{0},
	\]
	where $\bU^{0}=(u_{ij}^0)\in\R^{n_1\times |{\sf I}|}$ and $\widehat{\bm{k}}^0_1(\bm{x})=\big( \widehat{k}^0_1(\bm{x}, \xb^0_j)\big)_{1\le j\le n_1}$. Then, if we define $\widehat{\bm\phi}^0_{\sf I}=(\widehat{\phi}^0_{i})_{i\in \sf I}$, we have
	\begin{align}\label{eq_generalcontrol}
		&\quad \inf_{\bO\in O(|{\sf I}|)}\left\|\sqrt{\lambda_i^{01}} \bm{\phi}_{\sf I}^{01}(\bm{x})-\sqrt{\lambda_i}\bO\widehat{\bm\phi}^0_{\sf I}(\bm{x}) \right\|_\infty \nonumber \\
		&= \frac{1}{\sqrt{n_1}}\inf_{\bO\in O(|{\sf I}|)}\bigg\|\frac{1}{\sqrt{\lambda^{01}_i}}\bm{k}_1(\bm{x})^\top\bU^{01}-\frac{1}{\sqrt{\lambda_i}}\widehat{\bm{k}}^0_1(\bm{x})^\top\bU^{0}\bO\bigg\|_\infty \nonumber \\
		&\le  \frac{1}{\sqrt{n_1}}\bigg\|\frac{1}{\sqrt{\lambda^{01}_i}}\widehat{\bm{k}}^0_1(\bm{x})^\top\bU^{0}-\frac{1}{\sqrt{\lambda_i}}\widehat{\bm{k}}^0_1(\bm{x})^\top\bU^{0}\bigg\|_\infty +\frac{1}{\sqrt{n_1}}\bigg\|\frac{1}{\sqrt{\lambda^{01}_i}}{\bm{k}}_1(\bm{x})^\top\bU^{0}-\frac{1}{\sqrt{\lambda^{01}_i}}\widehat{\bm{k}}^0_1(\bm{x})^\top\bU^{0}\bigg\|_\infty \nonumber \\
		&\quad +\frac{1}{\sqrt{n_1}} \inf_{\bO\in O(|{\sf I}|)}\bigg\|\frac{1}{\sqrt{\lambda^{01}_i}}{\bm{k}}_1(\bm{x})^\top\bU^{01}-\frac{1}{\sqrt{\lambda^{01}_i}}{\bm{k}}_1(\bm{x})^\top\bU^{0}\bO\bigg\|_\infty \nonumber \\
		&:= E_1+E_2+E_3.
		\end{align}
		For $E_1,$ by Cauchy–Schwarz inequality, we have that 
		\beq\label{eq_E1control}
		E_1\le \frac{1}{\sqrt{n_1}} \frac{|\lambda_i-\lambda^{01}_i|}{\sqrt{\lambda_i \lambda^{01}_i}} \sqrt{\sum_{j=1}^{n_1} (k_1(\bm{x}, \xb_j^0))^2}  \|\bU^0\|_{2,\infty}.
		\eeq
		Together with the assumption that $1 \leq i \leq \mathsf{K}$ (c.f. (\ref{eq_defni})), (\ref{v.conv0}) and (\ref{part_II}), we have that 
		\begin{equation*}
			E_1 \prec \frac{1}{\sqrt{n_2}},
		\end{equation*}
		where we used the boundnedness of the kernel $k_1$ in Proposition \ref{prop_pdfkernel}. For $E_2,$   we have that
			\begin{align}\label{eq_E2control}
				E_2&= \frac{1}{\sqrt{n_1\lambda^{01}_i}}\bigg\|[{\bm{k}}_1(\bm{x})-\widehat{\bm{k}}^0_1(\bm{x})]^\top\bU^{0}\bigg\|_\infty \nonumber\\
			&\le \frac{1}{\sqrt{n_1\lambda^{01}_i}}\max_{k\in\sf I}\bigg| \sum_{j=1}^{n_1} [{k}_1(\bm{x}, \xb^0_j) -\widehat{k}^0_1(\bm{x}, \xb^0_j)]  u^0_{kj}\bigg|\\
			& \leq  \frac{1}{\sqrt{\lambda_i^{01}}} \sqrt{\frac{1}{n_1} \sum_{j=1}^{n_1}\left[ k_1(\bm{x},\xb_j^0)-\widehat{k}_1^0(\bm{x},\xb_j^0) \right]^2} \nonumber \\
			& \leq  \frac{1}{\sqrt{\lambda^{01}_i} } \max_{1 \leq j \leq n_1}\bigg| \frac{1}{n_2}\sum_{k=1}^{n_2}k(\bm{x},\yb^0_k)k(\yb^0_k,\xb^0_j)- \int_{\mathcal{S}}  k(\bm{x},\bm{z})k(\bm{z},\xb^0_j)\mathring{\mathtt{P}}_{2}(\mathrm{d}\bm{z})\bigg| \nonumber \\
			&\prec n_2^{-1/2}, \nonumber
		\end{align}
		where in the thrid step we used Cauchy-Schwarz inequality, in the fourth step we used the fact that both $k_1$ and $\widehat{k}^0_1$ are bounded and the definition in (\ref{eq_k01k02definition}), and in the last step we used a discussion similar to (\ref{eq_discussionsimilarargument}) with Lemma \ref{lem_boundedconcentration}. Similarly, for $E_3$, by Cauchy-Schwarz inequality, we have 
		\begin{align}\label{eq_E3decomposition}
			E_3&=\frac{1}{\sqrt{n_1}} \inf_{\bO\in O(|{\sf I}|)}\bigg\|\frac{1}{\sqrt{\lambda^{01}_i}}{\bm{k}}_1(\bm{x})^\top[\bU^{01}-\bU^{0}\bO]\bigg\|_\infty\le \frac{1}{\sqrt{n_1\lambda^{01}_i}}\inf_{\bO\in O(|{\sf I}|)}\bigg| \sum_{j=1}^{n_1} {k}_1(\bm{x}, \xb^0_j) (u^{01}_{ij}-u^0_{ij})\bigg|\nonumber \\
			& \prec\inf_{\bO\in O(|{\sf I}|)} \|\bU^{01}-\bU^0\bO\|_{2,\infty}\le \inf_{\bO\in O(|{\sf I}|)} \|\bU^{01}-\bU^0\bO\|.
		\end{align}
		Consequently, it suffices to control  $ \inf_{\bO\in O(|{\sf I}|)} \|\bU^{01}-\bU^0\bO\|$. According to the assumption of (\ref{eq_ricondition}), (\ref{v.conv0}) and (\ref{part_II}), we see that with high probability, for some constant $c>0$
		\beqs
		\min_{i\in {\sf I}, j\in {\sf I}^c}|\lambda_i-\lambda_{j}|\ge c\sfr_i, \qquad \|\bN_{01}-\bW_{01}\|\le c\sfr_i/2.
		\eeqs
		Then we can use Lemma \ref{eigenvector.lem} and (\ref{part_II}) to show that  
		\begin{align*}
		\inf_{\bO\in O(|{\sf I}|)} \|\bU^{01}-\bU^0\bO\|\le\sqrt{2}	\|\bU^{01}(\bU^{01})^\top-\bU^0(\bU^0)^\top\|\prec \frac{1}{\sfr_i\sqrt{n_2}},
		\end{align*}
		where the first inequality follows from Lemma 1 of \cite{cai2018rate}.
		Consequently, we have that $E_3 \prec \sfr_i^{-1} n_2^{-1/2}.$ Inserting all the above controls into (\ref{eq_generalcontrol}), we see that 
		\begin{equation*}
			\inf_{\bO\in O(|{\sf I}|)}\left\|\sqrt{\lambda_i^{01}} \bm{\phi}_{\sf I}^{01}(\bm{x})-\sqrt{\lambda_i}\bO\widehat{\bm\phi}^0_{\sf I}(\bm{x}) \right\|_\infty
			 \prec \frac{1}{\sfr_i \sqrt{n_2}}+\frac{1}{\sqrt{n_2}}. 
		\end{equation*}
		Together with (\ref{eq_partoneeigenvector}), we can complete the proof of (\ref{eigenvector_leftspectral}). 
}
\end{proof}

\subsection{Proof of results in Section \ref{sec_robustness}}\label{sec_proof33}

\begin{proof}[\bf Proof of Theorem \ref{thm_noiseconvergence}]
	For $\bN_1=(N^1_{ij})_{1\le i,j\le n}$ defined in (\ref{eq_twomatricesnoisy}), we have that 
	\[
	N^1_{ij}=\frac{1}{n_1n_2}\sum_{k=1}^{n_2}f\bigg( \frac{\|\xb_i-\yb_k\|_2}{h^{1/2}_n}\bigg)f\bigg(\frac{\|\yb_k-\xb_j\|_2}{h^{1/2}_n}\bigg),
	\]
	where we used the short-hand notation that $f(x)=\exp(-x^2).$ During the proof, we will also use the following auxiliary quantity
	\[
	\widetilde{N}^1_{ij}=\frac{1}{n_1n_2}\sum_{k=1}^{n_2}f\bigg( \frac{\|\xb_i-\yb_k\|_2}{\sqrt{h_n^0}}\bigg)f\bigg(\frac{\|\yb_k-\xb_j\|_2}{\sqrt{h_n^0}}\bigg).
	\]
	where $h_n^0$ is the clean bandwidth as defined in (\ref{eq_clearnsignalkernel}). Moreover, for $\bN_{01}=(N_{ij}^{01})$ defined in (\ref{eq_twomatrices}), we have that 
	\[
	N^{01}_{ij}=\frac{1}{n_1n_2}\sum_{k=1}^{n_2}f\bigg( \frac{\|\xb^0_i-\yb^0_k\|_2}{\sqrt{h_n^0}}\bigg)f\bigg(\frac{\|\yb^0_k-\xb^0_j\|_2}{\sqrt{h_n^0}}\bigg).
	\]
	
	We start with the proof of the eigenvalue convergence (\ref{eq_eigs_v}) and control $\| \bN_{01}-\bN_1 \|$. We will control the entry-wise difference by studying 
	\begin{align*}
		|N^1_{ij}-N^{01}_{ij}| &\le 	|N^1_{ij}-\widetilde{N}^1_{ij}|+	|\widetilde{N}^1_{ij}-N^{01}_{ij}|.
	\end{align*}
	First, under the model assumption (\ref{model}), we have that 
	\begin{align} \label{approx.bnd}
		&|\widetilde{N}^{1}_{ij}-N^{01}_{ij}|\nonumber \\
		&= \bigg|\frac{1}{n_1n_2}\sum_{k=1}^{n_2}\exp\bigg( -\frac{\|\xb_i-\yb_k\|^2_2+\|\yb_k-\xb_j\|_2^2}{h_n}\bigg)-\frac{1}{n_1n_2}\sum_{k=1}^{n_2}\exp\bigg(- \frac{\|\xb^0_i-\yb^0_k\|_2^2+\|\yb^0_k-\xb^0_j\|_2^2}{h_n}\bigg)\bigg| \nonumber \\
		& = \bigg|\frac{1}{n_1n_2}\sum_{k=1}^{n_2}\exp\bigg( -\frac{\|\xb^0_i-\yb^0_k\|^2_2+\|\yb^0_k-\xb^0_j\|_2^2}{h_n}\bigg) \nu(k,i,j) \bigg|,  
	\end{align}
	where $\nu(k,i,j)$ is defined as
	\begin{equation*}
		\nu(k,i,j):= \exp\bigg(-\frac{\|{\bxi}_i-\bzeta_k\|_2^2+\|\bzeta_k-{\bxi}_j\|_2^2-2({\bxi}_i-\bzeta_k)^\top(\xb_i^0-\yb_k^0)-2({\bzeta}_k-\bxi_j)^\top(\yb_k^0-\xb_j^0)}{h_n}\bigg)-1.
	\end{equation*}
	By Lemma \ref{lem_concentrationinequality}, we have
	\beqs
	\max_{i,k}\|\bxi_i-\bzeta_k\|_2^2\prec \sigma^2p. 
	\eeqs
	Moreover, by Lemma \ref{sg.bnd.lem} and the setup in (\ref{eq_covstructure}), we have that 
	\beq\label{eq_boundhere}
	\max_{i\ne j}\|\xb_i^0-\yb_k^0\|_2^2\prec \sum_{i=1}^r\theta_i,\ \max_{i,k}|({\bxi}_i-\bzeta_k)^\top(\xb_i^0-\yb_k^0)|\prec \sigma(\sum_{i=1}^r\theta_i)^{1/2}.
	\eeq
	According to Proposition \ref{lem_bandwidthconcentration}, we see that
	\begin{equation*}
		\nu(k,i,j) \prec \frac{\sigma^2p+\sigma (\sum_{i=1}^r\theta_i)^{1/2}}{ \sum_{i=1}^r \theta_i}. 
	\end{equation*}
	Inserting the above bounds into (\ref{approx.bnd}), we readily see that
	\beq\label{eq_otherbounds}
	|\widetilde{N}^1_{ij}-N^{01}_{ij}| \prec \frac{\sigma^2p+\sigma (\sum_{i=1}^r\theta_i)^{1/2}}{n_1 \sum_{i=1}^r \theta_i}.
	\eeq
	Second, similar to the above discussions, we have that 
	\begin{align}
		|N^{1}_{ij}-\widetilde{N}^1_{ij}|&=\left|\frac{1}{n_1n_2}\sum_{k=1}^{n_2}\exp\bigg( -\frac{\|\xb_i-\yb_k\|^2_2+\|\yb_k-\xb_j\|_2^2}{h_n}\bigg)-\frac{1}{n_1n_2}\sum_{k=1}^{n_2}\exp\bigg( -\frac{\|\xb_i-\yb_k\|^2_2+\|\yb_k-\xb_j\|_2^2}{h_n^0}\bigg)\right|\nonumber \\
		&=\frac{1}{n_1 n_2}\left|\sum_{k=1}^{n_2}\exp\bigg( -\frac{\|\xb_i-\yb_k\|^2_2+\|\yb_k-\xb_j\|_2^2}{h_n} \upsilon(k,i,j) \right|\label{similarboundusingusingusing},
	\end{align}
	where $\upsilon(k,i,j)$ is defined as 
	\begin{equation*}
		\upsilon(k,i,j):=1-\exp\bigg( \frac{\|\xb_i-\yb_k\|^2_2+\|\yb_k-\xb_j\|_2^2}{h_n}-\frac{\|\xb_i-\yb_k\|^2_2+\|\yb_k-\xb_j\|_2^2}{h_n^0}\bigg). 
	\end{equation*}
	Using (\ref{eq_boundhere}) and Proposition \ref{lem_bandwidthconcentration}, we find that 
	\begin{equation*}
		\upsilon(k,i,j) \prec \eta. 
	\end{equation*}
	This implies that
	\beqs
	|N_{ij}^1-\widetilde{N}_{ij}^1|\prec \frac{\eta}{n_1}.
	\eeqs
	Together with (\ref{eq_otherbounds}), we have that 
	\begin{equation*}
		|N_{ij}^1-N_{ij}^{01}| \prec \frac{\eta}{n_1}. 
	\end{equation*}
	By Gershgorin circle theorem, we can conclude that
	\begin{equation}\label{eq_matrixmultiplicationerrorterm}
		\| \bN_{1}-\bN_{01} \| \prec \eta. 
	\end{equation}
	This completes the proof of (\ref{eq_eigs_v}).

	{
	For the eigenvectors convergence (\ref{eq_projectionbound}), due to similarity, we only prove the results for $\bU_{\sf I}$. Recall that for any positive definite matrix $\bA$ admitting the spectral decomposition $\bA = \sum_{i=1}^n \lambda_i(\bA)\bbeta_i\bbeta_i^\top$, it holds that
	\beqs
	\bbeta_i\bbeta_i^\top = \frac{1}{2\pi \ri} \oint_\Gamma (z{\bf I}-\bA)^{-1} \mathrm{d} z,  
	\eeqs
	where $\Gamma \equiv \Gamma_i \subset \mathbb{C}$ is some simply connected contour only containing $\lambda_i(\bA)$ but no other eigenvalues of $\bA$. The above integral representation serves as the starting point of our analysis.

	Now we denote the contour $\Gamma_i:=\mathsf{B}(\gamma_i, \frac{\mathsf{\sfr_i}}{C}),$ for some large constant $C>0,$ where $\mathsf{B}(\gamma_i, \frac{\sfr_i}{C})$ is the disk centered at $\gamma_i$ with radius $\frac{\sfr_i}{C}$ with $\sfr_i$ is defined in (\ref{defn_sfr}).  Combining (\ref{eq_eigs_v}), (\ref{v.conv0}) and the assumption of (\ref{eq_assumption}), we find that for $1 \leq i \leq \mathsf{K}$ 
	\beqs
	|\mu_i-\gamma_i|\le |\mu_i-\lambda_i|+|\lambda_i-\gamma_i|\prec \eta+n_1^{-1/2}+n_2^{-1/2}=\mathsf{o}(\sfr_i),
	\eeqs
	and for $1 \leq j \neq i \leq \mathsf{K},$ we have that  for some constant $C>0,$ with high-probability
	\begin{equation*}
		|\mu_j-\gamma_i| \geq \left| |\mu_j-\gamma_j|-|\gamma_j-\gamma_i| \right| \geq C \sfr_i. 
	\end{equation*} 
	This shows that with high probability,  $\mu_i$ is the only eigenvalue in $\{\mu_i\}$ that is inside $\Gamma_i.$ Consequently, it yields that  
	\beqs
	\bU_{\sf I}\bU_{\sf I}^\top = \frac{1}{2\pi \ri}\oint_{\Gamma_i} (z{\bf I}-\bN_1)^{-1} \mathrm{d} z.
	\eeqs
	This further implies that, with high probability
		\begin{align} \label{decomp.N1}
		\|\bU_{\sf I}\bU_{\sf I}-\bU_{\sf I}^0(\bU_{\sf I}^0)^\top\|&=  \bigg\|\frac{1}{2\pi \ri}\oint_{\Gamma_i} (z{\bf I}-\bN_1)^{-1} \mathrm{d} z-\frac{1}{2\pi \ri}\oint_{\Gamma_i} (z{\bf I}-\bN_{01})^{-1} \mathrm{d} z\bigg\| \nonumber \\
		&= \bigg\|\frac{1}{2\pi \ri}\oint_{\Gamma_i} [(z{\bf I}-\bN_1)^{-1}- (z{\bf I}-\bN_{01})^{-1}] \mathrm{d} z\bigg\|.
	\end{align}
	By the triangle inequality of complex integral over the contour, we have that 
	\begin{align}\label{eq_L2decomposigion}
		\|\bU_{\sf I}\bU_{\sf I}-\bU_{\sf I}^0(\bU_{\sf I}^0)^\top\|&\leq \frac{1}{2\pi \ri}\oint_{\Gamma_i} \|[(z{\bf I}-\bN_1)^{-1}- (z{\bf I}-\bN_{01})^{-1}]\|\cdot| \mathrm{d} z|\nonumber \\
		& \leq \frac{2 \pi \sfr_i}{2 \pi C} \sup_{z \in \Gamma_i} \left\| (z{\bf I}-\bN_{1})^{-1}-(z{\bf I}-\bN_{01})^{-1} \right \|.  
	\end{align}
	To control the right-hand side of (\ref{eq_L2decomposigion}), according to the resolvent identity, we have that   
	\begin{align} \label{res_id}
		(z{\bf I}-\bN_{1})^{-1}-(z{\bf I}-\bN_{01})^{-1}=(z{\bf I}-\bN_{1})^{-1}[\bN_{1}-\bN_{01}](z{\bf I}-\bN_{01})^{-1}.
	\end{align}
	Recall that $\{\lambda_i\}$  are the eigenvalues of $\bN_{01}.$ According to (\ref{defn_sfr}), (\ref{v.conv0}), the assumption of (\ref{eq_assumption}) as well as the definition of $\Gamma_i$ imply that
	\[
	\inf_{z\in\Gamma_i}\min_{k\in {\sf I}, j\in {\sf I}^c}\{|\lambda_k-z|, |\lambda_{j}-z|\}\asymp \sfr_i,
	\]
	which yields 
	\begin{equation*}
		\sup_{z \in \Gamma_i} \left\| (z{\bf I}-\bN_{01})^{-1} \right\| \prec \frac{1}{\sfr_i}. 
	\end{equation*}
	Similarly, by (\ref{eq_eigs_v}), we also have
	\begin{equation*}
		\sup_{z \in \Gamma_i}\left\| (z{\bf I}-\bN_{1})^{-1} \right\| \prec \frac{1}{\sfr_i}. 
	\end{equation*}
	Combining the above controls with (\ref{eq_matrixmultiplicationerrorterm}), (\ref{eq_L2decomposigion}) and  (\ref{res_id}), we conclude that
	\begin{equation*}
	\|\bU_{\sf I}\bU_{\sf I}-\bU_{\sf I}^0(\bU_{\sf I}^0)^\top\| \prec \frac{\eta}{\sfr_i}. 
	\end{equation*}
	The final result then follows from the simple relation
	\[
	\inf_{\bO\in O(|\sf I|)}\|\bU_{\sf I}-\bU^0_{\sf I}\bO\|\le	\sqrt{2}\|\bU_{\sf I}\bU_{\sf I}-\bU_{\sf I}^0(\bU_{\sf I}^0)^\top\|.
	\]}
\end{proof}

\begin{proof}[\bf Proof of Corollary \ref{cor_finalconvergence}]
	For the eigenvalues, (\ref{v.conv}) follows directly from (\ref{eq_eigs_v}) and (\ref{v.conv0}). { For the eigenfunctions, by symmetry, we only focus on the proof of (\ref{eq_lefteigenfunction}). Recall (\ref{eq_empericialeigenfunction}).  By triangle inequality, we have
	\begin{align}\label{eq_firstbegining}
		&\inf_{\bO\in O({\sf I})}\left\| \sqrt{\gamma_i} {\bm \phi}_{\sf I}(\bm{x}) -\sqrt{\mu_i}\bO \widehat{\bm \phi}_{\sf I}(\bm{x}) \right\|_\infty\nonumber \\
		&\le 	\inf_{\bO\in O({\sf I})}\left\| \sqrt{\gamma_i} {\bm \phi}_{\sf I}(\bm{x}) -\sqrt{\lambda_i}\bO \widehat{\bm \phi}^0_{\sf I}(\bm{x}) \right\|_\infty+	\inf_{\bO\in O({\sf I})}\left\| \sqrt{\lambda_i} \widehat{\bm \phi}^0_{\sf I}(\bm{x}) -\sqrt{\mu_i}\bO \widehat{\bm \phi}_{\sf I}(\bm{x}) \right\|_\infty.
	\end{align}
	Note that the first term on the RHS of the above equation can be controlled by (\ref{eigenvector_leftspectral}). It suffices to control the second term using a discussion similar to (\ref{eq_generalcontrol}). By definition, we have that 
	\begin{align}\label{eq_controlbound}
			&\inf_{\bO\in O({\sf I})}\left\| \sqrt{\lambda_i} \widehat{\bm \phi}^0_{\sf I}(\bm{x}) -\sqrt{\mu_i}\bO \widehat{\bm \phi}_{\sf I}(\bm{x}) \right\|_\infty=\frac{1}{\sqrt{n_1}}\inf_{\bO\in O(|{\sf I}|)}\bigg\|\frac{1}{\sqrt{\lambda_i}}\widehat{\bm{k}}^0_1(\bm{x})^\top\bU^{0}-\frac{1}{\sqrt{\mu_i}}\widehat{\bm{k}}_1(\bm{x})^\top\bU\bO\bigg\|_\infty \nonumber \\
			&\le \frac{1}{\sqrt{n_1}}\bigg\|\frac{1}{\sqrt{\lambda_i}}\widehat{\bm{k}}^0_1(\bm{x})^\top\bU^{0}-\frac{1}{\sqrt{\mu_i}}\widehat{\bm{k}}^0_1(\bm{x})^\top\bU^0\bigg\|_\infty +\frac{1}{\sqrt{n_1}}\bigg\|\frac{1}{\sqrt{\mu_i}}\widehat{\bm{k}}^0_1(\bm{x})^\top\bU^{0}-\frac{1}{\sqrt{\mu_i}}\widehat{\bm{k}}_1(\bm{x})^\top\bU^0\bigg\|_\infty\nonumber \\
			&\quad +\frac{1}{\sqrt{n_1}}\inf_{\bO\in O(|{\sf I}|)}\bigg\|\frac{1}{\sqrt{\mu_i}}\widehat{\bm{k}}_1(\bm{x})^\top\bU^{0}-\frac{1}{\sqrt{\mu_i}}\widehat{\bm{k}}_1(\bm{x})^\top\bU\bO\bigg\|_\infty\nonumber \\
		&:= F_1+F_2+F_3.
	\end{align}
	For $F_1,$ by a similar argument that leads to (\ref{eq_E1control}), we have that 
	\beqs
	F_1 \prec |\lambda_i-\mu_i| \prec \eta,
	\eeqs
	where in the second step we used (\ref{eq_eigs_v}). For $F_2,$ by a similar argument that leads to (\ref{eq_E2control}), we have that 
	\begin{align*}
		F_2 \prec  \max_{1\le j\le n_1,1\le k\le n_2} w(k,j),
	\end{align*}
	where $w(k,j)$ is defined as
	\begin{align*}
		w(k,j)&:=\left|\exp\bigg(-\frac{\|\bm{x}-\yb^0_k\|^2_2}{h_n^0}\bigg)\exp\bigg(-\frac{\|\yb^0_k-\xb^0_j\|_2^2}{h_n^{0}}\bigg)- \exp\bigg(-\frac{\|\bm{x}-\yb_k\|^2_2}{h_n}\bigg)\exp\bigg(-\frac{\|\yb_k-\xb_j\|_2^2}{h_n}\bigg) \right| \\
		& \leq \left|\exp\bigg(-\frac{\|\bm{x}-\yb^0_k\|^2_2}{h_n^0}\bigg)\exp\bigg(-\frac{\|\yb^0_k-\xb^0_j\|_2^2}{h_n^{0}}\bigg)- \exp\bigg(-\frac{\|\bm{x}-\yb_k\|^2_2}{h^0_n}\bigg)\exp\bigg(-\frac{\|\yb_k-\xb_j\|_2^2}{h_n^0}\bigg) \right| \\
		&+ \left|\exp\bigg(-\frac{\|\bm{x}-\yb_k\|^2_2}{h^0_n}\bigg)\exp\bigg(-\frac{\|\yb_k-\xb_j\|_2^2}{h_n^{0}}\bigg)- \exp\bigg(-\frac{\|\bm{x}-\yb_k\|^2_2}{h_n}\bigg)\exp\bigg(-\frac{\|\yb_k-\xb_j\|_2^2}{h_n}\bigg) \right| \\
		& :=F_{21}+F_{22}. 
	\end{align*}
	For $F_{21},$ by a discussion similar to (\ref{approx.bnd}),  we obtain
	\begin{equation*}
		F_{21} \prec \eta. 
	\end{equation*}
	For $F_{22},$ by a discussion similar to (\ref{similarboundusingusingusing}), we  obtain
	\begin{equation*}
		F_{22} \prec \eta.
	\end{equation*}
	Consequently, it follows that
	\begin{equation*}
		F_2 \prec \eta. 
	\end{equation*}
	For $F_3,$ by a discussion similar to (\ref{eq_E3decomposition}), we have that
	\begin{align*}
		F_3 \prec \inf_{\bO\in O(|\sf I|)}\|\bU_{\sf I}-\bU^0_{\sf I}\bO\| \prec \sqrt{\frac{\eta}{\sfr_i}},
	\end{align*}
	where in the last step we used (\ref{eq_projectionbound}). 
	
	Inserting the above results back to (\ref{eq_controlbound}), in light of (\ref{eq_firstbegining}) and (\ref{eigenvector_leftspectral}) in the main text, we  conclude that (\ref{eq_lefteigenfunction}) holds. }
	
\end{proof}

{
Denote $\mathbf{W}_1 \in \mathbb{R}^{n_1 \times p}$ as the purely noise matrix containing $\{\bm{\xi}_i\}$ and $\mathbf{W}_2 \in \mathbb{R}^{n_2 \times p}$ as the purely noise matrix containing $\{\bm{\zeta}_j\}.$ Moreover, denote
\begin{equation}\label{eq_upsilonbound}
\upsilon=\frac{\sum_{i=1}^r \theta_i}{p \sigma^2}+p^{-3/2}\sqrt{n_1 n_2}. 
\end{equation} }

\begin{proof}[\bf Proof of Theorem \ref{thm_noisededuction}] { Without loss of generality, we focus on the null case when both datasets contain pure noise that $\mathbf{x}_i=\bm{\xi}_i$ and $\mathbf{y}_j=\bm{\zeta}_j.$ For the general case, we will discuss it in the end of the proof.

According to the definition of $\mathbf{K}$, we have that for $1 \leq i \leq n_1, 1 \leq  j \leq n_2,$
\begin{equation}\label{eq_kij}
\mathbf{K}(i,j)= \exp \left(-\frac{\| \bm{\xi}_i \|_2^2+\|\bm{\zeta}_j\|_2^2}{h_n}-2 \frac{-\bm{\xi}_i^\top \bm{\zeta}_j}{h_n} \right).  
\end{equation} 
Denote 
\begin{equation*}
v_{ij}:=\frac{\| \bm{\xi}_i \|_2^2+\|\bm{\zeta}_j\|_2^2}{h_n}-\frac{p\sigma^2}{h_n}, \ u_{ij}:=\frac{\bm{\xi}_i^\top \bm{\zeta}_j}{h_n}.
\end{equation*}
According to Lemma \ref{lem_concentrationinequality} and part (2) of Proposition \ref{lem_bandwidthconcentration}, we have that
\begin{equation*}
v_{ij} \prec \frac{1}{\sqrt{p}}, \ \ u_{ij} \prec \frac{\sigma_1 \sigma_2}{\sigma^2 \sqrt{p}}. 
\end{equation*}
Now we expand $\mathbf{K}(i,j)$ around $p \sigma^2/h_n$ till the order of three to get 
\begin{align}\label{eq_keyexpansion}
\mathbf{K}(i,j)=\exp(-p \sigma^2 /h_n)&+2\exp(-p \sigma^2/h_n)u_{ij}-\exp(-p \sigma^2/h_n)v_{ij}+4\exp(-p \sigma^2/h_n) u_{ij}^2 \nonumber \\
&+\exp(-p \sigma^2/h_n)v^2_{ij}-2\exp(-p \sigma^2/h_n) v_{ij}u_{ij} +\mathrm{O}_{\prec}\left( p^{-3/2} \right).
\end{align}
Denote $\mathbf{V}$ as the $n_1 \times n_2$ matrix with entries $v_{ij}$ and $\mathbf{U}$ as the $n_1 \times n_2$ matrix with entries $u_{ij}.$ Moreover, denote $\mathbf{1}_{n} \in \mathbb{R}^n$ be a vector with entries all being one. Consequently, we can write $\mathbf{V}$ as follows 
\begin{equation}\label{eq_vdecomposition}
\mathbf{V}=\mathbf{a} \mathbf{1}_{n_2}^\top+\mathbf{1}_{n_1}\mathbf{b}^\top-\frac{p \sigma^2}{h_n}\mathbf{1}_{n_1} \mathbf{1}_{n_2}^\top,
\end{equation}
where $\mathbf{a}^top:=(\|\bm{\xi}_1\|_2^2,\ldots,\|\bm{\xi}_{n_1}\|_2^2) \in \mathbb{R}^{n_1}$ and $\mathbf{b}^\top:=(\|\bm{\zeta}_1\|_2^2,\ldots,\|\bm{\zeta}_{n_2}\|_2^2) \in \mathbb{R}^{n_2}.$ Furthermore, denote $\mathbf{W}_1$ as the $n_1 \times p$ matrix containing all $\mathbf{x}_i$ and $\mathbf{W}_2$ as the $n_2 \times p$ matrix containing all $\mathbf{y}_j.$ It is not hard to see that
\begin{equation}\label{eq_defnU}
\mathbf{U}=\mathbf{W}_1 \mathbf{W}_2^\top.
\end{equation}
For notional simplicity, we denote 
\begin{equation*}
\omega=\exp(-p \sigma^2 /h_n).
\end{equation*} 
Based on (\ref{eq_keyexpansion}), we can write 
\begin{equation*}
\mathbf{K}=\omega\mathbf{1}_{n_1} \mathbf{1}_{n_2}^\top+2\omega\mathbf{U}-\omega\mathbf{V}+4 \omega \mathbf{U} \circ \mathbf{U}+\omega \mathbf{V} \circ \mathbf{V}-2 \omega \mathbf{U} \circ \mathbf{V}+\mathrm{O}_{\prec}(p^{-3/2} \sqrt{n_1 n_2}), 
\end{equation*} 
where $\circ$ is the Hadamard product and  the error is in terms of Frobenius norm. By Lemma \ref{lem_lowrank} and (\ref{eq_vdecomposition}), we find that there exists some low rank matrix 
\begin{equation}\label{eq_L1decomposition}
\mathbf{L}_1=\omega\mathbf{1}_{n_1} \mathbf{1}_{n_2}^\top-\omega\mathbf{V}+\omega \mathbf{V} \circ \mathbf{V}+4 \omega \mathbf{E}_1-2 \omega\mathbf{E}_2, 
\end{equation}
so that we have 
\begin{equation}\label{eq_kexpanded}
\mathbf{K}=\mathbf{L}_1+2\omega\mathbf{U}+\mathrm{O}_{\prec}(p^{-3/2} \sqrt{n_1 n_2}).
\end{equation}
Consequently, we have that 
\begin{equation}\label{eq_aroundaround}
(\mathbf{K}-\mathbf{L}_1)(\mathbf{K}-\mathbf{L}_1)^\top=4\omega^2 \mathbf{U} \mathbf{U}^\top(1+\mathrm{O}_{\prec}(p^{-3/2}\sqrt{n_1 n_2})).  
\end{equation}
According to Theorem A.43 of \cite{bai2010spectral}, we can see that the $\ell_\infty$ norm of the ESDs between $\mathbf{K} \mathbf{K}^\top$ and   $(\mathbf{K}-\mathbf{L}_1)(\mathbf{K}-\mathbf{L}_1)^\top$ can be bounded by $3n_1^{-1} \operatorname{rank}(\mathbf{L}_1).$ Since $\operatorname{rank}(\mathbf{L}_1)$ is bounded, to study the limiting ESD of $n_1 n_2 \mathbf{N}_1=\mathbf{K} \mathbf{K}^\top,$ it suffices to analyze that of $4 \omega^2 \mathbf{U} \mathbf{U}^\top.$

As $\mathbf{U} \mathbf{U}^\top=\mathbf{W}_1 \mathbf{W}_2^\top \mathbf{W}_2 \mathbf{W}_1^\top$ share the same nonzero eigenvalues with $\mathbf{W}_1^\top \mathbf{W}_1 \mathbf{W}_2^\top \mathbf{W}_2,$ it suffices to understand the ESD of $\mathbf{W}_1^\top \mathbf{W}_1 \mathbf{W}_2^\top \mathbf{W}_2.$ According to Lemma \ref{lem_free}, we have seen that the ESD of $\frac{1}{p\sigma_1^2 \sigma_2^2 \sqrt{n_1 n_2}} \mathbf{W}_1^\top \mathbf{W}_1 \mathbf{W}_2^\top \mathbf{W}_2$ converges weakly to $ \mu_{\SMP_1 \boxtimes \SMP_2}.$  This immediately proves the results when $\mathcal{X}$ and $\mathcal{Y}$ contain purely noise.

Finally, we briefly discuss how to generalize the results to the general setting when $\mathcal{X}$ and $\mathcal{Y}$ contain weak signals that are dominated by the noise in the sense that (\ref{eq_weaksnr}) holds. Due to similarity, we mainly point out the significant differences. In this case, (\ref{eq_kij}) can be written as  
\begin{equation*}
\mathbf{K}(i,j)= \exp \left(-\|\xb^0_i-\yb^0_j \|_2^2/h_n-\| \bm{\xi}_i \|_2^2/h_n-\|\bm{\zeta}_j\|_2^2/h_n+2 \bm{\xi}_i^\top \bm{\zeta}_j/h_n+2(\mathbf{x}^0_i-\mathbf{y}^0_j)^\top (\bm{\xi}_i-\bm{\zeta}_j)/h_n \right).  
\end{equation*} 
According to the second part of the results of Proposition \ref{lem_bandwidthconcentration}, we have that 
\begin{equation*}
\|\xb^0_i-\yb^0_j \|_2^2/h_n \prec \upsilon, \ (\mathbf{x}^0_i-\mathbf{y}^0_j)^\top (\bm{\xi}_i-\bm{\zeta}_j)/h_n \prec \upsilon.  
\end{equation*}
Then for the rest of the prove, we can follow closely with the proof of Theorem 2.5 of \cite{DW2} to do a finite order expansion so that there exists some $\mathbf{L}_2$ with finite rank so that for $\upsilon$ in (\ref{eq_upsilonbound}), (\ref{eq_kexpanded}) can be updated to 
\begin{equation*}
\mathbf{K}=\mathbf{L}_1+2\omega\mathbf{U}+\mathrm{O}_{\prec}(\upsilon).
\end{equation*}  
Then we can conclude the proof using similar arguments around (\ref{eq_aroundaround}). This completes our proof. }
\end{proof}

\section{Additional discussions and auxiliary lemmas}\label{sec_additionaladalall}

\subsection{Hyperparameter selection, computational complexity,  and noise detection}\label{sec_tuningparametersection}
{

\paragraph{Hyperparameter selection.} {For the $k$NN hyperparameter $k$ in the alignability screening step of Algorithm \ref{al0}, we recommend choosing $k=30$ in practice. This default value was used in all our numerical studies. In our simulation, by comparing the results under different $k\in\{20,25,30,35,40\}$, we found the algorithm's performance was not sensitive to the choice of this parameter.} 

For the bandwidth-related hyperparameter $\omega$, as demonstrated in our theoretical analysis in Section \ref{sec_theoreticalanalysis}, under our specified assumptions, any constant value of $\omega$ within the range of 0 to 1 ensures the attainment of high-quality embeddings. In practice, to optimize the empirical performance and improve automation of the method, we recommend using a resampling approach developed in \cite{DW2}.   The method searches for an optimal $\omega$, through a resampling technique, such that the larger outlier eigenvalues are most distinguished/separated from the bulk eigenvalues of the kernel matrix $\bK$.   }

{
\paragraph{Computational complexity.} Regarding computational complexity, our proposed algorithm consists of two main components. First, we construct the duo-landmark kernel matrix, which incurs a computational cost of $\mathrm{O}(n^2 p)$, where $n = \max\{n_1, n_2\}$ and $p$ is the data dimension. Second, we compute the top $r$ singular values and singular vectors of this matrix via a truncated SVD, which has a computational complexity of $\mathrm{O}(n^2 r)$ using iterative methods such as the Lanczos algorithm. Therefore, the overall computational complexity of the proposed algorithm is $\mathrm{O}(n^2 p)$, which is comparable to that of kernel PCA applied to the combined dataset. As for memory requirements, the algorithm requires $\mathrm{O}(n^2)$ space to store the kernel matrix and the associated spectral quantities.

Regarding the condition number in large-scale settings, as discussed in Sections \ref{sec_manifoldmodelandlandmarkoperator} and \ref{sec_robustness}, our theoretical analysis reduces to a high-dimensional spiked covariance matrix model. Although the leading singular values are typically well-separated from the bulk spectrum, their magnitudes are generally comparable—thanks to our bandwidth selection strategy—which ensures that the resulting matrices are reasonably well-conditioned in practice. However, if some of our assumptions fail—resulting in leading eigenvalues that are significantly larger than the smallest ones—the numerical computations may suffer from reduced accuracy when using default numerical linear algebra packages available in standard software.

Nevertheless, the focus of our current paper is to develop statistically efficient and practically interpretable algorithms for constructing joint embeddings of noisy and nonlinear datasets. To improve numerical efficiency, two potential approaches may be considered. First, one could incorporate randomized numerical linear algebra techniques. However, this presents theoretical challenges. To the best of our knowledge, most existing algorithms and their convergence guarantees assume that the target data matrix is clean and noise-free. To rigorously integrate these techniques into our framework, one would need to analyze their behavior on noisy matrices and establish convergence results that depend on the condition number of the underlying clean matrix and the signal-to-noise ratio. Second, one could consider using non-smooth kernel functions, such as the $k$NN or $0$–$1$ kernel, to construct sparse kernel matrices. However, to our knowledge, the convergence properties of such sparse kernel matrices in the context of high-dimensional, noisy, non-smooth manifolds remain unknown. Moreover, the techniques developed in the current paper do not extend to such settings.  }

\paragraph{An algorithm for noise detection.} Based on Theorem \ref{thm_noisededuction}, we propose the following algorithm that can detect the cases where the overall signal-to-noise ratio is too small for our duo-landmark joint embedding.
A fundamental insight guiding this strategy is two-fold: on the one hand, we have the proximity of bulk eigenvalues to each other, resulting in ratios of consecutive eigenvalues close to one; on the other hand, after proper scaling, the bulk eigenvalues are all of the constant order. The algorithm is succinctly outlined below for completeness. \\

\numberwithin{thm}{section}

\begin{algo} Detect if the two datasets do not contain sufficiently large SNR. 
	\begin{enumerate}
		\item Let $\{w_i\}$ be the set of nonzero eigenvalues of ${\sf s}'\bN_1$ (or equivalently ${\sf s}'\bN_2$), where ${\sf s}'=\frac{\sqrt{n_1n_2}}{p}$.
		
		\item For some finite integer $k>0$ and small constants ${\sf c}_1, {\sf c}_2>0$,  we say the two datasets are dominated by noise if the following hold:
		\begin{itemize}
			\item Small eigen-gap: for all $i\ge k$, we have
${w_{i}}/{w_{i+1}} < 1+\mathsf{c}_1.$
			\item Bulk eigenvalue magnitude: the median $\bar w$ of $\{w_i\}$ satisfies $\bar w>{\sf c}_2$. 
		\end{itemize}
	\end{enumerate}
\end{algo}

\subsection{Some background on Riemannian manifold} \label{manifold.sec} 

Regarding Assumption \ref{assum_signal} in the main paper, below we provide some additional background on the theory of smooth manifold and Riemannian geometry, followed by some clarifications about the practical relevance of our theoretical framework. For a more comprehensive introduction to the topic, we refer the readers to the monographs \cite{boothby2003introduction, lee2013smooth}.

A \emph{Riemannian manifold} is a smooth manifold $\mathcal{M}$ equipped with a Riemannian metric $\mathrm{g},$ denoted as the pair $(\mathcal{M}, \mathrm{g}).$ When there is no confusion, we usually omit the metric $\mathrm{g}$ and simply write $\mathcal{M}.$ For a Riemannian manifold, it is useful to use the \emph{Riemannian density or volume form} to do the integration on the manifold, denoted as $\dd V,$ where $V \equiv V_g.$ Especially, if $f: \mathcal{M} \rightarrow \mathbb{R}$ is a compactly supported continuous function, we denote the integral of $f$ over $\mathcal{M}$ as $\int_{\mathcal{M}} f \dd V.$

In statistics and data science, observations are commonly interpreted as points in Euclidean space, and it is often earlier to do calculations in Euclidean space. Motivated by these aspects, it is useful to link an arbitrary Riemannian manifold $(\mathcal{M}, \mathrm{g})$ isometrically to some subspace of the Euclidean space with a specific metric. The feasibility of the above procedure is guaranteed by the embedding theory \cite{han2006isometric}. An \emph{embedding} $\iota: \mathcal{M} \rightarrow \mathbb{R}^r,$ for some $r \geq m,$ is a smooth map and a homeomorphism onto its image. When this happens, $\mathcal{M}$ is called a $m$-dimensional \emph{submanifold} of the Euclidean space and $\mathbb{R}^r$ is said to be its \emph{ambient space}. Moreover, we call $\iota(\mathcal{M})$ the \emph{embedded submanifold}. From the computational perspective, we are interested in the \emph{isometric embedding} so that the calculations of the distances, angles and curvatures reduce to those in the Euclidean space. 

The following theorem of John Nash \cite{MR75639} shows that every  Riemannian manifold can be considered as a submanifold of some ambient space $\mathbb{R}^r$, with $r$ bounded by some fixed number only depending on $m$.   

\begin{thm}\label{thm_nashembedding}
Let $(\mathcal{M}, \mathrm{g})$ be a $m$-dimensional Riemannian manifold. Then there exists an isometric embedding $\iota: \mathcal{M} \rightarrow \mathbb{R}^r$ from $\mathcal{M}$ to $\mathbb{R}^r$ for some $r$. Moreover, when $\mathcal{M}$ is compact, it is possible that 
\begin{equation*}
r \leq \frac{m(3m+11)}{2};
\end{equation*}
when $\mathcal{M}$ is non-compact, it is possible that 
\begin{equation*}
r \leq \frac{m(m+1)(3m+11)}{2}. 
\end{equation*}
\end{thm}

 With the above embedding $\iota$ and the \emph{pushforward}, we consider the induced metric $\iota_* \mathrm{g}$ on the embedded submanifold $\iota(\mathcal{M})$ which is clearly another Riemannian manifold. For integration, we consider the induced volume form $\iota_* \dd V$ which is the Riemannian density on $\iota(\mathcal{M}).$ Consequently, for any integrable function $f: \iota(\mathcal{M}) \rightarrow \mathbb{R},$ we can define the associated integral as $\int_{\iota(\mathcal{M})} f \iota_* \dd V.$ 

Finally, we explain how the above manifold model is connected to our statistical applications. Suppose we observe i.i.d. samples $\mathbf{x}_i, 1 \leq i \leq n,$ according to the random vector $X$ as in Assumption \ref{assum_signal}. Moreover, we assume that the support of $X$ is $\iota(\mathcal{M}_1).$ Therefore, for any integrable function $\zeta: \iota(\mathcal{M}) \rightarrow \mathbb{R},$ since $X$ is supported on $\iota(\mathcal{M}),$ the calculations of $\mathbb{E}(\zeta(X))$ can be efficiently reduced to the integration on $\mathcal{M}$ using the above induced measure and volume form (see Eq (\ref{expec}) in the main text). We emphasize that in real applications, it is the embedded submanifold that matters since the observations are sampled according to $X$ which is supported on $\iota(\mathcal{M}).$ Consequently, we focus on the understanding of the geometric structure of $\iota(\mathcal{M})$ rather $\mathcal{M}$ and $\iota$ separately. For more discussions on this perspective, we refer the readers to \cite{cheng2013local,DUNSON2021282, shen2022robust,wu2018think}.  

{Next, we discuss how the models proposed in Assumption \ref{assum_signal} and the kernels introduced in Definition \ref{defn_clmd} are directly connected to manifold models and underlying structures. First, for the model in Assumption \ref{assum_signal}, let $\mathrm{g}_\ell$ be the metric associated with the Riemannian manifold $\mathcal{M}_\ell$, $\dd V_\ell$ be the volume form associated with $\mathrm{g}_\ell$ and $\iota^*_{\ell} \dd V_{\ell}$ be the induced measure on $\iota_\ell(\mathcal{M}_\ell)$. Then by Radon-Nikodym theorem (e.g., \cite{billingsley2008probability}), for some differentiable function $\mathsf{f}_\ell$ defined on $\mathcal{M}_\ell$, under Assumption \ref{assum_signal}, we have that for $x \in \iota_\ell(\mathcal{M}_\ell)$
\begin{equation}\label{eq_density}
\dd \widetilde{\mathbb{P}_\ell}(x)=\mathsf{f}_\ell(\iota_\ell^{-1}(x)) \iota^*_\ell \dd V_\ell(x). 
\end{equation}
$\mathsf{f}_\ell$ is commonly referred as the p.d.f of $X$ (or $Y$) on $\mathcal{M}_\ell.$ For example, if $\mathsf{f}_1$ is constant, we call $X$ a uniform random sampling scheme. With the above setup, for $\ell=1$, we define the expectation with respect to $X$ on the embedded manifold $\iota_\ell(\mathcal{M}_\ell)$ as follows:  for an integrable function $\zeta: \iota_\ell(\mathcal{M}_\ell) \rightarrow \mathbb{R},$ we have that 
\begin{align}\label{expec}
\mathbb{E} \zeta(X) & =\int_{\Omega} \zeta(X (\omega)) \dd \mathbb{P}_\ell(\omega)=\int_{\iota_\ell(\mathcal{M}_\ell)} \zeta(x) \dd \widetilde{\mathbb{P}_\ell}(x) \nonumber \\
&=\int_{\mathcal{M}_\ell} \zeta(x) \mathsf{f}_\ell(\iota_\ell^{-1}(x)) \iota_\ell^* \dd V_\ell(x)=\int_{\mathcal{M}_\ell} \zeta(\iota_\ell(y)) \mathsf{f}_\ell(y) \dd V_\ell(y). 
\end{align} 
Similarly, the expectation with respect to $Y$ can be defined by setting $\ell=2$.}

{Second, for the kernels in Definition  \ref{defn_clmd}, we can write
	\begin{equation*}\label{eq_kernelone}
		k_1(\bm{x}_1, \bm{x}_2):=\int_{\mathcal{S}_2} k(\bm{x}_1, \bm{z}) k(\bm{z}, \bm{x}_2) \widetilde{\mathsf{P}}_2(\mathrm{d} \bm{z})=\int_{\mathcal{M}_2} k(\bm{x}_1, \bO\bR\iota_2({\bm w})) k( \bO\bR\iota_2({\bm w}), \bm{x}_2) {\sf f}_2(\bm{w})\mathrm{d} V_2(\bm{w}).
	\end{equation*}
	Similarly, for $\bm{y}_1, \bm{y}_2 \in \mathcal{S}_2,$ we can write for the kernel $k_2(\cdot,\cdot): \mathcal{S}_2 \times \mathcal{S}_2 \rightarrow \mathbb{R}$ as 
		\begin{equation*}\label{eq_kerneltwo}
		k_2(\bm{y}_1, \bm{y}_2):=\int_{\mathcal{S}_1} k(\bm{y}_1, \bm{z}) k(\bm{z}, \bm{y}_2) \widetilde{\mathsf{P}}_1(\mathrm{d} \bm{w})=\int_{\mathcal{M}_1} k(\bm{y}_1, \bO\bR\iota_1(\bm{w})) k(\bO\bR\iota_1(\bm{w}), \bm{y}_2) {\sf f}_1(\bm{w})\mathrm{d} V_1(\bm{w}). 
	\end{equation*}
}

\subsection{Some auxiliary lemmas from high-dimensional probability and random matrix theory}

We first present a few useful lemmas.

\begin{lem} \label{sg.bnd.lem}
	Let $B$ be an $m\times n$ matrix, and let $\mathbf{x}$ be a mean zero, sub-Gaussian random vector in $\R^n$ with parameter bounded by $K$. Then for any $t\ge 0$, we have
	\[
	\mathbb{P}(\|B \mathbf{x} \|_2\ge CK\|B\|_F+t)\le \exp\bigg(-\frac{ct^2}{K^2\|B\|^2}\bigg).
	\]
\end{lem} 
\begin{proof}
	See page 144 of \cite{vershynin2018high}.
\end{proof}

%

The next two lemmas provide Bernstein type inequalities for bounded and sub-exponential random variables.

\begin{lem}\label{lem_boundedconcentration} Let $x_i, 1 \leq i \leq n,$ be independent mean zero random variables such that $|x_i|\le K$ for all $i$. Then for every $t \geq 0$
	\begin{equation*}
		\mathbb{P}\left( \left| \sum_{i=1}^n x_i \right|\geq t \right) \leq 2 \exp\left[- \frac{t^2/2}{\sum_{i=1}^n\E x_i^2+Kt/3}\right].
	\end{equation*} 
\end{lem}
\begin{proof}
	See Theorem 2.8.4 of \cite{vershynin2018high}.
\end{proof}

{
	\begin{lem}\label{HW.lem}
		Let $X=(X_1,...,X_n)\in\R^n$ be a random vector with independent centered components $X_i$ with sub-Gaussian parameter bounded by $K$. Let $A$ be an $n\times n$ matrix. Then, for every $t\ge 0$, we have
		\[
		P(|X^\top AX-\E X^\top AX|>t)\le 2\exp\bigg\{ -c\min\bigg(\frac{t^2}{K^4\|A\|_{HS}^2}, \frac{t}{K^2\|A\|} \bigg) \bigg\}.
		\]
	\end{lem}
	\begin{proof}
		See Theorem 1.1 of \cite{rudelson2013hanson}.
	\end{proof}
}

The next lemma is the classical Chernoff bound for Binomial  random variables.

\begin{lem} \label{cher.lem}
	Let $x_1, ..., x_n$ be independent random variables with $\mathbb{P}(x_k=1)=\mathsf{p}$ and $\mathbb{P}(x_k=0)=1-\mathsf{p}$ for  each $k$. Then for any $t>n\mathsf{p}$, we have
	\[
	\mathbb{P}\bigg(\sum_{k=1}^n x_k>t\bigg)\le e^{-n\mathsf{p}}\bigg( \frac{en \mathsf{p}}{t}\bigg)^t.
	\]
\end{lem}
\begin{proof}
	See Section 2.3 of  \cite{vershynin2018high}.
\end{proof}


In the following lemmas, we will use stochastic domination to characterize the high-dimensional concentration. The next lemma collects some useful concentration inequalities for the noise vector.

\begin{lem}\label{lem_concentrationinequality} Suppose Assumption \ref{assum_mainassumption} holds. Then we have that 
	$$
	\max_{(i,j) \neq (k,l)}	\frac{1}{\sigma^2p}\left| (\bxi_i-\bzeta_j)^\top (\bxi_k-\bzeta_l) \right| \prec  { p^{-1/2}},
	$$
	and 
	$$ 
	\max_{i,k}	\left| \frac{1}{ \sigma^2p} \| \bxi_i-\bzeta_k \|_2^2-1 \right| \prec  { p^{-1/2}}. 
	$$
\end{lem}
\begin{proof}
	See Lemma A.4 of \cite{DW2} or Lemmas A.1 and A.2 of \cite{9205615}. 
\end{proof}
{
The following lemma provides some useful information for kernel random matrix approximation. 
\begin{lem}\label{lem_lowrank} For the matrices $\mathbf{U}$ and $\mathbf{V}$ defined in (\ref{eq_defnU}) and (\ref{eq_vdecomposition}), respectively,  under the assumptions of Lemma \ref{lem_free}, we find that there exists some finite rank matrices $\mathbf{E}_1$ and $\mathbf{E}_2$ so that 
\begin{equation*}
\mathbf{U} \circ \mathbf{U}=\mathbf{E}_1+\mathrm{O}_{\prec}(p^{-3/2}\sqrt{n_1 n_2}), \ \mathbf{U} \circ \mathbf{V}=\mathbf{E}_2+\mathrm{O}_{\prec}(p^{-3/2}\sqrt{n_1 n_2}).
\end{equation*} 
\end{lem}
\begin{proof}
The results have been essentially established in the proof of Theorem 2.2 in \cite{elkaroui2010spectrum}; see also the summary in Appendix A of \cite{DW2}.
\end{proof}
}

Finally, we provide an important result from operator theory. 
\begin{lem}\label{lem_Operatorcommuteshavesameeigenvalues} For two operators $\mathcal{A}$ and $\mathcal{B},$ if they commute in the sense that $\mathcal{A} \mathcal{B}=\mathcal{B} \mathcal{A},$ then $\mathcal{A} \mathcal{B}$ and $\mathcal{B} \mathcal{A}$ have the same set of nonzero eigenvalues and eigenfunctions. 
\end{lem}
\begin{proof}
	See Section 5.1 of \cite{miller2008quantum}. 
\end{proof}

 
\setstretch{1.12}
\bibliographystyle{chicago}
\bibliography{kernel}

\end{document}